\documentclass[12pt]{article}
\usepackage[margin=1in]{geometry}

\usepackage{wrapfig}

\usepackage[utf8]{inputenc} 
\usepackage[T1]{fontenc}    
\usepackage{hyperref}       
\usepackage{url}            
\usepackage{booktabs, multicol, multirow}       
\usepackage{amsfonts}       
\usepackage{amsmath}
\usepackage{amsthm}
\usepackage{amsfonts}       
\usepackage{amsmath}
\usepackage{amsthm}
\usepackage{xfrac}
\usepackage{textcomp}

\usepackage{caption}
\usepackage{multirow}
\usepackage{boldline}
\usepackage{hhline}
\usepackage{enumitem}

\usepackage{nicefrac}       
\usepackage{microtype}      
\usepackage[toc,page]{appendix}

\usepackage{upgreek,rotating}
\newcolumntype{?}{!{\vrule width 1pt}}

\usepackage[round]{natbib}
\usepackage{thm-restate}
\usepackage{graphicx}
\usepackage{subfigure}

\usepackage{tikz}
\usepackage{tikz-cd}

\usepackage{tikz-cd}
 \usepackage{relsize}
 \usetikzlibrary{positioning}
 \usetikzlibrary{shapes.misc, positioning}
\usetikzlibrary{calc,decorations.markings}

\usepackage{nicefrac}       
\usepackage{microtype}      
\usepackage[toc,page]{appendix}
\usepackage{booktabs}

\usepackage{upgreek,rotating}
\newcolumntype{?}{!{\vrule width 1pt}}

\usepackage{thm-restate}
\usepackage{graphicx}
\usepackage{subfigure}
\usepackage[margin=1in]{geometry}
\usepackage{multicol}
\usepackage{multirow}
\usepackage{array}
    \newcolumntype{P}[1]{>{\centering\arraybackslash}p{#1}}
    \newcolumntype{M}[1]{>{\centering\arraybackslash}m{#1}}

\usepackage{caption}

\usepackage{tikzit}
\usepackage[group-separator={,}]
{siunitx}

\title{The Tempered Hilbert Simplex Distance and Its Application To
Non-linear Embeddings of TEMs}
\author{
   Ehsan Amid$^\dagger$\thanks{Alphabetical author order.} \quad Frank Nielsen$^\ddagger$ \quad Richard Nock$^\star$ \quad Manfred K. Warmuth$^\star$ \\
 $^\dagger$Google DeepMind\\
 $^\ddagger$Sony Computer Science Laboratories Inc\\
 $^\star$Google Research\\
{\normalsize eamid@google.com, frank.nielsen@acm.org, \{richardnock,manfred\}@google.com} \\
}

\date{}

\usepackage{xcolor}
\usepackage{array}

\usepackage{mathabx}

\usepackage{booktabs,makecell} 

\usepackage{verbatim}

\usepackage{xfp}
\usepackage{dsfont}

\newcolumntype{?}{!{\vrule width 2pt}}

\newcolumntype{C}[1]{>{\centering\let\newline\\\arraybackslash\hspace{-5pt}}m{#1}}

\usepackage{upgreek,eucal,colortbl}
\usepackage{graphicx} 
\usepackage{multirow}
\usepackage{bm} 
\usepackage{color}
\usepackage{nicefrac}       

\usepackage{scalerel,stackengine}
\newcommand\reallywidecheck[1]{%
\savestack{\tmpbox}{\stretchto{%
  \scaleto{%
    \scalerel*[\widthof{\ensuremath{#1}}]{\kern-.6pt\bigwedge\kern-.6pt}%
    {\rule[-\textheight/2]{1ex}{\textheight}}
  }{\textheight}%
}{0.5ex}}%
\stackon[1pt]{#1}{\scalebox{-1}{\tmpbox}}%
}

\usepackage{tikz}
\usetikzlibrary{arrows,decorations.pathmorphing,backgrounds,positioning,fit,petri}
\usetikzlibrary{calc}

\usepackage{stmaryrd}

\definecolor{ao}{rgb}{0.0, 0.5, 0.0}

\makeatletter\newcommand*\bigcdot{\mathpalette\bigcdot@{.5}}
\newcommand*\bigcdot@[2]{\mathbin{\vcenter{\hbox{\scalebox{#2}{$\m@th#1\bullet$}}}}}
\makeatother

\definecolor{darkgreen}{RGB}{0,205,0}

\usepackage{pifont}

\usepackage{xcolor,colortbl}

\definecolor{Gray}{gray}{0.85}
\definecolor{LightCyan}{rgb}{0.88,1,1}

\definecolor{Gray}{gray}{0.85}
\definecolor{LightCyan}{rgb}{0.88,1,1}

\newcommand{\answeryes}[1]{\textcolor{blue}{\textbf{Yes}}}
\newcommand{\answerno}[1]{\textcolor{red}{\textbf{No}}}
\newcommand{\answerna}[1]{\textcolor{orange}{\textbf{Not Applicable}}}
\newcolumntype{a}{>{\columncolor{Gray}}c}
\newcolumntype{b}{>{\columncolor{white}}c}
\newcolumntype{d}{>{\columncolor{Gray}}r}

  \usepackage[framemethod=TikZ]{mdframed}
\mdfdefinestyle{MyFrame}{%
    linecolor=gray,
    outerlinewidth=2pt,
    roundcorner=10pt, 
    innertopmargin=0pt,
    innerbottommargin=\baselineskip,
    innerrightmargin=10pt,
    innerleftmargin=10pt,
    backgroundcolor=gray!10!white}

\mdfdefinestyle{MyFrameLoss}{%
    linecolor=gray,
    outerlinewidth=2pt,
    roundcorner=10pt, 
    innertopmargin=0pt,
    innerbottommargin=\baselineskip,
    innerrightmargin=10pt,
    innerleftmargin=10pt,
    backgroundcolor=white!60!cyan}

\mdfdefinestyle{MyFrameModel}{%
    linecolor=gray,
    outerlinewidth=2pt,
    roundcorner=10pt, 
    innertopmargin=0pt,
    innerbottommargin=\baselineskip,
    innerrightmargin=10pt,
    innerleftmargin=10pt,
    backgroundcolor=white!60!red}

\usepackage{amsmath}
\usepackage{amsfonts}
\usepackage{amssymb}
\usepackage{setspace}

\renewcommand{\epsilon}{\varepsilon}
\renewcommand{\phi}{\varphi}

\newcommand{\Mat}[1]{\mathbf{#1}}

\newcommand{\Prj}[2]{{
  \left.\kern-\nulldelimiterspace 
  #1 
  \vphantom{\big|} 
  \right|_{#2} 
}}

\newcommand{\defeq}{\stackrel{\mathrm{.}}{=}}

\usepackage{mathrsfs}

\newcommand{\ve}[1]{\bm{#1}}

\newtheorem{definition}{Definition}
\newtheorem{proposition}{Proposition}

\def\cqfd{\hfill\hbox{$\hbox{\vrule width 0.8pt
\vbox to6pt{\hrule depth 0.8pt width 5.2pt
\vfill\hrule depth 0.8pt}\vrule width 0.8pt}$}} 

\newtheorem{theorem}{Theorem}
\newtheorem{lemma}{Lemma}

\newtheorem{remark}{Remark}

\usepackage{eucal}

\usepackage{graphicx}

\usepackage{algorithm}
\usepackage{algorithmic}

\usepackage[all]{xy}
\CompileMatrices

\newcommand{\intset}[1]{\cbr{1..n}}

\newcommand{\node}{\nu}

\newcommand{\acrotem}{TEM}

\newcommand{\HG}{\text{\tiny HG}}
\newcommand{\tHG}{\text{\tiny $t$-HG}}
\newcommand{\tFD}{\text{\tiny $t$-FD}}
\newcommand{\tsFD}{\text{\tiny $t^*$-FD}}
\newcommand{\tNH}{\text{\tiny $t$-NH}}
\newcommand{\tvar}{\text{\tiny $t$-var}}

\newcommand{\tHGdiff}{\text{\tiny $t$-dHG}}
\newcommand{\ttpHGdiff}{\text{\tiny $t,\!\delta$-dHG}}
\newcommand{\tFDdiff}{\text{\tiny $t$-dFD}}

\newcommand\riemannint[3]{\sideset{^{(#1)\hspace{-0.1cm}}}{}{\int_{#2}^{#3}}}
\newcommand{\semitnorm}[2]{\sideset{^{(#1)}}{}{\mathop{\Vert}} {#2} \Vert}

\def\st{\ :\ }
\def\calC{\mathcal{C}}
\def\arccosh{\mathrm{arccosh}}
\def\bbR{\mathbb{R}}

\def\HG{\mathrm{HG}}
\def\Bi{\mathrm{Bi}}
\def\eps{\epsilon}
\def\st{{\ :\ }}  
\def\du{\mathrm{d}u}

\begin{document}

\maketitle
\begin{abstract}
Tempered Exponential Measures (\acrotem s) are a parametric generalization of the exponential family of distributions maximizing the tempered entropy function among positive measures subject to a probability normalization of their power densities. Calculus on \acrotem s relies on a deformed algebra of arithmetic operators induced by the deformed logarithms used to define the tempered entropy. In this work, we introduce three different parameterizations of finite discrete \acrotem s via Legendre functions of the negative tempered entropy function. In particular, we establish an isometry between such parameterizations in terms of a generalization of the Hilbert log cross-ratio simplex distance to a tempered Hilbert co-simplex distance. Similar to the Hilbert geometry, the tempered Hilbert distance is characterized as a $t$-symmetrization of the oriented tempered Funk distance. We motivate our construction by introducing the notion of $t$-lengths of smooth curves in a tautological Finsler manifold. We then demonstrate the properties of our generalized structure in different settings and numerically examine the quality of its differentiable approximations for optimization in machine learning settings.
\end{abstract}

\section{Introduction}
Tempered Exponential Measures (\acrotem s)~\citep{amid2023clustering} are an alternate generalization of the exponential family via the parametric variants of the standard logarithm and exponential functions, initially introduced in thermostatics~\citep{texp1,texp2}. Similar parametric families of statistical models, such as the $q$-exponential~\citep{naudts2009q,amari2011geometry} and the deformed exponential family~\citep{amari2012geometry}, have been proposed before via a max (Tsallis) entropy~\citep{tsallis1988possible} principle. However, the main modification in the axiomatic characterization of \acrotem s that sets them apart from similar formulations is the normalization constraint, imposed not on the measure itself but a closely related quantity called the \emph{co-density}. This alteration, although seemingly cosmetic at first, induces major implications, including the canonical form, information geometry, and other implicit properties highly relevant to machine learning. Consequently, \acrotem s have been emerging as a parametric family with applications in machine learning as diverse as clustering~\citep{amid2023clustering}, boosting~\citep{nock2023boosting} and optimal transport~\citep{amid2023optimal}.

In this paper, we focus on the parametric characterization of the discrete \acrotem s via the dual functions of the negative tempered entropy. We establish isometries between such parameterizations via a generalization of the Hilbert simplex distance~\citep{HB-OriginHilbert-2014}. Hilbert discovered the now-called Hilbert geometry during the summer of 1894~\citep{Hilbert-1895} 
when he investigated his 4th problem~\citep{HB-fourthproblem-2014}:   
 The study of metrics on projective space subsets for which line segments are geodesics.
 The Hilbert distance $\rho_{\HG}^\Omega(\ve{r}, \ve{s})$ between two distinct points $\ve{r}$ and $\ve{s}$ induced by a 
 bounded convex set $\Omega$ with boundary $\partial\Omega$ is defined according 
 to the logarithm of the cross-ratio of the four ordered collinear points $\bar{\ve{r}}, \ve{r}, \ve{s}, \bar{\ve{s}}$ where $\bar{\ve{r}}$ and $\bar{\ve{s}}$ are the intersection of the line $(\ve{r} \ve{s})$ with $\partial\Omega$:
 $$
\rho_{\HG}^\Omega(\ve{r}, \ve{s})=\chi\, \log \frac{\|\ve{r}-\bar{\ve{s}}\| \, \|\ve{s}-\bar{\ve{r}}\| }{\|\ve{r}-\bar{\ve{r}}\| \, \|\ve{s}-\bar{\ve{s}}\|},
 $$
 where $\chi>0$ is a scalar factor and $\|\cdot\|$ is any arbitrary norm (e.g., $\|\cdot\|_2$). The points 
 When $\chi=\frac{1}{2}$ and $\Omega=\{ \ve{x}\in\bbR^d \st \|\ve{x}\|_2<1\}$ is the open unit ball, 
 the Hilbert distance coincides with Klein hyperbolic distance~\citep{ProjectiveBook-2011}.
 Furthermore, when $\Omega$ is an open ellipsoid, the Hilbert geometry amounts to the Cayley-Klein geometry~\citep{CKgeometry-2006}.
 \citet{Birkhoff-1957}, motivated by studying the contraction factor of linear maps,
 defined a distance between pairs of rays
in a closed cone $\calC=\{(\lambda,\lambda \Omega) \st \lambda\geq 0\}$ 
which coincides with the Hilbert metric on any affine slice $\lambda\Omega$ of the cone.
The Perron-Frobenius theorem is obtained  as a corollary of the Birkhoff theorem.
Birkhoff distance is nowadays commonly called Hilbert projective distance~\citep{HB-BirkhoffHilbert-2014} and is used in machine learning to study the convergence of  Sinkhorn-type algorithms solving various regularized optimal transport~\citep{PC-OT-2019}.
Hilbert geometry can be studied from the Finslerian viewpoint 
and becomes Riemannian only when $\Omega$ is an ellipsoid~\citep{HB-FinslerHilbert-2014}.
When $\Omega$ is a simplex of $\bbR^d$, \citet{deLaHarpe-2000} proved that 
Hilbert geometry is isometric to a polytopal normed vector space.
This Hilbert simplex geometry has been considered in machine learning for clustering tasks~\citep{ClusteringHilbert-2019}
 and a differentiable approximation of the Hilbert simplex distance 
 has been used for non-linear embeddings~\citep{NLEHilbert-2023} 
 which experimentally performs better than hyperbolic graph embeddings for some classes of graphs.

\section{Discrete Tempered Exponential Measures}

\subsection{A Primer on \acrotem s}
\paragraph{Definitions} We start by reviewing the tempered logarithm and exponential functions~\citep{nGT},
\begin{gather}
    \log_t x = \frac{1}{1 - t}(x^{1 - t} - 1)\,,  \label{eq:log}\\
    \exp_t x = [1 + (1 -t)\, x]_+^{\frac{1}{1 - t}}\,\,\, \big([\cdot]_+ = \max(\cdot, 0)\big)\,, \label{eq:exp}
\end{gather}
parameterized by a \emph{temperature} $t\in \mathbb{R}$. Throughout our construction, we restrict $t < 2$. For $t = 1$, we recover the standard $\log$ and $\exp$ functions as limit cases. Au contraire, the sum-product property $\log_t a + \log_t b \neq \log_t (a\cdot b)$, among others, no longer holds for $t\neq 1$. Rather, such property requires switching standard operators to a generalized $t$-algebra, generalizing their counterparts at $t=1$,
\begin{align}
    a \oplus_t b & = \log_t(\exp_t a\cdot \exp_t b)\,,\label{eq:t-plus}\\   
    a \ominus_t b & = \log_t\Big(\frac{\exp_t a}{\exp_t b}\Big)\,,\,\,\, \text{s.t.}\,\,\exp_t b \neq 0\,.\label{eq:t-minus}    
\end{align}
Both operations\footnote{Similar operations were initially introduced by \citet{nlwGA}. See \citet{amid2023clustering} for a comprehensive list.} are associative (while $\oplus_t$ being commutative), with the neutral element $0$,
\begin{align}
    & x \oplus_t\, 0 = x \ominus_t 0 = \ominus_t (\ominus_t x) = x\,.\label{eq:zero-element}
\end{align}
where $\ominus_t x = 0 \ominus_t x$. For instance, $(x \ominus_t w) \ominus_t (y \ominus_t z) = x \ominus_t w \ominus_t y \oplus_t z$. Predominantly relevant to our construction, we have $\log_t a \oplus_t \log_t b = \log_t (a\cdot b)$ and accordingly for $\ominus_t$, replacing the product with division. 

\begin{figure*}[t!]
\vspace{-0.4cm}
\begin{center}
    \subfigure[$t=0.8$]{\includegraphics[width=0.258\linewidth]{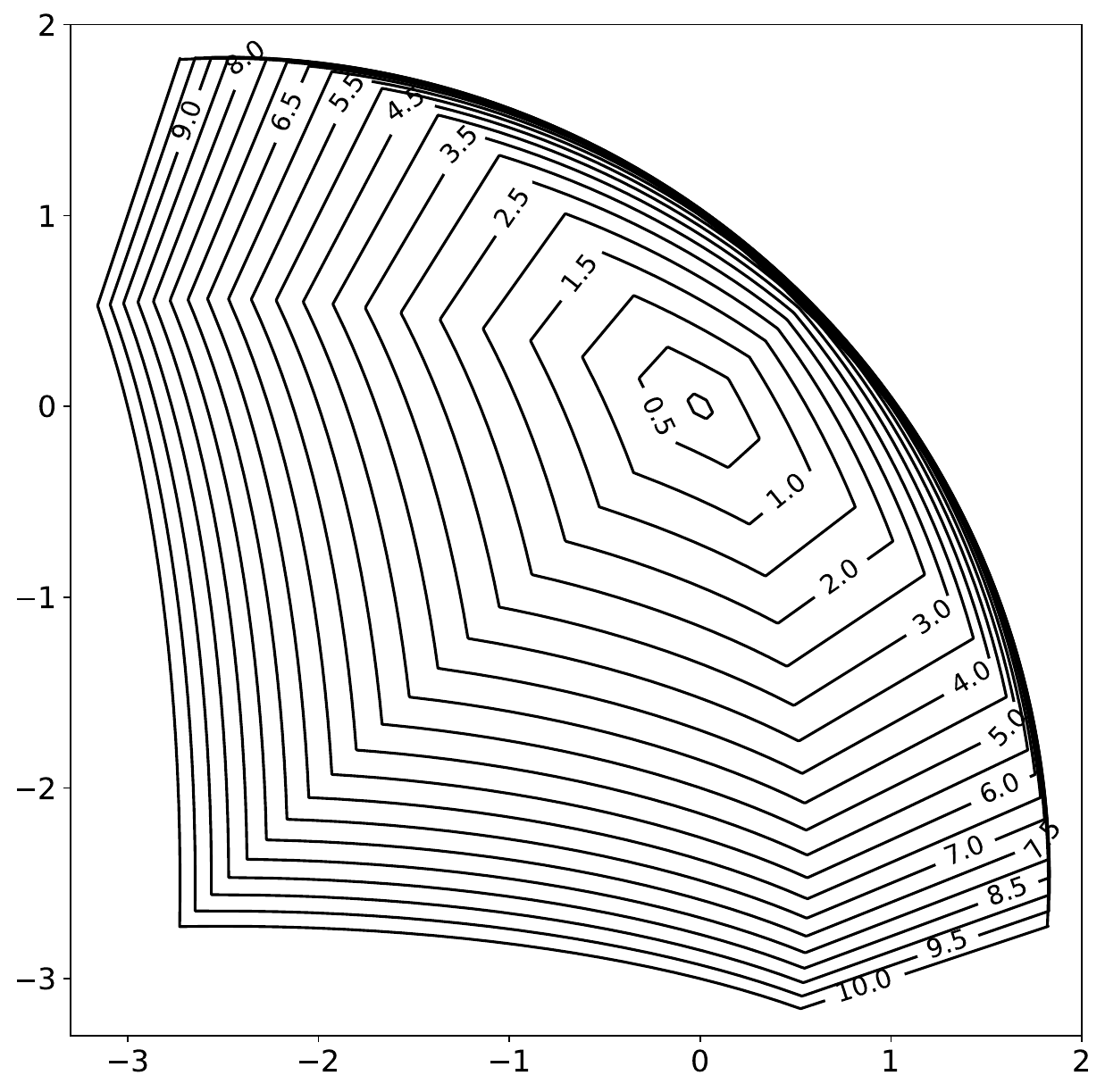}}
    \subfigure[$t=1.0$]{\includegraphics[width=0.255\linewidth]{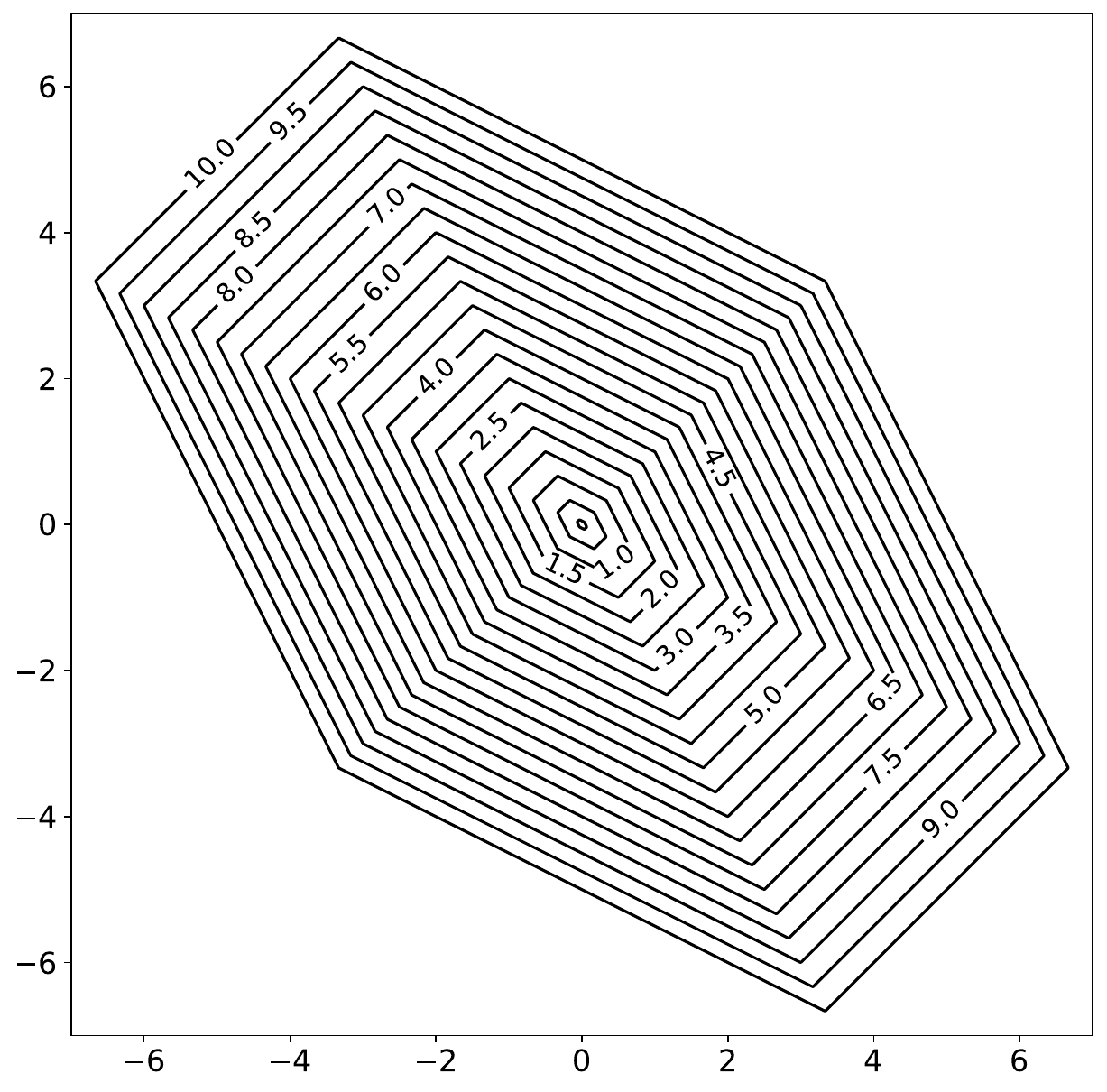}}
    \subfigure[$t=1.1$]{\includegraphics[width=0.255\linewidth]{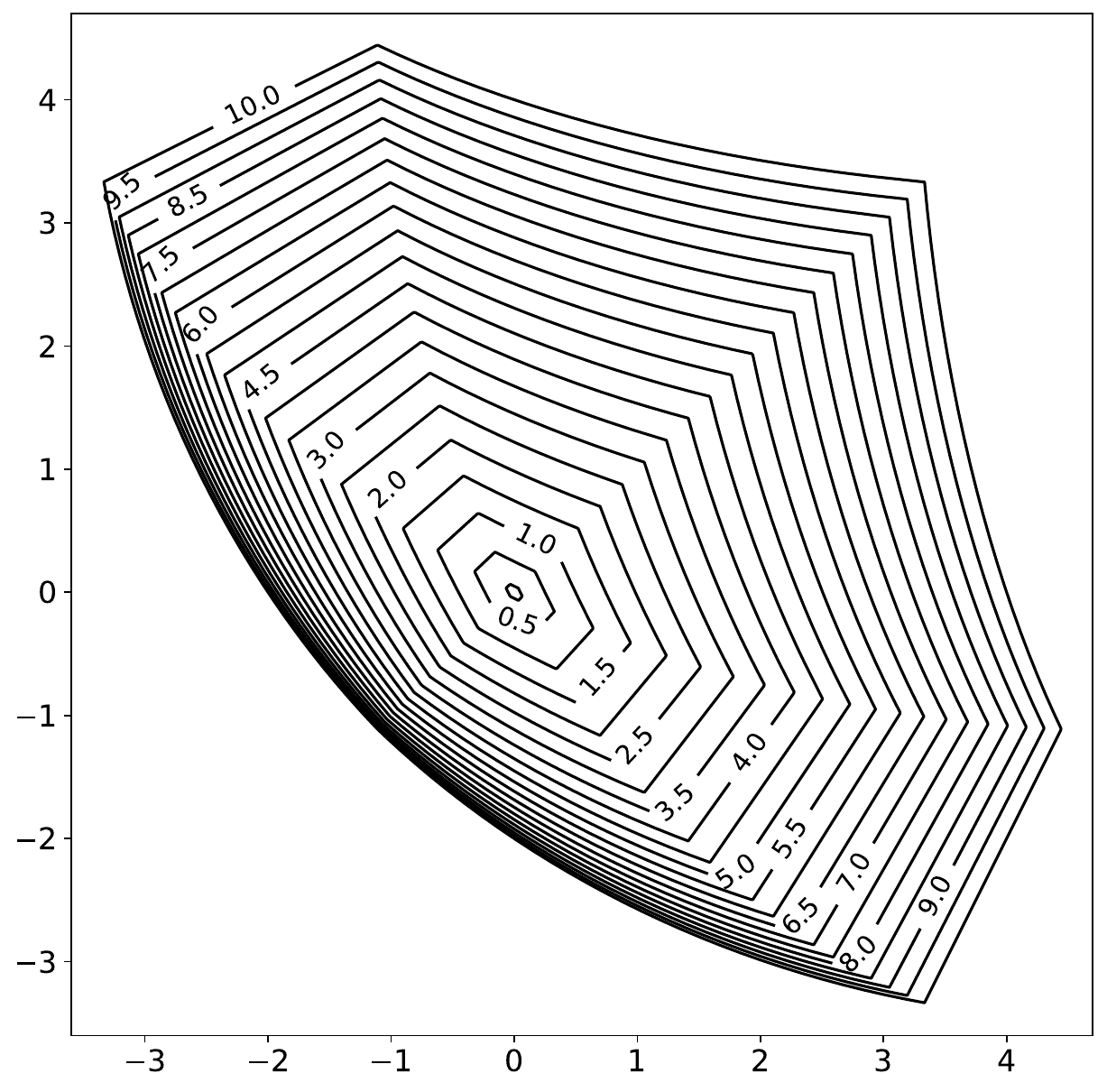}}
    \vspace{-0.2cm}
    \caption{$B^r_{\tilde{V}^d_t}$ balls of $\|\cdot\|_{\tNH}$ distance of different radii ($r \leq 10$) on $\tilde{V}^3_t$ shown along the first two dimensions.}
    \label{fig:relative}
    \end{center}
    \vspace{-0.5cm}
\end{figure*}

\paragraph{Tempered Exponential Measures} \acrotem s~\citep{amid2023clustering} are defined as a generalization of the exponential family with unnormalized densities. \acrotem s admit the following canonical form 
\begin{equation}
    \label{eq:tem}
    \begin{split}
    \tilde{p}_{t|\ve{\theta}}(\ve{x}) & = \frac{\exp_t\big(\ve{\theta}\cdot \ve{\phi}(\ve{x})\big)}{\exp_t (G_t(\ve{\theta}))} = \exp_t\big(\ve{\theta}\cdot \ve{\phi}(\ve{x}) \ominus_t G_t(\ve{\theta})\big)\,,
    \end{split}
\end{equation}
where $\ve{\theta}$ is the \emph{natural parameters}, $\ve{\phi}(\ve{x})$ is the \emph{sufficient statistics}, and $G_t(\ve{\theta})$
is the \emph{cumulant} function. The exponential family of distributions is a special form when $t=1$. However, \acrotem s impose the unit mass constraint (ensured by the cumulant function) not on the measure itself but on a relevant quantity called the \emph{co-density}, $p_{t|\ve{\theta}} = \tilde{p}^{1/t^*}_{t|\ve{\theta}}$, that is, $\int p_{t|\ve{\theta}}(\ve{x})\mathrm{d}\xi = \int \tilde{p}^{1/t^*}_{t|\ve{\theta}}(\ve{x})\mathrm{d}\xi = 1$ for $t^*=\frac{1}{2-t}$.

Remarkably, any discrete probability distribution $\ve{p} \in \Delta^d$, in which $\Delta^d = \{\ve{p} \in \mathbb{R}_+^d\vert \sum_i p_i = 1\}$ denotes the $(d-1)$-dimensional probability simplex, can be characterized uniquely via a discrete \acrotem, i.e., $\tilde{\ve{p}} \in \tilde{\Delta}_t^d$ in the \emph{co-simplex} $\tilde{\Delta}_t^d = \{\tilde{\ve{p}} \in \mathbb{R}_+^d\vert \sum_i \tilde{p}^{1/t^*}_i = 1\}$ via the transformation $\tilde{\ve{p}} = \ve{p}^{t^*}$. Note that for $t=1$, we have $\tilde{\Delta}_1^d = \Delta^d$ and the identification is trivial. In this paper, we focus on different parameterizations of discrete \acrotem s induced by a closely related concept, the \emph{negative tempered entropy} function, which in the discrete case can be written as
\begin{equation}
    \label{eq:entropy}
    F_t(\tilde{\ve{p}}) = \sum_{i \in [d]}\big(\tilde{p}_i \log_t \tilde{p}_i - \log_{t-1} \tilde{p}_i\big)\,.
\end{equation}
\acrotem s are motivated by maximizing the tempered relative entropy subject to a moment constraint (and unit mass of the corresponding co-density). For $t=1$, \eqref{eq:entropy} reduces to the negative Shannon entropy function (extended to positive measures), and thus, we recover the exponential family formulation.

\subsection{Parameterizations of Discrete \acrotem s}
We first discuss a parameterization for discrete \acrotem s based on the dual of the negative tempered entropy function~\eqref{eq:entropy} in the reduced form. This parameterization is a generalization of the reduction of the number of parameters for discrete distributions by one by utilizing the normalization constraint. Next, we consider the overparameterized form where the parameterizations are attained via the unconstrained and constrained Legendre dual of \eqref{eq:entropy}.

\paragraph{Minimal Form}
The family $\tilde{\ve{p}} \in \tilde{\Delta}_t^d$ of discrete \acrotem s can be written in the canonical form~\eqref{eq:tem} by taking an arbitrary nonzero component, e.g., $\tilde{p}_d$, where $\tilde{p}_d = \big(1 - \sum_{i \in [d-1]} \tilde{p}_i^{1/{t^*}}\big)^{t^*}$ and defining the natural parameters $\hat{\theta}_i = \log_t\frac{\tilde{p}_i}{\tilde{p}_d}$ and  $G_t(\hat{\ve{\theta}}) = \log_t \frac{1}{\tilde{p}_d}$. We can then write
\[
\tilde{p}(x) = \frac{\exp_t(\sum_{i\in [d-1]}\hat{\theta}_i \delta_i(x))}{\exp_t(G_t(\hat{\ve{\theta}}))}\, ,
\]
where $\delta_i(x) = 1$ if $x = i$ and $0$ otherwise. The negative of the tempered entropy function and the corresponding derivative (also called the \emph{link}) function for this case is given by\footnote{We ignore a constant term in the definition.}
\begin{align*}
 \hat{F}_t(\tilde{\ve{p}}) & = \sum_{i \in [d]} \tilde{p}_i \log_t \tilde{p}_i\,, \\
 [\hat{f}_t(\tilde{\ve{p}})]_i & = \frac{\partial \hat{F}_t(\tilde{\ve{p}})}{\partial \tilde{p}_i} = \log_t\frac{\tilde{p}_i}{\tilde{p}_d} = \hat{\theta}_i\,,\, i \in [d-1]\, .
\end{align*}
The Legendre dual~\citep{urruty} and the \emph{inverse link} are given by
\begin{align*}
\hat{F}^*_t(\hat{\ve{\theta}}) & = \log_t\Big[\big(1 + \sum_{i\in[d-1]} \exp_t^{1/{t^*}}\hat{\theta}_i\big)^{t^*}\Big]\,,\\
[\hat{f}^*_t(\hat{\ve{\theta}})]_i & = \frac{\exp_t\hat{\theta}_i}{\big(1 + \sum_{i\in [d-1]} \exp_t^{1/{t^*}}\hat{\theta}_i\big)^{t^*}}\,,\, i\in [d-1]\, .
\end{align*}
Remarkably, the Bregman divergence~\citep{bregman} induced by $\hat{F}_t$ can be simplified by the variable reduction into a convenient form given by
\[
D_{\hat{F}_t}\!(\tilde{\ve{p}}, \tilde{\ve{q}})\! = \!\!\!\!\!\!\!\sum_{i\in [d-1]}\!\!\!\!\! \tilde{p}_i\big[\log_t \frac{\tilde{p}_i}{\tilde{p}_d} - \log_t\frac{\tilde{q}_i}{\tilde{q}_d}\big] - \log_t \frac{1}{\tilde{p}_d} + \log_t\frac{1}{\tilde{q}_d}\, .
\]

\begin{table*}
\renewcommand{\arraystretch}{2}
  \centering
  \resizebox{\columnwidth}{!}{%
    \begin{tabular}{llcc}\toprule
      \textbf{Representation} & \textbf{Parameter} & \textbf{Link} & \textbf{Inverse Link} \\ \midrule
      Minimal Form & $\,\,\,\,\hat{\ve{\theta}} \in \mathbb{R}^{d-1}$ & $[\hat{f}_t(\tilde{\ve{p}})]_i = \log_t \frac{\tilde{p}_i}{\tilde{p}_d}$ & $[\hat{f}^*_t(\hat{\ve{\theta}})]_i = \frac{\exp_t\hat{\theta}_i}{\big(1 + \sum_{i\in [d-1]} \exp_t^{1/{t^*}}\hat{\theta}_i\big)^{t^*}}$\\ \hline
      Unconstrained Overparameterized & $\,\,\,\,\ve{\theta} \in \mathbb{R}^d$ & $f_t(\tilde{\ve{p}}) = \log_t \tilde{\ve{p}} $ & $f^*_t(\ve{\theta}) = \exp_t \ve{\theta}$\\\hline
       Constrained Overparameterized & $\,\,\,\,\check{\ve{\theta}} \in \mathbb{R}^d$ & $\check{f}_t(\tilde{\ve{p}}) = \log_t\frac{\tilde{\ve{p}}}{\tilde{\lambda}_t(\tilde{\ve{p}})}  $ & $[\check{f}^*_t(\check{\ve{\theta}})]_i = \frac{\exp_t\check{\theta}_i}{\big(\sum_i \exp_t^{1/{t^*}}\check{\theta}_i\big)^{t^*}}$
     \\ \bottomrule
    \end{tabular}
    }
    \vspace{-0.1cm}
    \caption{Different parameterizations of discrete \acrotem s $\tilde{\ve{p}} \in \tilde{\Delta}_t^d$ via Legendre functions of the negative tempered entropy function~\eqref{eq:entropy}.}
    \label{tab:links}
    \vspace{-0.1cm}
  \end{table*}


\paragraph{Overparameterized Form} We start by characterizing the (unconstrained) Legendre dual of the tempered entropy function 
\begin{equation*}
    \begin{split}
    \label{eq:dual-unconst}
        F_t^*(\ve{\theta})\! =\! \sup_{\tilde{\ve{p}}\,\in\mathbb{R}^d_+}\!\!\big\{\ve{\theta}\cdot \tilde{\ve{p}} - F_t(\tilde{\ve{p}})\big\} = t^*\big(\!\!\sum_{i\in [d]}\!\! \exp_t^{1/t^*}\!\! \theta_i - 1\big)\,,
    \end{split}
\end{equation*}
with the associated link
   $f^*_t(\ve{\theta}) = \exp_t \ve{\theta}$.
Thus, we can characterize discrete \acrotem s with the dual variable $\ve{\theta} = \log_t \tilde{\ve{p}}$ as the natural parameters and the sufficient statistics $\phi(x)_i = \delta_i(x)\,,\,\, i\in [d]$. In this case, we have $G(\ve{\theta}) = 1$. The constrained dual function, on the other hand, is calculated with the additional constraint that the input argument belongs to \acrotem s. This characterization allows an alternative representation in an overparameterized form. The constrained Legendre dual function is defined by imposing the additional constraint that $\tilde{\ve{p}} \,\in \tilde{\Delta}_t^d$,
\begin{equation}
    \begin{split}
        &\check{F}_t^*(\check{\ve{\theta}}) = \sup_{\tilde{\ve{p}} \,\in \tilde{\Delta}_t^d}\big\{\check{\ve{\theta}}\cdot\tilde{\ve{p}} - F_t(\tilde{\ve{p}})\big\}\\
        & = \sup_{\tilde{\ve{p}}\,\in\mathbb{R}^d_+}\big\{\check{\ve{\theta}}\cdot \tilde{\ve{p}} - F_t(\tilde{\ve{p}}) + \lambda_t\,(\sum_{i\in[d]} \ve{p}_i^{1/{t^*}} - 1)\,\big\} ,
    \end{split}
\end{equation}
in which the Lagrange multiplier $\lambda_t$ is used to enforce the equality constraint. This yields the following duality result,
\begin{equation}
\label{eq:thetc}
\check{\ve{\theta}} = \log_t \tilde{\ve{p}} + \lambda_t(\tilde{\ve{p}})\, \tilde{\ve{p}}^{1-t} = \log_t\frac{\tilde{\ve{p}}}{\tilde{\lambda}_t(\tilde{\ve{p}})}\, ,
\end{equation}
where $\tilde{\lambda}_t = 1/\exp_t(\lambda_t)$. In fact, $\lambda_t$ can be written in a closed form using an extension of the results in~\cite{thesis} for dual of a convex function with an equality constraint. The equality constraint $\sum_i \tilde{p}_i^{1/{t^*}} =1$ induces the tangent space $T_{\tilde{\ve{p}}}\tilde{\Delta}_t^d = \{\ve{u}: \tilde{\ve{p}}^{1-t}\cdot \ve{u} = 0\}$ and the constrained dual variable $\check{\ve{\theta}}$ lies on this tangent space
\[
\Mat{P}_t(\tilde{\ve{p}}) = \Mat{I} - \frac{1}{\sum_i \tilde{p}_i^{2-2t}}\, \tilde{\ve{p}}^{1-t}\,(\tilde{\ve{p}}^{1-t})^{\top},\, \text{ and }\,\, \Mat{P}_t\, \check{\ve{\theta}} = \check{\ve{\theta}}\, .
\]
\begin{proposition}
\label{prop:lambda}
The Lagrange multiplier $\lambda_t$ for the overparameterized case can be written as
\begin{equation}
    \label{eq:lambda}
    \lambda_t(\ve{\tilde{p}}) = \log_t\Big[\Big(\frac{\sum_i \tilde{p}^{1-t}_i}{\sum_i \tilde{p}^{2-2t}_i}\Big)^{\frac{1}{1-t}}\Big]\, .
\end{equation}
\end{proposition}
\begin{remark}
Applying the L'Hopital's rule to the limit case $t=1$ results in $\log$-inverse geometric mean,
\[
\lambda_1(\ve{p}) = \log\Big(\big(\prod_i p_i\big)^{-\frac{1}{d}}\Big)\,,
\]
and thus yields the following parameterization for the probability vectors $\ve{p} \in \Delta^d$:
\[
\check{\ve{\theta}} = \log \ve{p} - \frac{1}{d}\sum_i \log p_i\,\ve{1} = \log\frac{\ve{p}}{\big(\prod_i p_i\big)^{\frac{1}{d}}}\, .
\]
\end{remark}
Using Eq.~\eqref{eq:thetc}, we can write the constrained dual convex function and the corresponding link as
\begin{align}
\check{F}^*(\check{\ve{\theta}}) & = \log_t\Big[\big(\sum_i \exp_t^{1/{t^*}}\check{\theta}_i\big)^{t^*}\Big]\,, \label{eq:log-sum-exp}\\
   [\check{f}^*_t(\check{\ve{\theta}})]_i & = \frac{\partial \check{F}^*_t(\check{\ve{\theta}})}{\partial \check{\theta}_i} = \frac{\exp_t\check{\theta}_i}{\big(\sum_i \exp_t^{1/{t^*}}\check{\theta}_i\big)^{t^*}}\,.\label{eq:tempered-softmax}
\end{align}

Eq.~\eqref{eq:log-sum-exp} introduces a tempered variant of the log-sum-exp function along with its tempered softmax link in~\eqref{eq:tempered-softmax}. The following result is an extension of the invariance of the standard softmax along constant vectors.
\begin{proposition}
\label{prop:temp-soft-inv}
The tempered softmax is invariant along the normal vector to the tangent space $\hat{\ve{n}}\, \propto\, \tilde{\ve{p}}^{1-t}$, i.e., $\check{f}^*(\check{\ve{\theta}} + r\,\hat{\ve{n}}) = \check{f}^*(\check{\ve{\theta}})$ for $r\in\mathbb{R}$.
\end{proposition}
Figure \ref{fig:dual} illustrates the relation between the constrained $\check{\ve{\theta}}$ and unconstrained $\ve{\theta} = \log_t \tilde{\ve{p}}$ parameterizations in the overparameterized form. Table~\ref{tab:links} summarizes minimal and overparameterized representations.
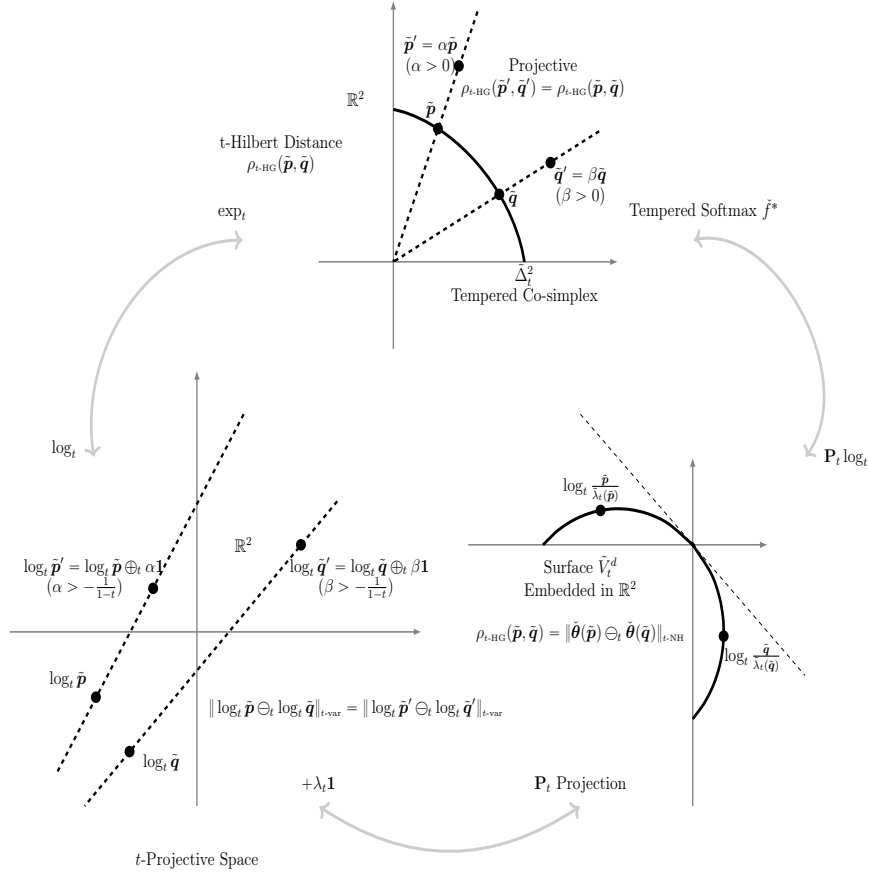
\begin{figure*}[t!]
\begin{center}
\resizebox{0.7\textwidth}{0.7\textwidth}{%
    \begin{tikzpicture}
	\begin{pgfonlayer}{nodelayer}
		\node [style=none] (2) at (-5.25, -3) {};
		\node [style=none] (3) at (0, 5.5) {};
		\node [style=none] (4) at (6, 5.5) {};
		\node [style=none] (5) at (-2, 5.5) {};
		\node [style=none] (6) at (0, 3.5) {};
		\node [style=none] (7) at (0, 11.5) {};
		\node [style=none] (9) at (10, -1) {};
		\node [style=none] (10) at (2, -1) {};
		\node [style=none] (11) at (8, -7) {};
		\node [style=none] (12) at (8, 1) {};
		\node [style=none] (13) at (-5.25, -3) {};
  \node [style=none] (65) at (-4, 6) {}; 
   \node (anch0) at (-4.3, 6.6) [style=none]{$\exp_t$};
		\node [style=none] (66) at (-8, 1) {}; 
  \draw [<->, line width=2, gray!40] (65) to [out=180,in=100] (66);
  \node [style=none] (67) at (-2, -7) {}; 
  \node (anch2) at (-2, -6.5) [style=none]{$+ \lambda_t \ve{1}$};
		\node [style=none] (68) at (5, -7) {}; 
  \node (anch3) at (5, -6.5) [style=none]{$\Mat{P}_t$ Projection};
  \draw [<->, line width=2, gray!40] (67) to [out=-30,in=-150] (68);
  \node (anch1) at (-8.8, 1.2) [style=none]{$\log_t$};
  \node [style=none] (69-a) at (8, 6) {}; 
  \node (anch4) at (8.3, 6.7) [style=none]{Tempered Softmax $\check{f}^*$};
		\node [style=none] (70-a) at (11, 1) {}; 
  \node (anch5) at (12.1, 1.) [style=none]{$\Mat{P}_t \log_t$};
  \draw [<->, line width=2, gray!40] (69-a) to [out=10,in=50] (70-a);
		\node [style=none] (14) at (0.75, -3) {};
		\node [style=none] (15) at (-10.25, -3) {};
		\node [style=none] (16) at (-5.25, -7) {};
		\node [style=none] (17) at (-5.25, 3) {};
  \draw [black,line width=2.] plot [smooth, tension=1.15] coordinates { (0, 9) (2.25, 7.75) (3.5, 5.5) };
		\node [style=none] (20) at (5, 5.5) {};
		\node [style=none] (22) at (5.5, 8.5) {};
		\node [style=none] (23) at (2.25, 11.25) {};
		\node (25) at (-7.95, -4.5) [style=pointbullet]{};
  \node (25-t) at (-8.7, -4.1) [style=none]{$\log_t\tilde{\ve{p}}$};
		\node [style=none] (26) at (-4, 2) {};
		\node [style=pointbullet] (27) at (-7.05, -5.75) {};
  \node (27-t) at (-6.2, -6.) [style=none]{$\log_t\tilde{\ve{q}}$};
		\node [style=none] (28) at (-1.5, 0) {};
		\node [style=pointbullet] (29) at (-6.42, -2) {};
		\node [style=pointbullet] (30) at (-2.47, -1) {};
		\node [style=none] (31) at (-5.25, -8.25) {$t$-Projective Space};
		\node [style=none] (32) at (-8.25, -7) {};
		\node [style=none] (33) at (-9, -6.25) {};
		\node [style=none] (34) at (11, -4) {};
		\node [style=none] (35) at (5, 2) {};
		\node [style=pointbullet] (37) at (1.19, 8.56) {};
  \node [style=none] (37b) at (1., 9) {$\tilde{\ve{p}}$};
		\node [style=pointbullet] (38) at (2.83, 7.05) {};
  \node [style=none] (38b) at (3.2, 6.95) {$\tilde{\ve{q}}$};
		\node [style=none] (39) at (-1, 9.25) {$\mathbb{R}^2$};
		\node [style=pointbullet] (40) at (1.75, 10) {};
		\node [style=none] (41) at (1., 10.5) {$\tilde{\ve{p}}'=\alpha\tilde{\ve{p}}$};
  \node [style=none] (41b) at (1., 10.) {$(\alpha > 0)$};
		\node [style=none] (42) at (5, 7.5) {$\tilde{\ve{q}}'=\beta\tilde{\ve{q}}$};
  \node [style=none] (42b) at (5, 7) {$(\beta > 0)$};
		\node [style=none] (43) at (3.5, 4.7) {Tempered Co-simplex};
		\node [style=pointbullet] (44) at (4.2, 7.78) {};
		\node [style=none] (45) at (4, 10.) {Projective};
		\node [style=none] (46) at (4, 9.5) {$\rho_{\tHG}(\tilde{\ve{p}}', \tilde{\ve{q}}') = \rho_{\tHG}(\tilde{\ve{p}}, \tilde{\ve{q}})$};
		\node [style=none] (47) at (-3, 8.25) {t-Hilbert Distance};
		\node [style=none] (48) at (-3, 7.75) {$\rho_{\tHG}(\tilde{\ve{p}}, \tilde{\ve{q}})$};
		\node [style=none] (49) at (-0.9, -1.5) {$\log_t \tilde{\ve{q}}' = \log_t \tilde{\ve{q}}\oplus_t \beta \ve{1}$};
		\node [style=none] (50) at (-8, -1.5) {$\log_t\tilde{\ve{p}}' = \log_t\tilde{\ve{p}}\oplus_t\alpha \ve{1}$};
		\node [style=none] (51) at (-1, -4.75) {$\Vert \log_t \tilde{\ve{p}}\ominus_t \log_t\tilde{\ve{q}}\Vert_{\tvar} = \Vert \log_t \tilde{\ve{p}}'\ominus_t \log_t\tilde{\ve{q}}'\Vert_{\tvar}$};
  \draw [black, line width=2.] plot [smooth, tension=0.6] coordinates {
  (4.000, -1.000) (4.327, -0.719) (4.409, -0.662) (4.506, -0.599) (4.634, -0.523) (4.793, -0.442) (4.991, -0.358) (5.237, -0.277) (5.538, -0.210) (5.902, -0.173) (6.328, -0.191) (6.810, -0.290) (7.347, -0.513) (8.000, -1.000) (8.000, -1.000) (8.487, -1.653) (8.710, -2.190) (8.809, -2.672) (8.827, -3.098) (8.790, -3.462) (8.723, -3.764) (8.642, -4.009) (8.558, -4.208) (8.476, -4.367) (8.399, -4.498) (8.331, -4.602) (8.253, -4.710) (8.000, -5.000)
 };
		\node [style=none] (20) at (5, 5.5) {};
		\node [style=none] (54) at (5, -1.5) {Surface $\tilde{V}_t^d$ };
		\node [style=none] (55) at (5, -2.) {Embedded in $\mathbb{R}^2$};
		\node [style=none] (56) at (5, -3) {$\rho_{\tHG}(\tilde{\ve{p}}, \tilde{\ve{q}}) = \Vert\check{\ve{\theta}}(\tilde{\ve{p}})\ominus_t\check{\ve{\theta}}(\tilde{\ve{q}})\Vert_{\tNH}$};
		\node [style=pointbullet] (57) at (8.827, -3.098) {}; 
  \node [style=none] (57-b) at (9.6, -3.6) {$\log_t\frac{\tilde{\ve{q}}}{\tilde{\lambda}_t(\tilde{\ve{q}})}$};
		\node [style=pointbullet] (58) at (5.538, -0.210) {}; 
  \node [style=none] (58-b) at (5.3, 0.3) {$\log_t\frac{\tilde{\ve{p}}}{\tilde{\lambda}_t(\tilde{\ve{p}})}$};
  \node [style=none] (59) at (3.5, 5.2) {$\tilde{\Delta}_t^2$};
		\node [style=none] (60) at (-4, -1) {$\mathbb{R}^2$};
  \node [style=none] (61) at (-1., -2.) {($\beta > -\frac{1}{1-t}$)};
  \node [style=none] (62) at (-8.25, -2) {($\alpha > -\frac{1}{1-t}$)};
	\end{pgfonlayer}
	\begin{pgfonlayer}{edgelayer}
		\draw [style=arrow3] (4.center) to (5.center);
		\draw [style=arrow3] (7.center) to (6.center);
		\draw [style=arrow3] (9.center) to (10.center);
		\draw [style=arrow3] (12.center) to (11.center);
		\draw [style=arrow3] (14.center) to (15.center);
		\draw [style=arrow3] (17.center) to (16.center);
		\draw [style=dotted] (28.center) to (32.center);
		\draw [style=dotted] (26.center) to (33.center);
		\draw [style=dashed] (35.center) to (34.center);
		\draw [style=dotted] (23.center) to (3.center);
		\draw [style=dotted] (22.center) to (3.center);
	\end{pgfonlayer}
\end{tikzpicture}
}
\vspace{-0.2cm}
    \caption{Different representations for $t < 1$. \acrotem s reside on a curved tempered co-simplex in 2D (top). The $t$-Hilbert distance is projective, thus invariant along the rays $\bbR_+\tilde{\Delta}_t^d$. The mapping to the unconstrained overparameterized representation is governed via the $\log_t$ function (left), and the corresponding variation $t$-norm distance is invariant with respect to $t$-addition of any constant $> -\frac{1}{1-t}$. The constrained overparameterized representation (right) is a $\Mat{P}_t$ projection of the unconstrained form, and is mapped to the arched shaped surface $\tilde{V}_t^d$ (curved inward for $t < 1$ and outward for $t > 1$). At $t=1$, $\tilde{V}_1^d$ corresponds to the subspace $\{(x, y) \in \bbR^2\,\vert\, y = -x\}$ for standard exponential family.}\label{fig:dual}
    \end{center}
    \vspace{-0.5cm}
\end{figure*}

\section{Tempered Co-simplex Geometry}
We lift the embeddings of the discrete probability distributions in the Hilbert simplex geometry to the embeddings of \acrotem s. The construction involves extending Hilbert's projective metric to a tempered generalization. This tempered generalization is derived from first principles using, at the core, a generalized notion of Riemann sum, provided in Section \ref{subsec-tempfunk}. 

\subsection{Tempered Funk and Hilbert Distances}\label{sub-t-hil}
Consider an open bounded convex set $\Omega$. We first consider a tempered generalization of the Funk distance ($t$-Funk distance, for short) between two points $\ve{r}, \ve{s} \in \Omega$:
\begin{equation}
    \label{eq:t-funk-dist}
    \rho_{\tFD}^{\Omega}(\ve{r}, \ve{s}) = \begin{cases} \log_t \frac{\Vert \ve{r} - \bar{\ve{s}}\Vert}{\Vert \ve{s} - \bar{\ve{s}}\Vert} & \ve{r} \neq \ve{s}\\
    0 & \ve{r} = \ve{s}\end{cases}\,.
\end{equation}
Here, $\bar{\ve{s}}$ denotes the intersection of the affine ray $R(\ve{r}, \ve{s})$ originating from $\ve{r}$ and passing through
$\ve{s}$ with the domain boundary $\partial \Omega$. Evidently, the tempered Funk distance is also an asymmetric dissimilarity measure $\rho_{\tFD}^{\Omega}(\ve{r}, \ve{s}) \neq \rho_{\tFD}^{\Omega}(\ve{s}, \ve{r})$ but satisfies the triangular inequality for certain cases (see Proposition~\ref{prop:dist-properties} below). We provide a motivation for deriving the $t$-Funk distance from a tautological Finsler structure by replacing the integration for calculating the length of curves with a notion of $t$-integration in the appendix.

\begin{figure*}
\vspace{-0.3cm}
\centering

\begin{tabular}{cccc}
\includegraphics[width=0.2\textwidth]{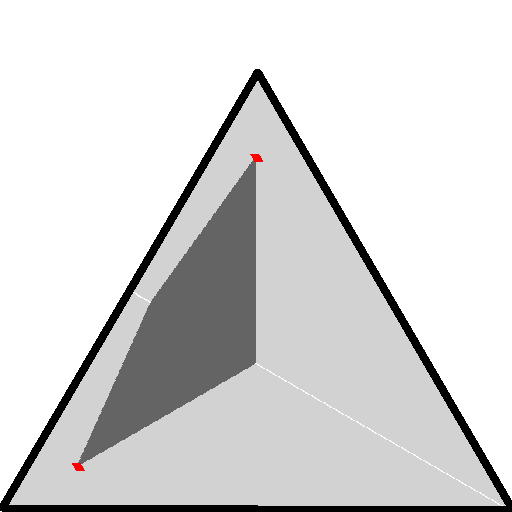}&
\includegraphics[width=0.2\textwidth]{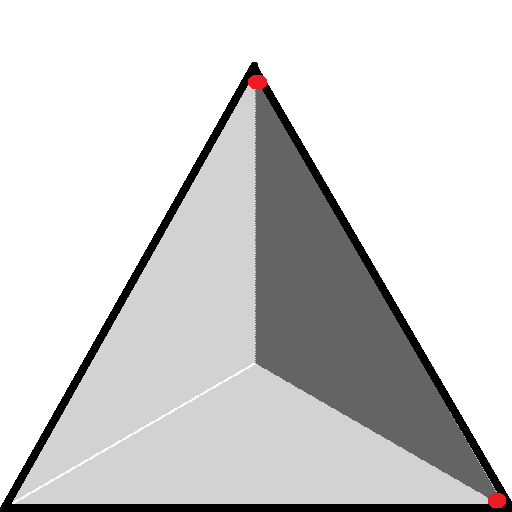}&
\includegraphics[width=0.2\textwidth]{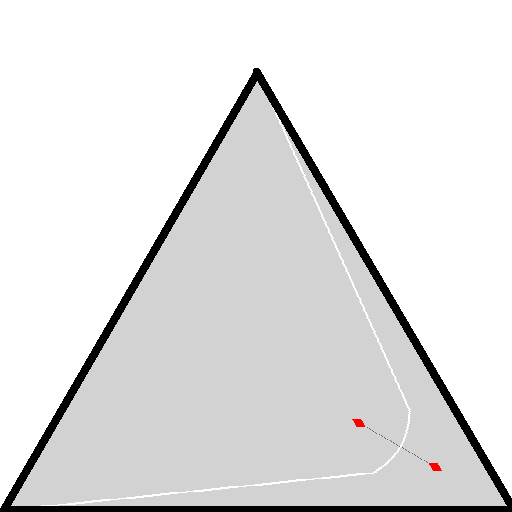}&
\includegraphics[width=0.2\textwidth]{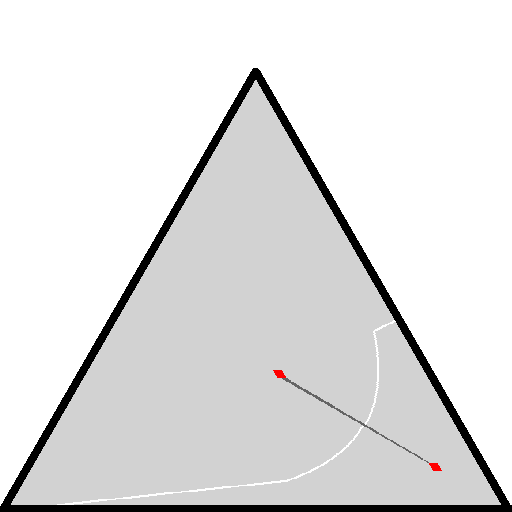}
\end{tabular}
\vspace{-0.3cm}
\caption{Voronoi bisectors (white) and triangle equality regions (dark gray) between points $p$ and $q$ with respect to Hilbert simplex distance for (from left to right)
(a) $\ve{p}=(0.1,0.8,0.1)$ and $\ve{q}=(0.8,0.1,0.1)$,
(b) $\ve{p}=(\eps,\eps,1-2\eps)$ and $\ve{q}=(\eps,1-2\eps,\eps)$,
(c) $\ve{p}=(0.1,0.1,0.8)$ and $\ve{q}=(0.2,0.2,0.6)$
and (d) $\ve{p}=(0.1,0.1,0.8)$ and $\ve{q}=(0.3,0.31,0.39)$.
\label{fig:biti}
}
\vspace{-0.4cm}
\end{figure*}

As the standard Hilbert distance can be viewed as a symmetrization of Funk distance, we define the tempered Hilbert distance (or $t$-Hilbert distance) between $\ve{r}, \ve{s} \in \Omega$ similarly but as a \emph{$t$-symmetrization} of the Funk distance:
\begin{equation}
    \label{eq:t-hilbert-dist}
    \begin{split}
    \rho_{\tHG}^{\Omega}(\ve{r}, \ve{s}) & = \rho_{\tFD}^{\Omega}(\ve{r}, \ve{s}) \oplus_t \rho_{\tFD}^{\Omega}(\ve{s}, \ve{r})\\
    & = \begin{cases} \log_t\frac{\Vert \ve{r} - \bar{\ve{s}}\Vert \Vert \ve{s} - \bar{\ve{r}}\Vert}{\Vert \ve{r} - \bar{\ve{r}}\Vert \Vert \ve{s} - \bar{\ve{s}}\Vert} & \ve{r} \neq \ve{s}\\
    0 & \ve{r} = \ve{s}\end{cases}\,,
    \end{split}
\end{equation}
where $\bar{\ve{r}} \in \partial \Omega$ is the boundary point defined similarly as above. Indeed, $t$-Hilbert distance is independent of the underlying norm of the cross-ratio because all 1d norms are positive scalars of the absolute value. Thus, $t$-Hilbert distance only depends on the $1$-D interval domain $\Omega_{\ve{r}\ve{s}} = \Omega \cap (\ve{r}\ve{s})$. 
Note that we have
\begin{eqnarray}
\rho_{\tHG}^{\Omega}(\ve{r}, \ve{s})=
\log_t \exp \rho_{1-\mathrm{HG}}^{\Omega}(\ve{r}, \ve{s}). \label{eqpropt-HG}
\end{eqnarray}
Since $\log_t$ is a strictly monotone function, we deduced that all $t$-Hilbert Voronoi diagrams coincide.
\begin{equation*}
\ve{\gamma}_t^\Omega(\ve{r}, \ve{s})=\{ \ve{t}\in\Omega : \rho_{\tHG}^{\Omega}(\ve{r}, \ve{t}) \oplus_t
\rho_{\tHG}^{\Omega}(\ve{t}, \ve{s})=\rho_{\tHG}^{\Omega}(\ve{r}, \ve{s})
\}\,
\end{equation*}
denotes a \emph{$t$-geodesics}. 
Notably, geodesics in $t$-Hilbert geometry are no longer straight (Euclidean) lines, but rather $t$-geodesics,
for all $\ve{t} \in [\ve{r}\ve{s}]$ where $[\ve{r}\ve{s}]$ is the closed line segment connecting $\ve{r}$ and $\ve{s}$. Theorem~\ref{thm:contr} extends the contraction of the $t$-Hilbert distance for positive linear maps beyond $t=1$.
\begin{proposition}
\label{prop:t-hilbert-metric}
The $t$-Hilbert distance $\rho_t^\Omega(p,q)$ is a metric distance for
$t\geq 1$.
\end{proposition}
\begin{theorem}
\label{thm:contr}
Let $\calC$ be a closed pointed cone in $\bbR^d$. Then every positive linear map $A$ is a contraction with respect to $\rho_{\tHG}^{\calC}$. Additionally, $\kappa_{\tHG}(A) \geq \kappa_{\HG}(A)$ where $\kappa_{\tHG}(A)$ and $\kappa_{\HG}(A)$ are the contraction ratios with respect to $\rho_{\tHG}^{\calC}$ and $\rho_{\HG}^{\calC}$, respectively.
\end{theorem}


\begin{figure*}[t!]
\vspace{-0.6cm}
\begin{center}
    \subfigure{\includegraphics[width=0.22\linewidth]{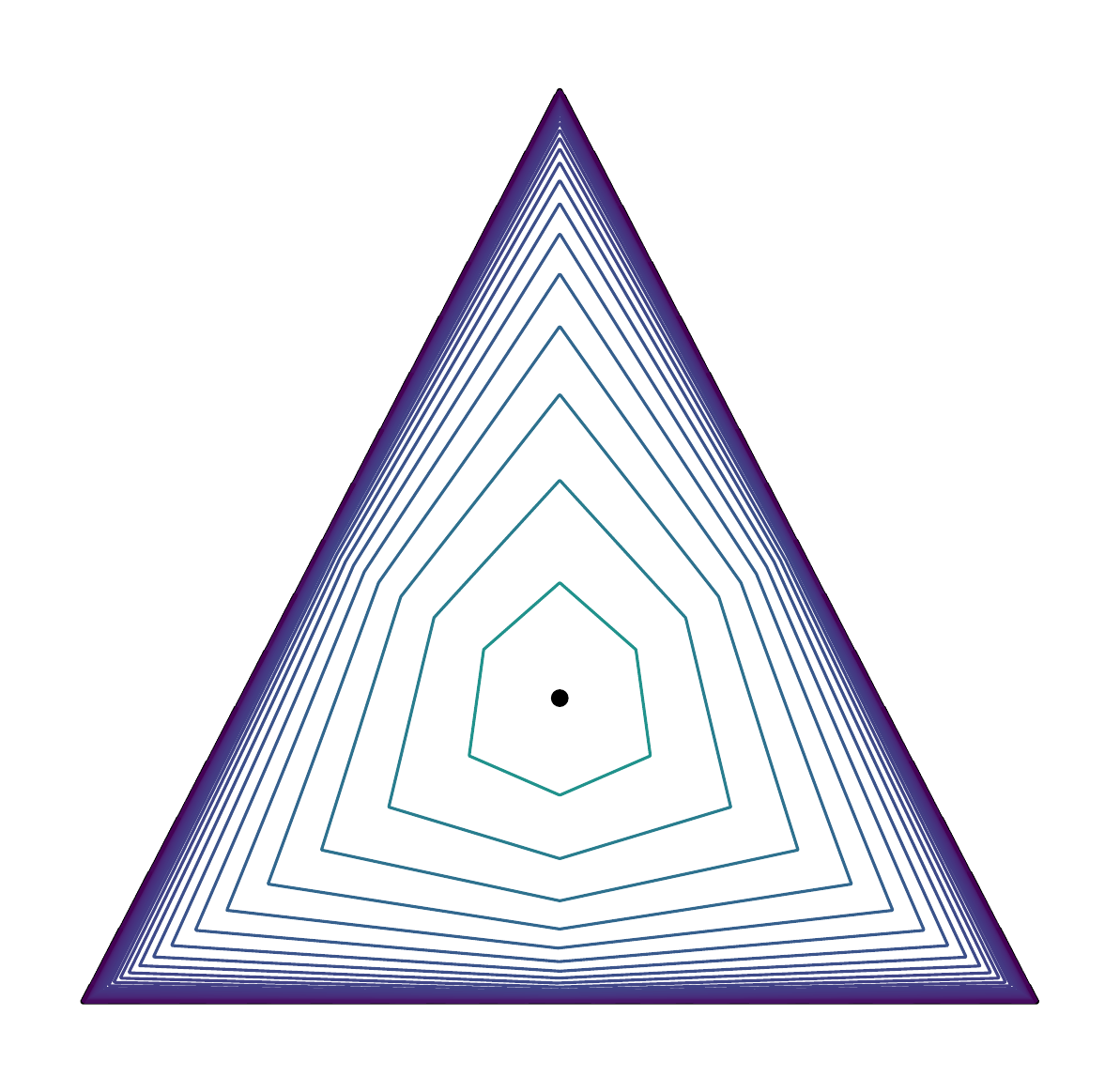}}
    \subfigure{\includegraphics[width=0.22\linewidth]{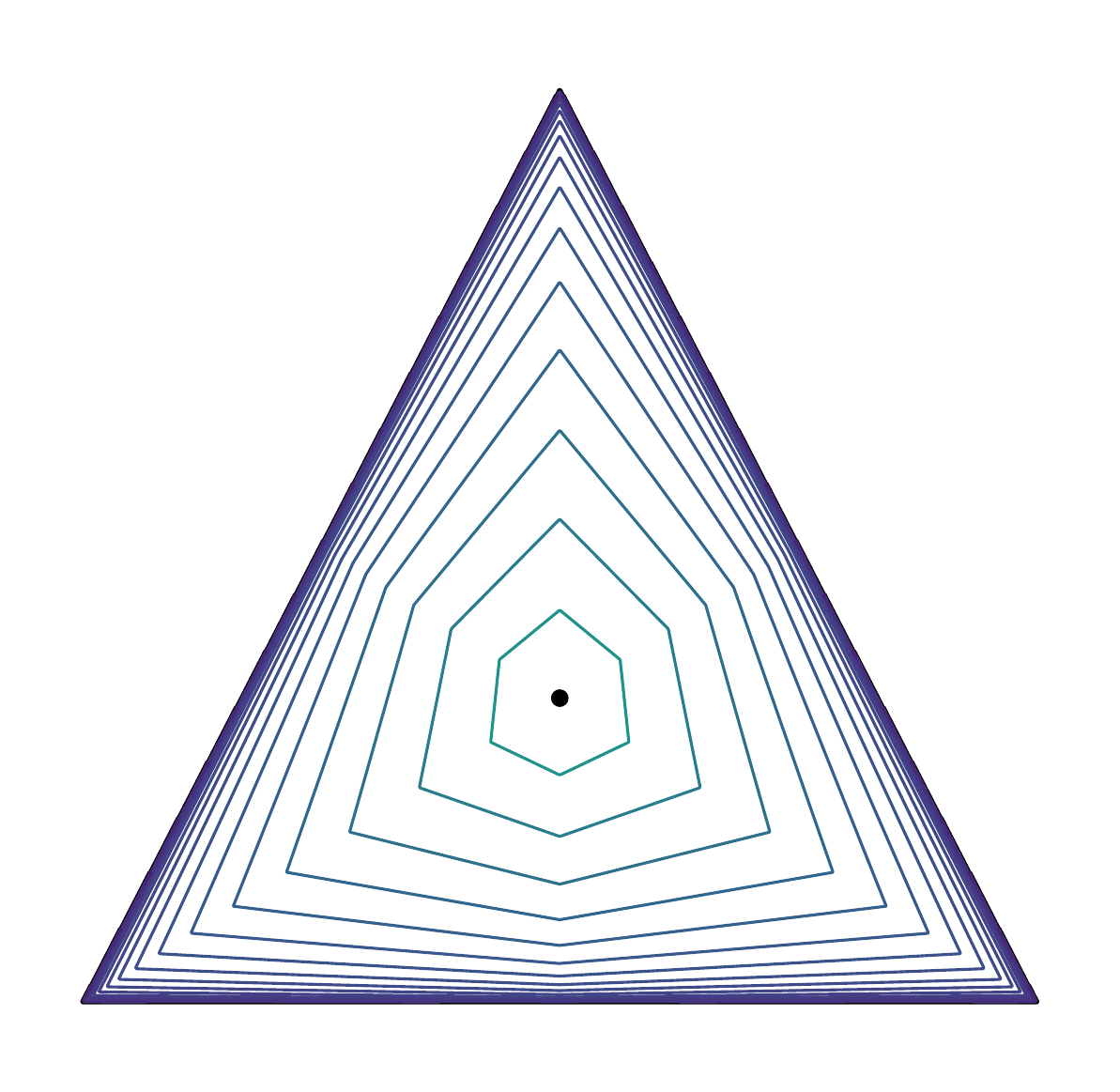}}
    \subfigure{\includegraphics[width=0.22\linewidth]{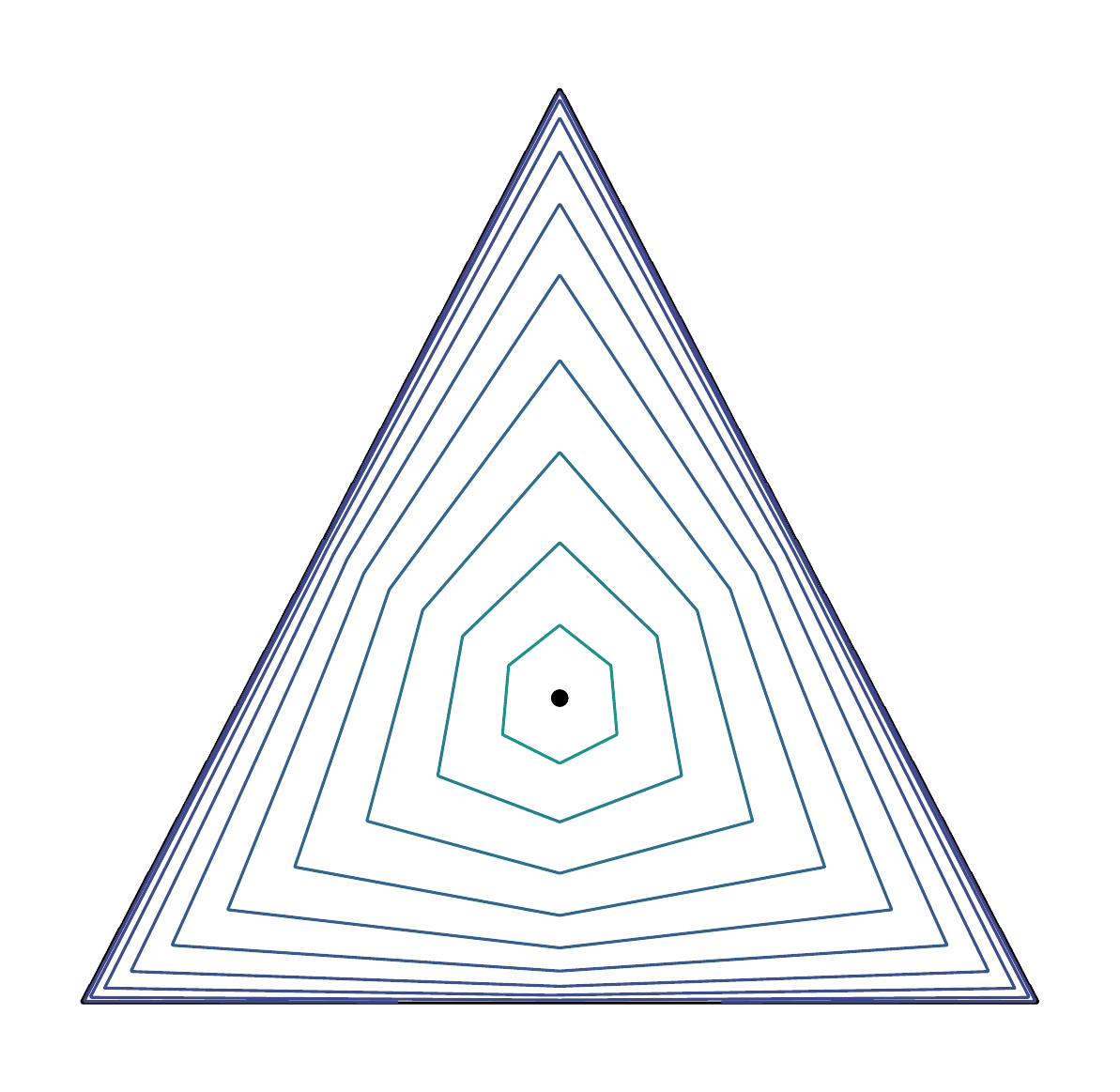}}\\\vspace{-0.5cm}
    \subfigure[$t=0.6$]{\addtocounter{subfigure}{-3}\includegraphics[width=0.22\linewidth]{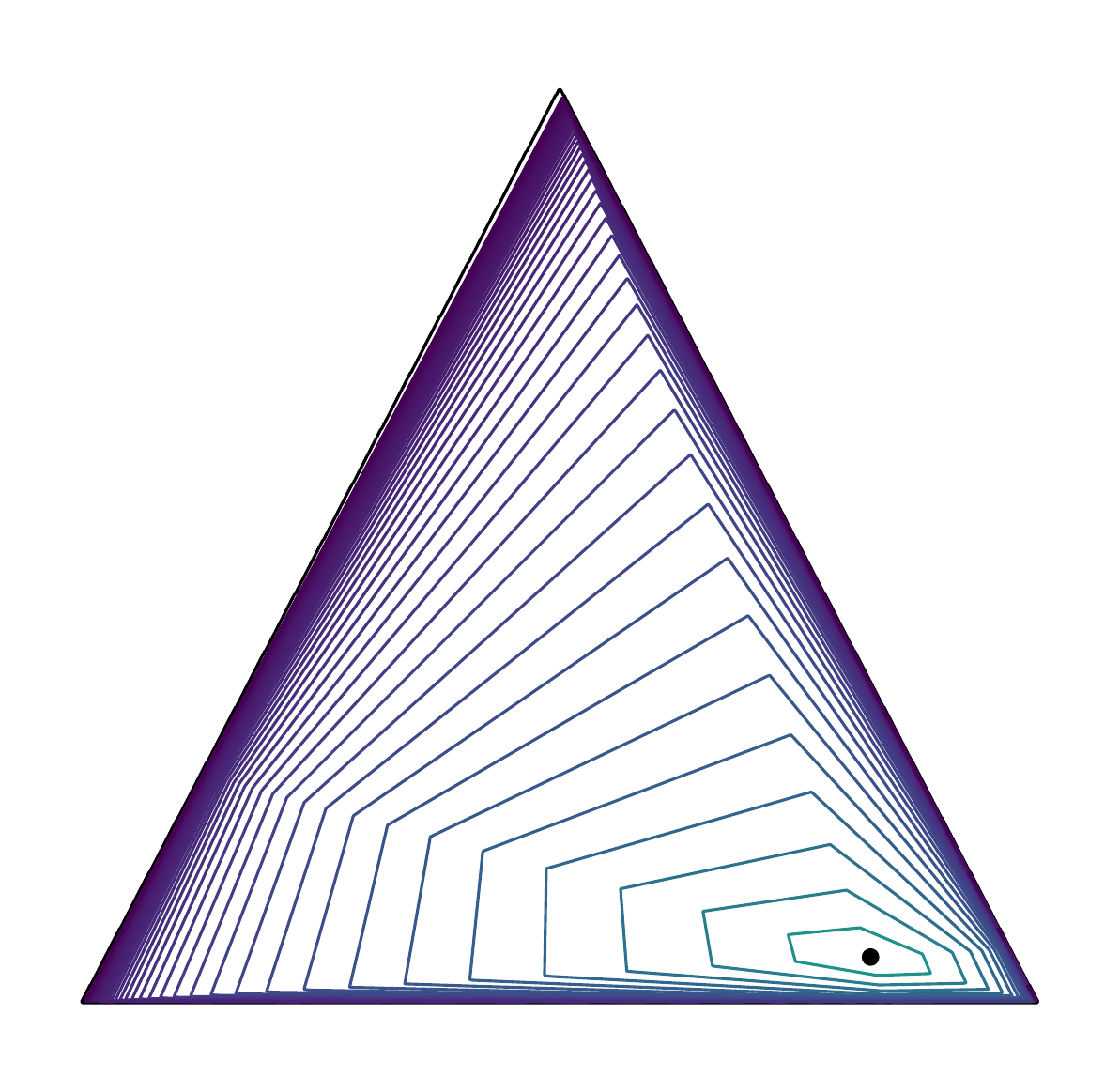}}
    \subfigure[$t=1.0$]{\includegraphics[width=0.22\linewidth]{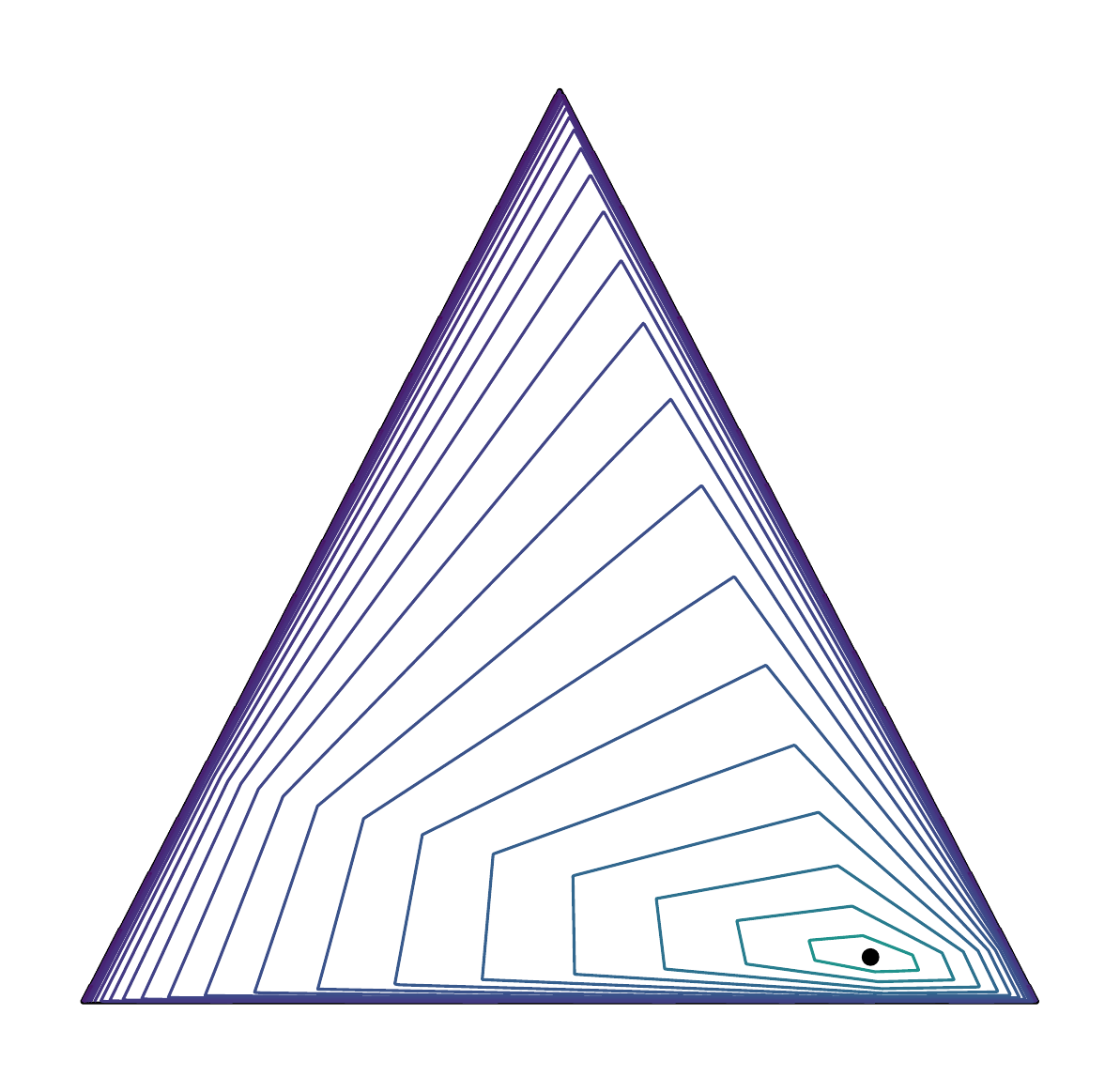}}
    \subfigure[$t=1.2$]{\includegraphics[width=0.22\linewidth]{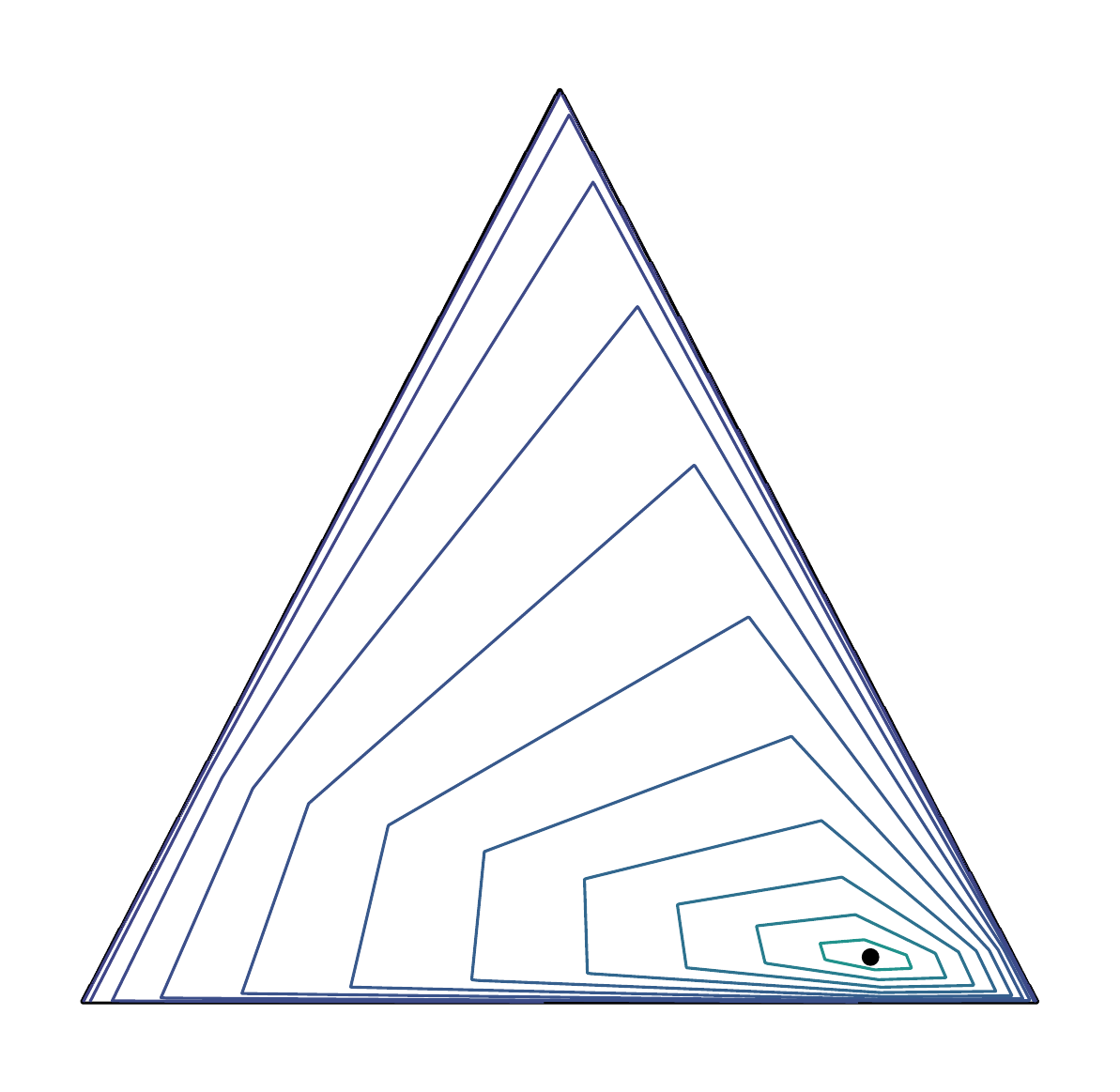}}
    \vspace{-0.3cm}
    \caption{Balls of different radii with respect to the $t$-Hilbert distance shown on the simplex for different $t$. The center is shown with a black dot. Darker colors indicate larger distances.
}
    \label{fig:hilbert-simplex}
    \end{center}
    \vspace{-0.4cm}
\end{figure*}

\subsection{$t$-Funk and $t$-Hilbert Distances from First Principles}\label{subsec-tempfunk}
  
All the material developed in Section \ref{sub-t-hil} follows naturally from a \textit{single} change in the path that derives them from a tautological Finsler structure \citep{HandbookHilbert-2014}, namely the classical notion of \textit{Riemann sums}. In our case, we define Riemann $t$-sums using $t$-summations ($\oplus_t$) instead of $+$. 
At the core lies the use of the $t$-addition, $\oplus_t$ to define Riemann $t$-sums. For any $f$ defined over a closed interval $[a,b]$ and a division $\Delta$ of this interval using $n+1$ reals $x_0 \defeq a <x_1 < ... < x_{n-1} < x_n \defeq  b$, we let
\begin{eqnarray*} 
S_\Delta(f) & \defeq & (\mbox{\Large $\oplus$}_t)_{i=1}^n (x_i - x_{i-1}) f(\xi_i)
\end{eqnarray*}
define the Riemann $t$-sum of $f$ over $[a,b]$ using $\Delta$, with $\xi_i \in ]x_{i-1}, x_i[, \forall i$. As $n$ increases, we see that the terms factoring monomials in $(1-t)$ of degree $i$ become akin to general volumes in dimension $i+1$: for $i=0$, we get ordinary unit lengths (what Riemann sums integrate), for $i=1$, we get unit surfaces, and so on. 
Passing to the limit, we get a notion of $t$-integration, noted {\tiny $\riemannint{t}{.}{.}$} (for $t=1$, this is just $\int_.^.$), which we use to define the $t$-length of a curve. The $t$-Funk distance is then induced via the $t$-length of a ray connecting two points in a tautological Finsler structure. Because the construction is lengthy, details are provided in Appendix~\ref{sec-temp-Funk}.

\subsection{$t$-Hilbert Co-simplex Distance}\label{sub-t-hil-co}
The co-simplex $\tilde{\Delta}_t^d$ is non-convex for $t\neq 1$, and naturally, we cannot define tempered Funk and Hilbert distance to this set. However, we can apply these distances directly on the simplex by lifting each pair of \acrotem s $\tilde{\ve{p}}, \tilde{\ve{q}} \in \tilde{\Delta}_t^d$ to its corresponding co-density $\ve{p}, \ve{q} \in \Delta^d$ via the following lemma.
\begin{lemma}
\label{lem:tstar-dist}
\begin{align*}
   t^*\rho^{\Delta^d}_{\tsFD}(\tilde{\ve{p}}^{1/t^*}\!\!, \tilde{\ve{q}}^{1/t^*}) =  
   \log_t \max_i \frac{\tilde{p}_i}{\tilde{q}_i}\,.
\end{align*}
\end{lemma}
We thus define the tempered Funk distance on the co-simplex as $ \rho_{\tFD}(\tilde{\ve{p}}, \tilde{\ve{q}}) = t^*\rho^{\Delta^d}_{\tsFD}(\tilde{\ve{p}}^{1/t^*}\!\!, \tilde{\ve{q}}^{1/t^*})$ and consequently, the tempered Hilbert on the co-simplex as its $t$-symmetrization:

\begin{equation}
    \label{eq:t-hilbert-dist-simplex}
    \begin{split}
    \rho_{\tHG}(\tilde{\ve{p}}, \tilde{\ve{q}}) & = \rho_{\tFD}(\tilde{\ve{p}}, \tilde{\ve{q}}) \oplus_t \rho_{\tFD}(\tilde{\ve{q}}, \tilde{\ve{p}}) = \log_t\max_i\frac{\tilde{p}_i}{\tilde{q}_i}\max_i\frac{\tilde{q}_i}{\tilde{p}_i} = \log_t \frac{\max_i\frac{\tilde{p}_i}{\tilde{q}_i}}{\min_i\frac{\tilde{p}_i}{\tilde{q}_i}}
    \,.
    \end{split}
\end{equation}

Figure~\ref{fig:biti} illustrates the Voronoi bisector $\Bi(\ve{p},\ve{q})=\{\ve{x}\in\Delta^d \st \rho_\HG(\ve{p},\ve{x})=\rho_\HG(\ve{q},\ve{x})\}$ (studied in~\citep{HilbertCG-2017,HilbertVoronoi-2023}) and the triangle equality region $R(\ve{p},\ve{q})=\{\ve{x}\in\Delta^d \st \rho_\HG(\ve{p},\ve{x})+\rho_\HG(\ve{x},\ve{q})=\rho_\HG(\ve{p},\ve{q})\}$ ($(d-1)$-dimensional counterpart of a geodesic) for a pair of points $\ve{p}$ and $\ve{q}$ (red) with respect to the Hilbert simplex distance.
Observe  that when $\ve{p}$ and $\ve{q}$ are collinear with a simplex vertex (Figure~\ref{fig:biti}~(c)), $R(\ve{p},\ve{q})$ is 1D. The Voronoi bisectors are identical for the $t$-Hilbert co-simplex distance as the distance is monotonic with respect to the Hilbert simplex distance (see Theorem~\ref{thm:contr}). However, the triangle equality regions correspond to $t$-triangle equality where $+$ is replaced with $\oplus_t$.

The following proposition summarizes the properties of the $t$-Hilbert distance.
\begin{proposition}
\label{prop:dist-properties}
    $\rho_{\tHG}$ satisfies the following properties:
    \begin{itemize}
    \item [(i)] (projective distance) $\rho_{\tHG}(\tilde{\ve{p}}, \tilde{\ve{q}}) = 0$ iff $\tilde{\ve{p}} = K\cdot\tilde{\ve{q}}$ for some $K > 0$.
    \item [(ii)] (symmetry) $ \rho_{\tHG}(\tilde{\ve{p}}, \tilde{\ve{q}}) =  \rho_{\tHG}(\tilde{\ve{q}}, \tilde{\ve{p}})$.
    \item [(iii)] (triangle inequality) $\rho_{\tHG}(\tilde{\ve{p}}, \tilde{\ve{r}}) \leq \rho_{\tHG}(\tilde{\ve{p}}, \tilde{\ve{q}}) + \rho_{\tHG}(\tilde{\ve{q}}, \tilde{\ve{r}})$ for $1 \leq t < 2$.
    \item [(iv)] ($t$-triangle inequality) $\rho_{\tHG}(\tilde{\ve{p}}, \tilde{\ve{r}}) \leq \rho_{\tHG}(\tilde{\ve{p}}, \tilde{\ve{q}}) \oplus_t \rho_{\tHG}(\tilde{\ve{q}}, \tilde{\ve{r}})$ for $t \in \mathbb{R}$.
      \end{itemize}
    \end{proposition}

In addition, the $t$-Hilbert distance satisfies a notion of information monotonicity for any coarse-graining of the input arguments. However, the reduction needs to be performed on the corresponding co-densities, as outlined in the following lemma.
\begin{lemma}[Monotonicity]
\label{lem:monotone}
Let $\tilde{\ve{p}}, \tilde{\ve{q}} \in \Delta^{m}_t$. Let $\tilde{\ve{p}}_r = ((\tilde{p}_1^{1/{t^*}} + \tilde{p}_2^{1/{t^*}})^{t^*}, \cdots, \tilde{p}_d)$ and $\tilde{\ve{q}}_r = ((\tilde{q}_1^{1/{t^*}} + \tilde{q}_2^{1/{t^*}})^{t^*}, \cdots, \tilde{q}_d)$ denote their coarse-grained points on $\tilde{\Delta}^d_t$. We have
\begin{equation*}
    \label{eq:t-proj-monotone}
   \rho_{\tFD}(\tilde{\ve{p}}_r, \tilde{\ve{q}}_r) \leq \rho_{\tFD}(\tilde{\ve{p}}, \tilde{\ve{q}})\,.
\end{equation*}
\end{lemma}
We can extend Lemma~\ref{lem:monotone} to any partitioning $\mathcal{X} = \{X_1, \ldots, X_m\}$ of $\{1, \ldots, d\}$ where $m \leq d$ such that $\tilde{p}_{\vert\mathcal{X}}[i] = (\sum_{j \in X_i} \tilde{p}_j^{1/t^*})^{t^*}$. We say the distance $D$ is \emph{$t$-information monotone} if $D(\tilde{\ve{p}}_{\vert\mathcal{X}}, \tilde{\ve{q}}_{\vert\mathcal{X}}) \leq D(\tilde{\ve{p}}, \tilde{\ve{q}})$.
\begin{theorem}
The tempered Funk distance $\rho_{\tFD}$ and the tempered Hilbert distance $\rho_{\tHG}$ in $\tilde{\Delta}_t^d$ satisfy the $t$-information monotonicity.
\end{theorem}
The proof follows by iteratively applying Lemma~\ref{lem:monotone} and using the definition of the tempered Hilbert distance~\eqref{eq:t-hilbert-dist-simplex}.

\begin{figure*}[t!]
\vspace{-0.4cm}
\begin{center}
    \subfigure[$100$ random points in $\bbR^{500}$]{\includegraphics[width=0.25\linewidth]{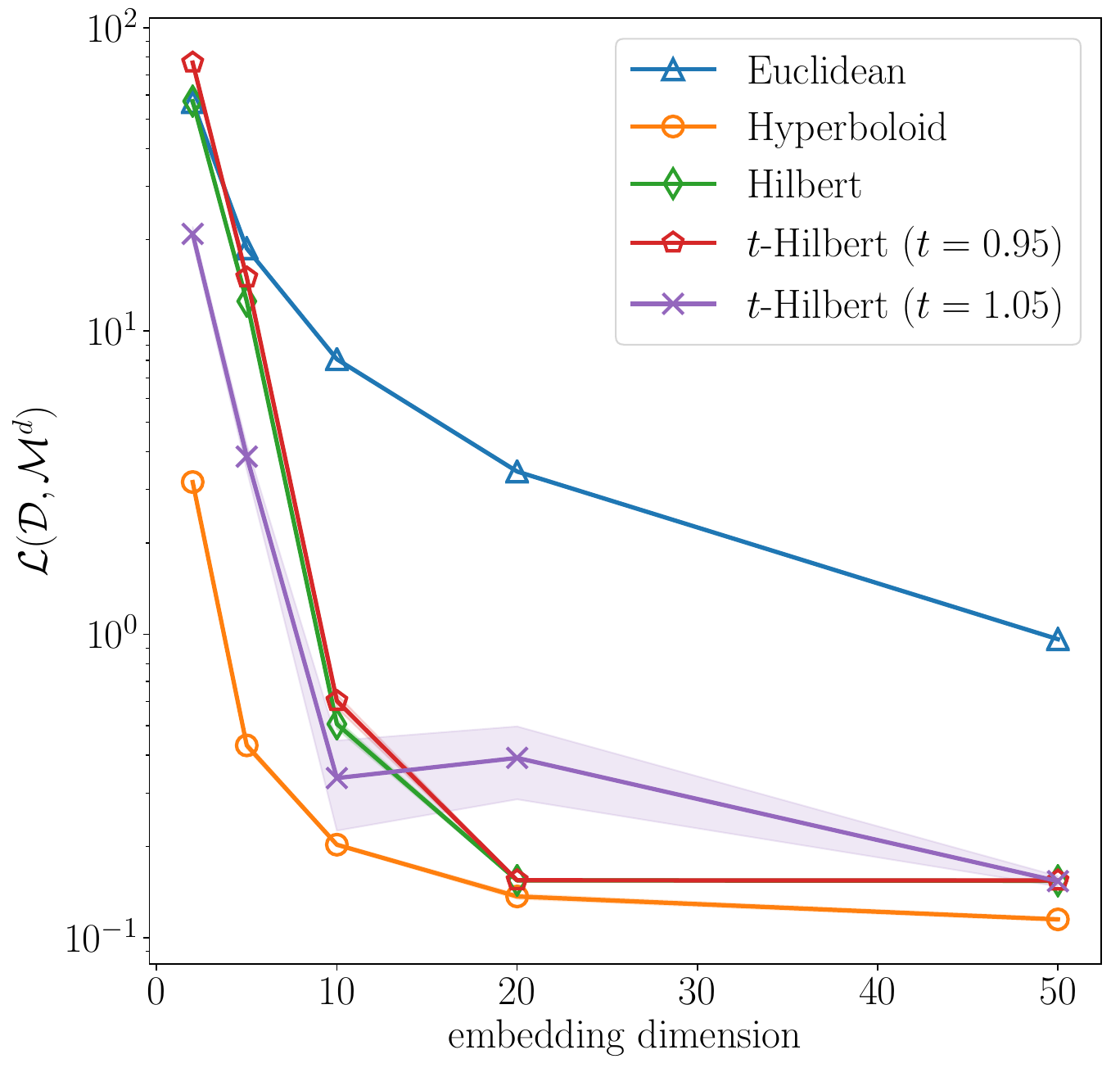}}
    \subfigure[Erd\"os-R\'enyi graphs ($n=200$, $p=0.5$)]{\includegraphics[width=0.265\linewidth]{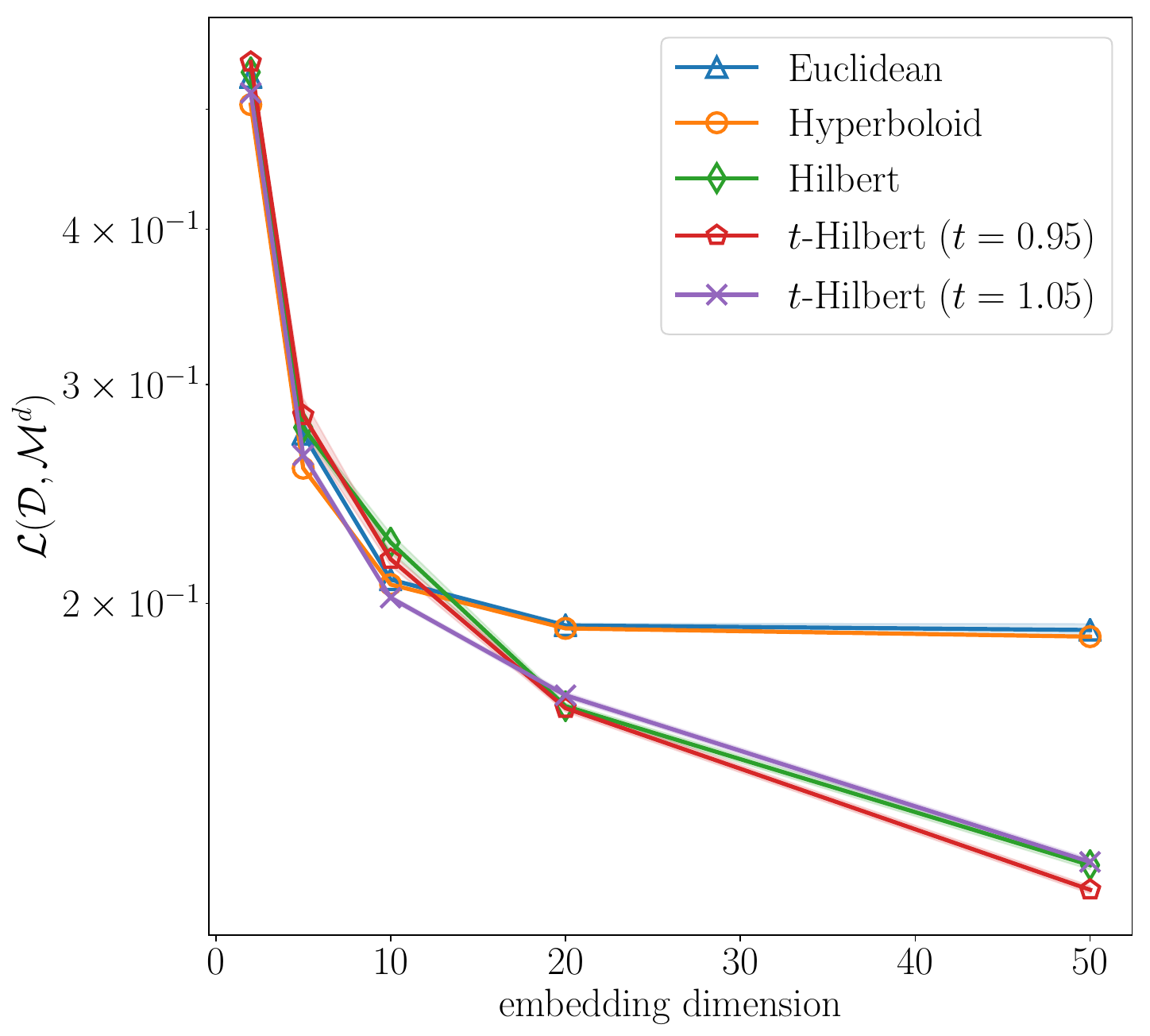}}
    \subfigure[Baraba\'asi-Albert graphs ($n=200$, $m=2$)]{\includegraphics[width=0.25\linewidth]{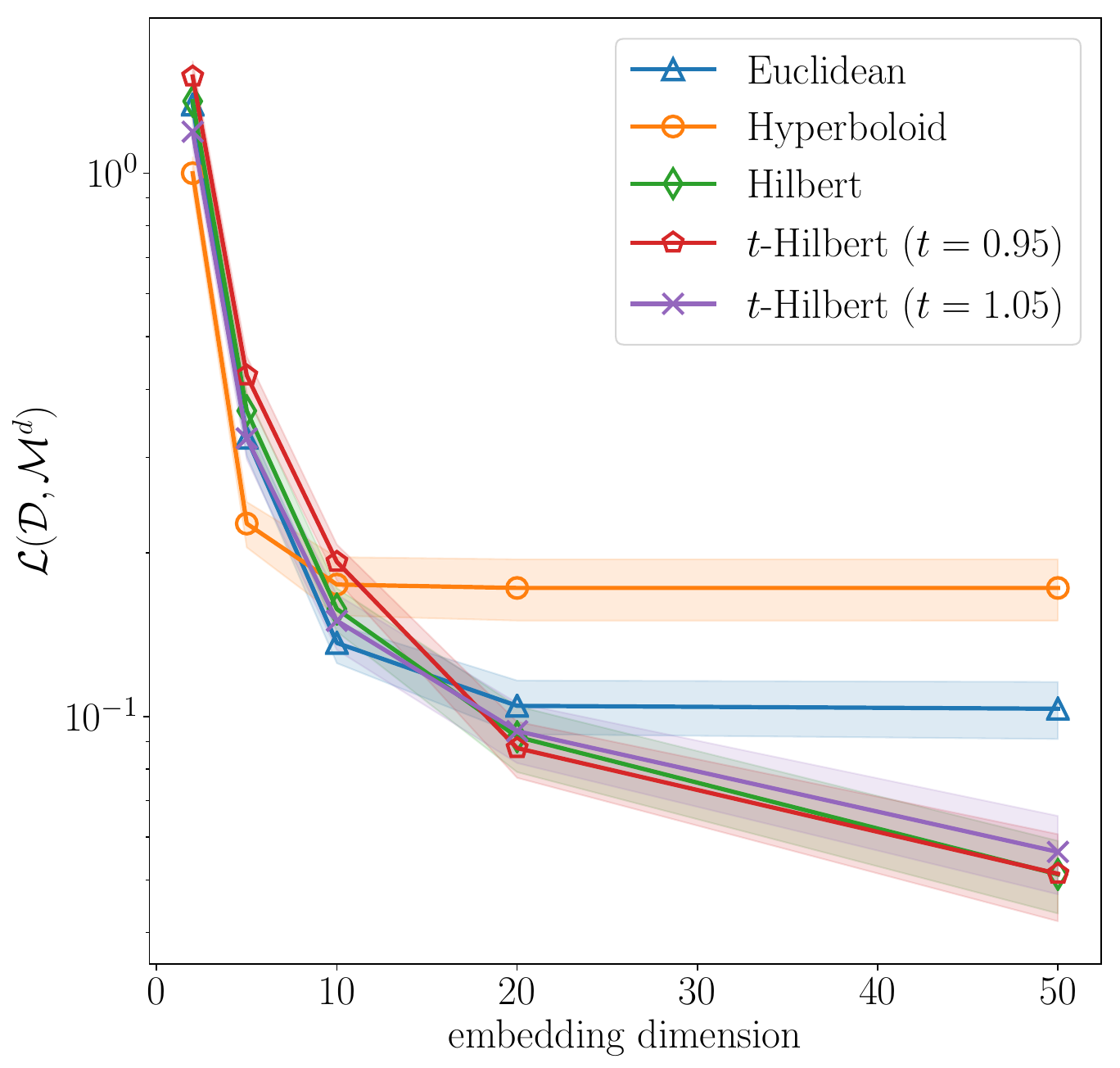}}
    \vspace{-0.2cm}
    \caption{Embedding loss across different embedding dimensions on three datasets.}
    \label{fig:emb}
    \end{center}
    \vspace{-0.5cm}
\end{figure*}

\subsection{Non-linear Embeddings of \acrotem s}
We show an isometry of the co-simplex of \acrotem s to a surface with a corresponding distance $(\tilde{V}^d_t, \|\cdot\|_{\tNH})$. Let $\tilde{V}^d_t = \{\ve{v} \in \mathbb{R}^d\,\vert\, \ve{v}\cdot \check{f}^*(\ve{v})^{1-t} = 0\}$ be the surface of $d$-dimensional vectors that are orthogonal to the normal of the co-simplex at the corresponding \acrotem\, given by the tempered softmax link~\eqref{eq:tempered-softmax}. This surface is the generalization of the linear vector space defined for the Hilbert simplex geometry at $t=1$. We define the distance $\|\cdot\|_{\tNH}$ in $\tilde{V}^d_t$ by first introducing the ball of radius $r$ as
\begin{equation}
    \label{eq:ball}
    B^r_{\tilde{V}^d_t} = \{\ve{u} \in \mathbb{R}^d: \vert u_i \ominus_t u_j \vert \leq r, \forall i \neq j\}\,.
\end{equation}
The distance between $\ve{v}, \ve{v}' \in \tilde{V}^d_t$ is then defined as
\begin{equation}
    \label{eq:t-nh}
\begin{split}
    \rho_{\tNH}(\ve{v}, \ve{v}') & = \Vert \ve{v} \ominus_t \ve{v}'\Vert_{\tNH} = \inf\{\tau:\, \ve{v}\ominus_t\ve{v}' \in B^{\tau}_{\tilde{V}^d_t}\}\,.
    \end{split}
\end{equation}
We can now establish our first isometry result for discrete \acrotem s.
\begin{theorem}
\label{thm:unconst}
The space of discrete \acrotem s with the $t$-Hilbert distance $(\tilde{\Delta}_t^d, \rho_{\tHG})$ is isometric to $(\tilde{V}^d_t, \rho_{\tNH})$ via the constrained overparameterized representation~\eqref{eq:thetc}.
\end{theorem}
Next, we introduce an extension of the variation semi-norm, defined in~\citet{NLEHilbert-2023}. Let
\begin{equation}
    \Vert \ve{x}\Vert_{\tvar} = \max_i x_i \ominus_t \min_i x_i\,.
\end{equation}
The function $\Vert \cdot\Vert_{\tvar}$ is positive definite and satisfies $t$-triangle inequality, but does not satisfy absolute homogeneity. Moreover, similar to its $t=1$ base case, $\Vert r\, \ve{1}\Vert_{\tvar} = 0$ for any $r \in \mathbb{R}$ (see the appendix). Our second result establishes an isometry for the discrete \acrotem s to $\mathbb{R}^d$ via the semi-norm $\|\cdot\|_{\tvar}$.
\begin{theorem}
\label{thm:const}
The space of discrete \acrotem s with the $t$-Hilbert distance $(\tilde{\Delta}_t^d, \rho_{\tHG})$ is isometric to $(\mathbb{R}^d, \rho_{\tvar})$ via the unconstrained overparameterized representation mapping given by $\ve{\theta} = \log_t \tilde{\ve{p}} \in \mathbb{R}^d$ where
\begin{equation}
    \rho_{\tvar}(\ve{\theta}, \ve{\theta}') = \Vert \ve{\theta} \ominus_t \ve{\theta}' \Vert_{\tvar}\,.
\end{equation}
\end{theorem}
\begin{remark}
In fact, we have
\[
\rho_{\tvar}(\ve{\theta}, \ve{\theta}') = \rho_{\tvar}(\check{\ve{\theta}}, \check{\ve{\theta}}') = \rho_{\tNH}(\ve{\theta}, \ve{\theta}')\,,
\]
and the two distances are interchangeable for the construction.
\end{remark}

\subsection{Differentiable Approximations}
The $\max$ operator in the $t$-Funk and $t$-Hilbert distance is non-differentiable. Similar to \citet{NLEHilbert-2023}, we explore differentiable approximations of the distances via smoothening the $\max$ function. Our approximation is based on the tempered log-sum-exp function
\begin{equation}
\label{eq:lse}
    \text{LSE}_{t}(\ve{x}, T) = \frac{1}{T}\, \log_t \sum_i \exp_t(T x_i)\,,
\end{equation}
using which we define the differentiable $t$-Funk and $t$-Hilbert distance with smoothing factor $T$ via the $\text{LSE}_{t}$ approximation by setting $x_i = \log_t \frac{\tilde{p}_i}{\tilde{q}_i}$ as
\begin{align*}
    \rho_{\tFDdiff}(\tilde{\ve{p}}, \tilde{\ve{q}}, T) & = \text{LSE}_{t}(\log_t\frac{\tilde{\ve{p}}}{\tilde{\ve{q}}}, T)\,,\\
    \rho_{\tHGdiff}(\tilde{\ve{p}}, \tilde{\ve{q}}, T) & = \rho_{\tFDdiff}(\tilde{\ve{p}}, \tilde{\ve{q}}, T) \oplus_t \rho_{\tFDdiff}(\tilde{\ve{q}}, \tilde{\ve{p}}, T)\,.
\end{align*}
We delegate the approximation error of the differentiable approximation as well as further experimental analysis to the appendix. The balls of various radii with respect to $\rho_{\tHGdiff}$ for different smoothing factors $T$ are shown in Figure~\ref{fig:diff-simplex}.
\subsection{Comparing Different Geometries}
We compare the representation quality of different geometries for embedding a set of points from three different datasets: i) randomly sampled points in $\bbR^{500}$, ii) Erd\"os-R\'enyi graphs, and iii) Baraba\'asi-Albert graphs. The datasets are generated according to~\citet{NLEHilbert-2023}, and the details are given in the appendix. We plot the approximation loss in Figure~\ref{fig:emb}. We observe that the $t$-Hilbert distance provides some advantage compared to the other geometries for slightly larger or smaller values of $t$ than one.

\section{Conclusion}


Discrete $t$-tempered exponential measures (TEMs) $\tilde{p}$ are positive measures  obtained from maximum entropy 
defined according to $t$-logarithms/$t$-exponentials~\citep{amid2023clustering}.
TEMs form a family $\tilde\Delta_t^d$ 
with normalized codensities $\tilde{p}^{2-t}$ falling inside the probability simplex $\Delta^d$.
In this work, we first presented the minimal and overparameterized constrained/unconstrained parameterizations of discrete TEMs.
We then described their embeddings   using a generalization of Hilbert's geometry
with corresponding $t$-Hilbert distance  satisfying the information monotonicity (Theorem~1).
We elicited isometric embeddings of $t$-Hilbert TEM spaces in Theorem~2 and Theorem~3.
Last, we reported a differentiable approximation of the $t$-Hilbert distances for machine learning and experimentally studied its approximation properties and embedding performance (Figure~5).
Besides, this work extends the $t$-calculus~\citep{nlwGA} by unraveling tempered extensions of the derivative and Riemannian integration which should prove useful in other settings.

An interesting open problem is to develop matrix Funk distances, properly integrating the case when matrices have different eigensystems. This is non-trivial because intuitively (all other things being equal), the distance should be largest when the eigensystems are the same,
mimicking the fact that learning is often hardest in that case \citep{freelunch}.

\bibliographystyle{plainnat}
\bibliography{refs}

\newpage
\onecolumn
\appendix

\section{Deformed $t$-logarithms and $t$-exponentials, and $t$-algebra}\label{app:deformedlogexp}

By observing that the ordinary logarithm $\log(x)$ can be expressed as the definite Riemannian integral
$\log(x)=\int_1^x \frac{1}{u}\du$ measuring the oriented area underneath the graph of the strictly decreasing function $f(u)=\frac{1}{u}$, one can generalize the logarithm by using any arbitrary strictly decreasing positive function $f(u)$~\citep{nGT}:
$\log_f(x)=\int_1^x \frac{1}{f(u)}\du$ for $x>0$. Hence, we always have $\log_f(1)=0$.
In particular, by choosing $f_t(u)=\frac{1}{u^t}$ for $t>0$, we obtain the $t$-logarithm~\citep{texp1} $\log_t(x)=\frac{1}{1-t} (x^{1-t}-1)$ for $t\not=1$ and $\log_1(u)=\log u$ when $t=1$.
The reciprocal function of the $t$-logarithm is called the $t$-deformed exponential $\exp_t(y)$:
$\frac{1}{1-t} (x^{1-t}-1)=y \Rightarrow x=\exp_t(y)=(1+(1-t)y)^{\frac{1}{1-t}}$ provided that $1+(1-t)y\geq 0$.
Hence, we define $\exp_t(y)=[(1+(1-t)y)^{\frac{1}{1-t}}]_+$ for $t\not =1$ where $[u]_+=\max(u,0)$, and $\exp_1(y)=\exp(y)$.
Notice that clipping of the $t$-exponentials may occur depending on the value of $t$.

\section{Tempered Funk Distance from a Weak Finsler Structure}\label{sec-temp-Funk}
We review some preliminary material related to Finsler geometry. The interested reader is referred to~\citet{troyanov2014origin,HB-FinslerHilbert-2014} for further details and historical remarks. 

\subsection{Weak Finsler Structure on a Manifold}
\begin{definition}
A weak Finsler structure on a smooth Manifold $M$ is a lower-semicontinuous function $F: TM \rightarrow [0, \infty]$, called the Lagrangian, such that for every point $\ve{x} \in M$ and and tangent vectors $\ve{\xi}, \ve{\xi}_1, \ve{\xi}_2 \in T_{\ve{x}}M$ the following properties hold:
\begin{enumerate}
    \item $F(\ve{x}, \lambda\,\ve{\xi}) = \lambda\, F(\ve{x}, \ve{\xi})$ for all $\lambda \geq 0$
    \item $F(\ve{x}, \ve{\xi}_1 + \ve{\xi}_2) \leq F(\ve{x}, \ve{\xi}_1) + F(\ve{x}, \ve{\xi}_2)$\quad (triangular inequality)
\end{enumerate}
In addition, if $F: TM \rightarrow [0, \infty]$ is finite and continuous and if $F(\ve{x}, \lambda\,\ve{\xi}) > 0$ for $\ve{\xi} \neq \ve{0}$, then $F$ is called a Finsler structure.
\end{definition}
Note that a weak Finsler structure is a weak Minkowski norm at every point $\ve{x} \in M$. For instance, a weak Finsler structure on $T_{\ve{x}}M \subseteq \mathbb{R}^n$ can be obtained from a weak Minkowski norm $F_0: \mathbb{R}^n \rightarrow \mathbb{R}$ by setting 
\[
F(\ve{x}, \ve{\xi}) = F_0(\ve{\xi})\, .
\]
Given a Finsler structure $F$ on $M$, its domain $\mathcal{D}_F \subseteq TM$ is defined as the set of all vectors with finite $F$-norm. The \emph{unit domain} $\mathcal{U}$ is defined as the bundle of all tangent unit balls,
\[
\mathcal{U} = \{(\ve{x}, \ve{\xi}) \in TM\vert F(\ve{x}, \ve{\xi}) < 1\}\,.
\]
The restriction $\mathcal{U}_{\ve{x}} = \mathcal{U} \cap T_{\ve{x}}M$ is a bounded convex set. A weak Finsler structure can be recovered from its unit domain via
\[
F(\ve{x}, \ve{\xi}) = \inf_{\tau > 0} \frac{1}{\tau}\ve{\xi} \in \mathcal{U}_{\ve{x}}\,.
\]


Our construction generalizes the notion of length in a Finsler geometry.
Given a Finsler structure $F$ on $M$, the forward-length of a smooth curve $\ve{\gamma}: [0, 1] \rightarrow M$ is defined as
\begin{equation}
    \ell^+(\ve{\gamma}) = \int_0^1 F(\ve{\gamma}(s), \dot{\ve{\gamma}}(s))\, \mathrm{d}s\,\label{deflength}.
\end{equation}
We may also define the backward-length
\begin{equation}
    \ell^-(\ve{\gamma}) = \int_0^1 F(\ve{\gamma}(s), -\dot{\ve{\gamma}}(s))\, \mathrm{d}s\,\label{deflength}.
\end{equation}
The forward- and backward-lengths coincide when the Minkowski norms are symmetric but differ otherwise.
This explains the fact that Finsler distances are quasi-metrics that may not be necessarily symmetric (e.g., Funk) although they satisfy the triangle inequality. To contrast with oriented Finsler distances, Riemannian distances are always symmetric and hence define proper metrics (see~\citet{matveev2012can,arnaudon2012medians}).

To extend these definitions, we first introduce the notion of a $t$-derivative and $t$-Riemann integral. We will simply refer to the forward-length of a curve as its length and omit the $+$ superscript. 

\subsection{From $t$-Derivative and $t$-Riemann Integral to $t$-Length}

\noindent \textbf{$t$-derivative} We generalize the notion of the standard derivative a scalar function $f: \mathbb{R} \rightarrow \mathbb{R}$ to a $t$-derivative.
\begin{definition}\label{tder}
  Suppose $f$ is defined on an open neighborhood of some $x \in \mathbb{R}$. When it exists, the $t$-derivative of $f$ in $x$ is the real
  \begin{equation}
    \label{eq:t-derivative}
    \mathrm{D}_t f(x) \defeq \lim_{\delta \rightarrow 0} \frac{f(x + \delta) \ominus_t f(x)}{\delta}.
\end{equation}
  \end{definition}
$\mathrm{D}_t$ generalizes the notion of standard derivative for $t=1$. Note that we can obtain a direct expression of $\mathrm{D}_t f$ by using the definition of $\ominus_t$: 
\begin{equation}
\label{eq:t-derivative-simple}
    \mathrm{D}_t f(x) = \lim_{\delta \rightarrow 0} \frac{1}{\delta} \cdot \frac{f(x+\delta) - f(x)}{1+(1-t)f(x)} = \frac{1}{1+(1-t)f(x)} \cdot \lim_{\delta \rightarrow 0} \frac{f(x+\delta) - f(x)}{\delta} = \frac{f'(x)}{1+(1-t)f(x)}.
\end{equation}
For instance, using \eqref{eq:t-derivative-simple}, we can find the functions whose $t$-derivative is constant: 
\begin{equation}
  \label{eq:t-const-fun}
  \mathrm{D}_t f(x) = K.
\end{equation}
This involves solving a simple differential equation $f' - (1-t)K f = K$, whose set of solutions is easily found to be
\begin{equation}
    \label{eq:t-const-fun}
    f(x) = V \exp\left((1-t)K x\right) - \frac{1}{1-t} \quad (V\mbox{ does not depend on }x).
  \end{equation}
Remark the difference between this function and the solution for $t=1$, which would be noted $f(x) = Kx + V$, and the risk to diverge in $t=1$. This can be prevented if we enforce the solution to converge to the solution of the conventional derivative for $t\rightarrow 1$. To get there, we compute its Taylor expansion around $t=1$,
  \begin{equation}
    \label{eq:t-const-fun2}
    f(x) = V - \frac{1}{1-t} + VK x (1-t) + \frac{V K^2 x^2(1-t)^2}{2} + o((1-t)^2).
  \end{equation}
  We see that there is only one choice, $V = 1/(1-t)$, which yields the desired series, $f(x) = Kx + o(1-t)$, and we have a unique solution to the differential equation that satisfies the behavior,
  \begin{equation}
    \label{eq:t-const-funfinal}
    f(x) = \frac{ \exp\left((1-t)K x\right) - 1}{1-t} \quad \left(\mbox{unique solution for } \mathrm{D}_t f(x) = K \wedge \lim_{t\rightarrow 1} f(x) = Kx\right).
  \end{equation}
Proceeding in a similar way, we would find that the unique function to $\mathrm{D}_t \circ \mathrm{D}_t f(x) = K$ with $\lim_{t\rightarrow 1} f(x) = (K/2) \cdot x^2$ (the second order derivative of $f$ is a fixed constant, $K$) is 
\begin{equation}
    \label{eq:t-const2-fun}
    f(x) = \frac{\exp\Big(\frac{\frac{1}{K} \exp((1 - t)\,K\, x) - (1 - t)\,x - \frac{1}{K}}{1 - t}\Big) - 1}{1 - t}.
  \end{equation}
Remark the non-trivial variations of functions in \eqref{eq:t-const-fun2}, \eqref{eq:t-const2-fun}. Using \eqref{eq:t-derivative-simple} or by simple calculation, we obtain
\begin{equation}
    \label{eq:t-der-log}
    \mathrm{D}_t \log_t(x) \defeq \lim_{\delta \rightarrow 0} \frac{\log_t(\frac{x + \delta }{x})}{\delta} = \lim_{\delta \rightarrow 0} \frac{\frac{\delta }{x}}{\delta } = \frac{1}{x}.
\end{equation}
Thus, for $\delta \ll x$, we can define an expansion
\begin{equation} 
    \log_t(x + \delta) \approx \log_t(x) \oplus_t \frac{\delta}{x}.
\end{equation}
\noindent \textbf{$t$-Riemann Integral} We wish to define a generalization of Riemann integration to the tempered algebra and for this objective, given an interval $[a,b]$ and a division $\Delta$ of this interval using $n+1$ reals $x_0 \defeq a <x_1 < ... < x_{n-1} < x_n \defeq  b$, we define the Riemann $t$-sum over $[a,b]$ using $\Delta$ as
\begin{eqnarray*}
S_\Delta(f) & \defeq & (\mbox{\Large $\oplus$}_t)_{i=1}^n |\mathbb{I}_i| f(\xi_i), \quad (\mathbb{I}_i \defeq [x_{\pi(i)-1}, x_{\pi(i)}], |\mathbb{I}_i | \defeq x_{\pi(i)} - x_{\pi(i)-1}), \xi_i \in \mathbb{I}_i,
\end{eqnarray*}
where $\pi$ is any permutation of $[n]\defeq \{1, 2, ..., n\}$ ($\oplus_t$ is commutative, so changing $\pi$ does not change the result; we do not put the $\xi_i$s in the argument of $S_.$ for readability). In the classical definition, $\pi = \mathrm{id}$. It will be useful to remember that we can fix beforehand the $\xi_i$s so they are given for a given division. Let $s(\Delta) \defeq \max_i |\mathbb{I}_i|$ denote the step of division $\Delta$. The conditions for $t$-Riemann integration are the same as for $t=1$.
\begin{definition}\label{tint}
  A continuous function $f : [a,b] \rightarrow \mathbb{R}$ is $t$-Riemann integrable iff there exists $L$ such that
\begin{eqnarray}
\forall \epsilon > 0, \exists \delta > 0 : \forall \mbox{ division } \Delta \mbox{ with }  s(\Delta) < \delta, \left|S_\Delta(f) - L\right| < \varepsilon. \label{eq-tRiem}
  \end{eqnarray}
  When this happens, we note
  \begin{equation}
\label{eq:t-int-def}
    \riemannint{t}{a}{b} f(x) \mathrm{d}_t x = L.
  \end{equation}
  \end{definition}

    \noindent \textbf{$t$-Riemann Integral} We define
\begin{equation}
\label{eq:t-int}
    \riemannint{t}{a}{b} f(x) \mathrm{d}_t x = F(b) \ominus_t F(a),
  \end{equation}
  as the \emph{$t$-integral} of $f$, where $F$ is called a \emph{$t$-primitive} s.t. $\mathrm{D}_t F = f$. For instance, from~\eqref{eq:t-der-log} and for $a, b \geq 0$, we have
\begin{equation}
    \riemannint{t}{a}{b} \frac{1}{x}\, \mathrm{d}_t x = \log_t b\, \ominus_t\, \log_t a = \log_t\frac{b}{a}\,.
  \end{equation}

      \noindent \textbf{$t$-Length} We use~\eqref{eq:t-int} to define a generalization of length to a $t$-length for smooth curve on a Finsler manifold. Let $\ve{\gamma}: [0, 1] \rightarrow M$ denote a smooth curve on a Finsler manifold with a Finsler structure $F$. We define the (forward) $t$-length of the curve $\ve{\gamma}$ as
\begin{equation}
    \label{eq:t-length}
   \ell^{(t)}(\ve{\gamma}) \defeq \riemannint{t}{0}{1} F(\ve{\gamma}(s), \dot{\ve{\gamma}}(s))\, \mathrm{d}_t s\,.
 \end{equation}
\subsection{Tautological Finlser Structure of a Convex Set}
\begin{definition}
Given a proper convex set $\Omega$, the tautological weak Finsler structure $F_f$ on $\Omega$ is the Finsler structure for which the unit ball at a point $\ve{x} \in \Omega$ is the domain $\Omega$ itself, with the point $\ve{x}$ as the center, thus, resulting the unit domain
\begin{equation}
    \mathcal{U} = \{(\ve{x}, \ve{\xi})\in T\Omega\,\vert\, \ve{\xi} \in \Omega - \ve{x}\}\,,
\end{equation}
where the $\Omega - \ve{x}$ should be viewed as a translation of the convex set $\Omega$ by the point $\ve{x}$. The Lagrangian is then given by
\begin{equation}
\label{eq:taut}
    F_f(\ve{x}, \ve{\xi}) = \inf\{\tau > 0\,\vert\, \ve{x} + \frac{\ve{\xi}}{\tau} \in \Omega\}\,,
\end{equation}
Consequently, $\ve{x} + \frac{\ve{\xi}}{F_f(\ve{x}, \ve{\xi})} \in \partial\Omega$, otherwise $F_f(\ve{x}, \ve{\xi})=0$ if the ray $\ve{x} + \mathbb{R}_+\ve{\xi}$ is contained in $\Omega$.
\end{definition}
In can be shown that the tautological Finsler structure of a half-space $\mathcal{H} = \{\ve{x}\in \mathbb{R}^d\,\vert\,\ve{\nu}\cdot\ve{x}\leq c\} \in \mathbb{R}^d$ is given by
\begin{equation}
\label{eq:finsler-halfspace}
    F_f(\ve{x}, \ve{\xi}) = \max\Big(\frac{\ve{\nu}\cdot\ve{\xi}}{c - \ve{\nu}\cdot\ve{x}},\, 0\Big)\,.
\end{equation}
\begin{theorem}
\label{thm:funk}
The tautological distance in a proper convex domain $\Omega \in \mathbb{R}^d$ induced by the $t$-length of the ray is given by
\begin{equation}
\label{eq:rhof}
    \rho_f(\ve{r}, \ve{s}) = \log_t \frac{\Vert \ve{r} - \bar{\ve{s}}\Vert}{\Vert \ve{s} - \bar{\ve{s}}\Vert}\,,
\end{equation}
where $\bar{\ve{s}}$ is the intersection of the ray $R(\ve{r}, \ve{s})$ emanating from $\ve{r}$ and passing through $\ve{s}$ with the boundary $\partial \Omega$. If the ray is contained in $\Omega$, then $\bar{\ve{s}}$ is considered to be a point at infinity and therefore, $\rho_f(\ve{r}, \ve{s}) = 0$.
\end{theorem}
Note that the tautological distance via the $t$-length in~\eqref{eq:rhof} is identical to the $t$-Funk distance we define in~\eqref{eq:t-funk-dist}. To prove Theorem~\ref{thm:funk}, we first need to restate two lemmas from~\citet{HB-FinslerHilbert-2014}.
\begin{lemma}[\citet{HB-FinslerHilbert-2014}]
\label{lem:subset}
Let $\Omega_1$ and $\Omega_2$ be two convex domains in $\mathbb{R}^d$. If $F_1$ and $F_2$ are the corresponding tautological structures, then 
\[
\Omega_1 \subseteq \Omega_2 \Longleftrightarrow F_1(\ve{x}, \ve{\xi}) \geq F_2(\ve{x}, \ve{\xi})
\]
for all $(\ve{x}, \ve{\xi}) \in T\Omega_1$.
\end{lemma}
\begin{lemma}[\citet{HB-FinslerHilbert-2014}]
\label{lem:finsler-rs}
Let $\Omega$ be a convex domain in $\mathbb{R}^d$ and $\ve{r}, \ve{s} \in \Omega$. If $\ve{s} = \ve{r} + \tau\ve{\xi}$ for some $\ve{\xi} \in \mathbb{R}^d$ and $\tau \geq 0$, then
\begin{equation}
    F_f(\ve{s}, \ve{\xi}) = \frac{F_f(\ve{r}, \ve{\xi})}{1 - \tau\,F_f(\ve{r}, \ve{\xi})}\,.
\end{equation}
\end{lemma}
\begin{proof}[Proof of Theorem~\ref{thm:funk}]
We start by calculating the distance when the domain is a half-space $\Omega = \mathcal{H} = \{\ve{x}\in \mathbb{R}^d\,\vert\,\ve{\nu}\cdot\ve{x}\leq c\}$ for some $\ve{\nu} \neq \ve{0}$ by $t$-integrating~\eqref{eq:finsler-halfspace} along a curve $\ve{\beta}$ connecting $\ve{r}$ and $\ve{s}$:
\begin{align}
    \rho^{\mathcal{H}}_f(\ve{r}, \ve{s}) & = \riemannint{t}{0}{1} \max\Big(\frac{\ve{\nu}\cdot\dot{\ve{\beta}}(\tau)}{c - \ve{\nu}\cdot\ve{\beta}(\tau)},\, 0\Big) \mathrm{d}_t \tau\nonumber\\
    & = \riemannint{t}{0}{1} \max\Big(\frac{(c - \ve{\nu}\cdot\ve{\beta}(\tau))'}{\vert c - \ve{\nu}\cdot\ve{\beta}(\tau)\vert},\, 0\Big) \mathrm{d}_t \tau\nonumber\\
    & = \max\Big(\log_t \big(\frac{c - \ve{\nu}\cdot\ve{r}}{c - \ve{\nu}\cdot\ve{s}}\big),\, 0\Big)\,.\label{eq:taut-int-form}
\end{align}
Suppose $\ve{a} \in \partial\mathcal{H}$ is the boundary point along the ray emanating from $\ve{r}$ and passing through $\ve{s}$. Then $\ve{\nu}\cdot \ve{a} = c$ and can write~\eqref{eq:taut-int-form}
\begin{equation}
\label{eq:taut-int-form2}
    \rho_f^{\mathcal{H}}(\ve{r}, \ve{s}) = \log_t\frac{\Vert \ve{a} - \ve{r}\Vert}{\Vert \ve{a} - \ve{s}\Vert}\,,
\end{equation}
where $\Vert \cdot\Vert$ is any arbitrary norm (e.g., $\|\cdot\|_2$). For a general convex domain $\Omega$, let $\ve{a} \in \partial \Omega$ be defined similarly to the previous case for the points $\ve{r}$ and $\ve{s}$. We have
\begin{equation}
    F^{\Omega}_f(\ve{r}, \ve{\xi}) = \frac{1}{\Vert \ve{a} - \ve{r}\Vert}\,,
\end{equation}
where $\ve{\xi} = \frac{\ve{s} - \ve{r}}{\Vert\ve{s} - \ve{r}\Vert}$ is the unit vector along the ray connecting $\ve{r}$ and $\ve{s}$. Using Lemma~\ref{lem:finsler-rs}, we have
\begin{equation}
    F^{\Omega}_f(\ve{\beta}(\tau), \dot{\ve{\beta}}(\tau)) = \frac{F^{\Omega}_f(\ve{r}, \ve{\xi})}{1 - \tau\,F^{\Omega}_f(\ve{r}, \ve{\xi})} = \frac{1}{\Vert\ve{a} - \ve{r}\Vert - \tau}
\end{equation}
along the curve $\ve{\beta}(\tau) = \ve{r} + \tau\,\ve{\xi}$ and thus, we can calculate the $t$-length along $\ve{\beta}$ as
\begin{align}
    \rho^{\Omega}_f(\ve{r}, \ve{s}) & = \riemannint{t}{0}{\Vert\ve{s} - \ve{r}\Vert} \frac{1}{\Vert\ve{a} - \ve{r}\Vert - \tau} \mathrm{d}_t \tau\nonumber\\
    & = \log_t\frac{1}{\Vert\ve{a} - \ve{r}\Vert - \Vert\ve{s} - \ve{r}\Vert} \ominus_t \log_t\frac{1}{\Vert\ve{a} - \ve{r}\Vert}\nonumber\\
    & = \log_t\frac{\Vert \ve{a} - \ve{r}\Vert}{\Vert \ve{a} - \ve{s}\Vert}\,,
\end{align}
since $\ve{r}$, $\ve{s}$, and $\ve{a}$ are collinear. Thus, taking $\mathcal{H}$ to be the supporting hyperplane to $\Omega$ at $\ve{a}$, we have $\rho^{\Omega}_f(\ve{r}, \ve{s}) \leq \rho^{\mathcal{H}}_f(\ve{r}, \ve{s})$. However, from Lemma~\ref{lem:monotone}, we have $\rho^{\Omega}_f(\ve{r}, \ve{s}) \geq \rho^{\mathcal{H}}_f(\ve{r}, \ve{s})$ and thus, the result holds with equality.

\end{proof}
\begin{proposition}
The unit speed linear $t$-geodesic starting at $\ve{r}\in\Omega$ in the direction of $\ve{\xi}\in T_{\ve{r}}\Omega$ is the path
\begin{equation}
    \ve{\gamma}_{\ve{r}, \ve{\xi}}(\tau) = \ve{r} + \frac{1 - \exp_t\ominus_t \tau}{F_f(\ve{r}, \ve{\xi})}\,\ve{\xi}\,.
\end{equation}
\end{proposition}
\begin{proof}
By~\eqref{eq:taut}, we have
\[
\bar{\ve{s}} = \ve{r} + \frac{\ve{\xi}}{F_f(\ve{r}, \ve{\xi})}\,, \text{ \,\, and therefore \,\, } \bar{\ve{s}} - \ve{\gamma}(\tau) = \frac{\exp_t\ominus_t \tau}{F_f(\ve{r}, \ve{\xi})}\cdot \ve{\xi}\,.
\]
Thus, we have
\[
\rho_f(\ve{r}, \ve{\gamma}(\tau)) = \log_t\frac{\Vert \ve{r} - \bar{\ve{s}}\Vert}{\Vert \ve{\gamma}(\tau) - \bar{\ve{s}}\Vert} = \log_t \frac{1}{\exp_t\ominus_t \tau} = \tau\,.
\]
\end{proof}

\section{Properties of the Function $\Vert \cdot\Vert_{\tvar}$}
The function $\Vert \cdot\Vert_{\tvar}$ behaves almost like a \emph{semi-$t$-norm} in the following sense:
\begin{enumerate}
    \item $t$-Subadditivity: $\semitnorm{t}{\ve{u} \oplus_t \ve{v}} \leq \semitnorm{t}{\ve{u}} \oplus_t \semitnorm{t}{\ve{v}}$, for all $\ve{u}, \ve{v} \in \bbR^d$.
    \item Absolute homogeneity does not hold in general but for $\alpha > 0$, we have: $\semitnorm{t}{\alpha\,\ve{u}} = \alpha\cdot\semitnorm{t_\alpha}{\ve{u}}$, where $t_\alpha= 1 - (1 - t)\alpha$ (given that $\exp_t(\min_i \alpha\,u_i) = \exp_{t_\alpha}\!(\min_i u_i) \neq 0$).
    \item Positive semi-definiteness: $\semitnorm{t}{\ve{u}} \geq 0$, and if $\semitnorm{t}{\ve{u}} = 0$ then $\ve{u} = r\ve{1}$ for $r \in \bbR$. 
\end{enumerate}

\section{Proofs}
\begin{proof}[Proof of Proposition~\ref{prop:lambda}] The proof follows by substituting the definition of $\log_t$ in~\eqref{eq:thetc} and applying the equality $\Mat{P}_t\check{\ve{\theta}} = \check{\ve{\theta}}$.
\end{proof}
\begin{proof}[Proof of Proposition~\ref{prop:temp-soft-inv}]
By~\eqref{eq:thetc} and for a given $r \in \mathbb{R}$, we can write
\begin{equation}
    \check{\ve{\theta}} + r\,\hat{\ve{n}} = \check{\ve{\theta}} + r\|\tilde{\ve{p}}^{1-t}\|_{2}\cdot \tilde{\ve{p}}^{1-t} = \log_t \frac{\tilde{\ve{p}}}{K}\,,
\end{equation}
where $K = 1/\exp_t(\lambda + r\|\tilde{\ve{p}}^{1-t}\|_{2})$. Applying the function $\check{f}^*$ to the rhs concludes the proof.
\end{proof}
\begin{proof}[Proof of Proposition~\ref{prop:t-hilbert-metric}]
Consider the function $h_t(u)=\log_t \exp(u)$ which satisfies $h_t(u)=0$
if and only if $u=0$ for any $t$.
The function $h_t(u)$ is increasing for $t>1$ and $x>0$ since
$h_t'(u)=\frac{1-\exp((1-t)u)}{t-1}>0$, and hence $h_t(u)\geq 0$.
Function $h_t(u)$ is also subadditive for $t>1$ since
$h_t(a+b)=\log_t(e^a e^b)=\log_t e^a+\log_t e^b+(1-t)(\log_t e^a)(\log_t
e^b)=h_t(a)+h_t(b)-(1-t)h_t(a)h_t(b)\leq h_t(a)+h_t(b)$.
Thus we have the property that $h_t(u)$ which is a metric
transform~\citep{MetricTransform-1981}.
It follows that $\rho_t^\Omega(p,q)$ is a metric for $t\geq 1$.
\end{proof}
\begin{proof}[Proof of Theorem~\ref{thm:contr}]
Given $\ve{r}, \ve{s} \in \calC$, we have $$\rho_{\tHG}^{\calC}(\ve{r}, \ve{s}) = \log_t \exp \rho_{\HG}^{\calC}(\ve{r}, \ve{s})\,.$$
Since $\rho_{\HG}^{\calC}$ is a contraction for positive linear maps, it suffices to show that the function $x \mapsto \log_t \exp x$ is monotonic for $t < 2$. Monotonicity simply follows from the derivative $(\log_t \exp)' (x) = \exp^{1 - t} (x) > 0$. From~\citet{HilbertContraction-1982} (Theorem 4.1) and the fact that $\rho_{\tHG}^{\calC}$ is a monotonic function of $\rho_{\HG}^{\calC}$, we conclude that $\kappa_{\tHG}(A) \geq \kappa_{\HG}(A)$ for all positive $A$.
\end{proof}
\begin{proof}[Proof of Lemma~\ref{lem:tstar-dist}]
\begin{align*}
   t^*\rho^{\Delta^d}_{\tsFD}(\tilde{\ve{p}}^{1/t^*}, \tilde{\ve{q}}^{1/t^*}) =  t^* \log_{t^*}\max_i \frac{\tilde{p}^{1/t^*}_i}{\tilde{q}^{1/t^*}_i} = t^* \max_i \log_{t^*} \frac{\tilde{p}^{1/t^*}_i}{\tilde{q}^{1/t^*}_i} = \max_i\log_t \frac{\tilde{p}_i}{\tilde{q}_i}
   = \log_t \max_i \frac{\tilde{p}_i}{\tilde{q}_i}\,.
\end{align*}
\end{proof}
    \begin{proof}[Proof of Proposition~\ref{prop:dist-properties}]
  \,

\noindent \textbf{Proof of condition (i)} We remark that for any $\tilde{\ve{p}}, \tilde{\ve{q}}$,
    \begin{eqnarray}
m_t (\tilde{\ve{p}},\tilde{\ve{q}}) \cdot \tilde{p}_i \leq \tilde{q}_i \leq M_t (\tilde{\ve{p}},\tilde{\ve{q}}) \cdot \tilde{p}_i, \forall i \in [m].
      \end{eqnarray}
Hence, if $\rho_{\tHG}(\tilde{\ve{p}},\tilde{\ve{q}}) = 0$ then $m_t (\tilde{\ve{p}},\tilde{\ve{q}}) = M_t (\tilde{\ve{p}},\tilde{\ve{q}}) = K$ and $\tilde{\ve{p}} = K \cdot \tilde{\ve{q}}$, and if $\tilde{\ve{p}} = K \cdot \tilde{\ve{q}}$ then we immediately have $\rho_{\tHG}(\tilde{\ve{p}},\tilde{\ve{q}}) = 0$.

\noindent \textbf{Proof of condition (ii)} It follows directly from the definition of $\ominus_t$ and $t$-negation that $\ominus_t \log_t a/b = \ominus_t (\log_t a \ominus_t \log_t b) = \log_t b \ominus_t \log_t a = \log_t b/a$. Thus,
\begin{align}
    \rho_{\tHG}(\tilde{\ve{p}},\tilde{\ve{q}}) & = \ominus_t \min_i \big(\log_t \tilde{p}_i \ominus_t \log_t \tilde{q}_i\big) \ominus_t \ominus_t\max_j \big(\log_t \tilde{p}_j \ominus_t \log_t \tilde{q}_j\big) = \rho_{\tHG}(\tilde{\ve{q}},\tilde{\ve{p}}).
\end{align}

\noindent \textbf{Proof of condition (iii)} Similar to the proof of \cite{bHM} (Theorem 2.1), we have
\[
\tilde{\ve{p}} \leq M(\tilde{\ve{p}}, \tilde{\ve{q}}) \tilde{\ve{q}} \leq M(\tilde{\ve{p}}, \tilde{\ve{q}}) M(\tilde{\ve{q}},  \tilde{\ve{r}}) \tilde{\ve{r}}
\]
Thus,
\[
\log_t M(\tilde{\ve{p}}, \tilde{\ve{r}}) \leq \log_t M(\tilde{\ve{p}}, \tilde{\ve{q}}) \oplus_t \log_t M(\tilde{\ve{q}},  \tilde{\ve{r}})\, .
\]
Similarly, we have
\[
\log_t m(\tilde{\ve{p}}, \tilde{\ve{r}}) \geq \log_t m(\tilde{\ve{p}}, \tilde{\ve{q}}) \oplus_t \log_t m(\tilde{\ve{q}}, \tilde{\ve{r}})\, .
\]
Combining the two inequalities concludes the proof of the statement.

\noindent \textbf{Proof of condition (iv)} We need to show
\[
\log_t \Big(\frac{M(\tilde{\ve{p}}, \tilde{\ve{q}}) \cdot M(\tilde{\ve{q}},  \tilde{\ve{r}})}{m(\tilde{\ve{p}}, \tilde{\ve{q}}) \cdot m(\tilde{\ve{q}},  \tilde{\ve{r}})}\Big) \leq \log_t \Big(\frac{M(\tilde{\ve{p}}, \tilde{\ve{q}})}{m(\tilde{\ve{p}}, \tilde{\ve{q}})}\Big) + \log_t \Big(\frac{M(\tilde{\ve{q}},  \tilde{\ve{r}})}{m(\tilde{\ve{q}},  \tilde{\ve{r}})}\Big)\,,
\]
for $1 \leq t < 2$. The inequality amounts to showing
\[
\log_t(a\,b) \leq \log_t a + \log_t b\, ,\quad  a, b \geq 1,
\]
or, since $(1 - t) \leq 1$, to show
\[
(ab)^\tau + 1 \geq a^\tau + b^\tau,\,\,\, a, b \geq 1 \text{ and }\tau > 0.
\]
Fix $a$ and consider the function $f(x) = (ax)^\tau + 1 - a^\tau - x^\tau$. Note that $f(1) = 0$ and $f'(x) = \tau\,x^{\tau-1}(a^\tau - 1) \geq 0$. Thus, the function is increasing for all $x > 1$.
    \end{proof}
\begin{proof}[Proof of Lemma~\ref{lem:monotone}]
Denote $\iota = \max\{\tilde{p}_1/\tilde{q}_1, \tilde{p}_2/\tilde{q}_2\}$. Assuming $\tilde{q}_1, \tilde{q}_2 > 0$, we have $\tilde{p}_1 \leq \iota\,\tilde{q}_1$ and $\tilde{p}_2 \leq \iota\,\tilde{q}_2$. Thus,
\begin{equation}
    \frac{(\tilde{p}_1^{1/{t^*}} + \tilde{p}_2^{1/{t^*}})^{t^*}}{(\tilde{q}_1^{1/{t^*}} + \tilde{q}_2^{1/{t^*}})^{t^*}} \leq \frac{\iota\, (\tilde{q}_1^{1/{t^*}} + \tilde{q}_2^{1/{t^*}})^{t^*}}{(\tilde{q}_1^{1/{t^*}} + \tilde{q}_2^{1/{t^*}})^{t^*}} = \iota. 
\end{equation}
Hence, we have
\begin{equation}
    \log_t \max\left\{\frac{(\tilde{p}_1^{1/{t^*}} + \tilde{p}_2^{1/{t^*}})^{t^*}}{(\tilde{q}_1^{1/{t^*}} + \tilde{q}_2^{1/{t^*}})^{t^*}}, \ldots, \frac{\tilde{p}_d}{\tilde{q}_d}\right\} \leq \log_t \max_i \frac{\tilde{p}_i}{\tilde{q}_i}.
\end{equation}
The proof is complete by the definition of the $t$-Funk distance.
\end{proof}
\begin{proof}[Proof of Theorem~\ref{thm:unconst}]
The mapping from $\tilde{\Delta}_t^d$ via~\eqref{eq:thetc} is bijective. Given $\tilde{\ve{p}}, \tilde{\ve{q}} \in \tilde{\Delta}_t^d$, the distance $\rho_{\tNH}$ between $\check{\ve{\theta}} = \log_t\frac{\tilde{\ve{p}}}{\tilde{\lambda}_t(\tilde{\ve{p}})}$ and $\check{\ve{\theta}}' = \log_t\frac{\tilde{\ve{q}}}{\tilde{\lambda}_t(\tilde{\ve{q}})}$ amounts to 
\begin{align*}
  \rho_{\tNH}(\check{\ve{\theta}}, \check{\ve{\theta}}') & = \inf_\tau \Big\vert  \log_t\frac{\tilde{p}_i}{\tilde{\lambda}_t(\tilde{\ve{p}})} \ominus_t \log_t\frac{\tilde{q}_i}{\tilde{\lambda}_t(\tilde{\ve{q}})} \ominus_t (\log_t\frac{\tilde{p}_j}{\tilde{\lambda}_t(\tilde{\ve{p}})} \ominus_t \log_t\frac{\tilde{q}_j}{\tilde{\lambda}_t(\tilde{\ve{q}})})\Big\vert   \in B^{\tau}_{\tilde{V}^d_t}\,,\,\, i \neq j\\
  & = \inf_\tau \Big\vert  \log_t\frac{\frac{\tilde{p}_i}{\tilde{q}_i}}{\frac{\tilde{p}_j}{\tilde{q}_j}}\Big\vert  \in B^{\tau}_{\tilde{V}^d_t}\,,\,\, i \neq j\\
  & = \log_t\frac{\max_i\frac{\tilde{p}_i}{\tilde{q}_i}}{\min_j\frac{\tilde{p}_j}{\tilde{q}_j}}\\
  & = \rho_{\tHG}(\tilde{\ve{p}}, \tilde{\ve{q}})\,.
\end{align*}
\end{proof}
\begin{proof}[Proof of Theorem~\ref{thm:const}]
The mapping from $\tilde{\Delta}_t^d$ to $\bbR^d$ via $\log_t$ is bijecttive. Given $\tilde{\ve{p}}, \tilde{\ve{q}} \in \tilde{\Delta}_t^d$, the distance $\rho_{tvar}$ between $\ve{\theta} = \log_t \tilde{\ve{p}}$ and $\ve{\theta}' = \log_t \tilde{\ve{q}}$ amounts to
\begin{align*}
    \rho_{\tvar}(\ve{\theta}, \ve{\theta}') = \Vert \ve{\theta} \ominus_t \ve{\theta}' \Vert_{\tvar} & = \max_i \{\log_t \tilde{p}_i \ominus_t \log_t \tilde{q}_i\} \ominus_t \min_j \{\log_t \tilde{p}_j \ominus_t \log_t \tilde{q}_j\}\\ 
    & = \log_t\frac{\max_i\frac{\tilde{p}_i}{\tilde{q}_i}}{\min_j\frac{\tilde{p}_j}{\tilde{q}_j}} = \rho_{\tHG}(\tilde{\ve{p}}, \tilde{\ve{q}})\,.
\end{align*}
\end{proof}

\section{Differentiable Approximations}
In the following, we provide the tempered log-sum-exp function and bound its approximation error. Let
\begin{equation*}
    \text{LSE}_{t}(\ve{x}, T) = \frac{1}{T}\, \log_t \sum_i \exp_t(T x_i)\,.
\end{equation*}
We have
\begin{equation}
   \frac{1}{T}\,\big(\max_i T x_i \oplus_t \varepsilon^\ell_t(\ve{x}, T)\big) \leq \rho_{t-diff}(\ve{x}, T) \leq \frac{1}{T}\,\big(\max_i T x_i \oplus_t  \varepsilon^r_t(\ve{x}, T)\big)\,,
\end{equation}
where
\begin{align*}
    \varepsilon_t^\ell(\ve{x}, T) & = \log_t\big(1 + (d - 1)\,\exp_t \ominus_t\Vert T\ve{x}\Vert_{\tvar}\big)\,,\\
    \varepsilon_t^r(\ve{x}, T) & = \log_t\big((d - 1) + \exp_t \ominus_t\Vert T\ve{x}\Vert_{\tvar}\big)\,.
\end{align*}
Additionally, for $t \leq 1$,
\begin{equation}
    \lim_{T\rightarrow \infty}\text{LSE}_{t}(\ve{x}, T) = \max_i x_i\,.
\end{equation}
Note that we can also write the bounds as a standard sum 
\begin{equation*}
        \frac{1}{T}\,\big(\max_i T x_i \oplus_t \log_t \varepsilon^\cdot_t(\ve{x}, T)\big)
        = \max_i x_i + \frac{1}{T}\big((\max_i T x_i)^{1 - t}\,\varepsilon^\cdot_t(\ve{x}, T)\big)\,.
\end{equation*}
We now define the differentiable $t$-Funk and $t$-Hilbert distance with smoothing factor $T$ via the $\text{LSE}_{t}$ approximation by setting $x_i = \log_t \frac{\tilde{p}_i}{\tilde{q}_i}$ as
\begin{align*}
    \rho_{\tFDdiff}(\tilde{\ve{p}}, \tilde{\ve{q}}, T) & = \text{LSE}_{t}(\log_t\frac{\tilde{\ve{p}}}{\tilde{\ve{q}}}, T)\,,\\
    \rho_{\tHGdiff}(\tilde{\ve{p}}, \tilde{\ve{q}}, T) & = \rho_{\tFDdiff}(\tilde{\ve{p}}, \tilde{\ve{q}}, T) \oplus_t \rho_{\tFDdiff}(\tilde{\ve{q}}, \tilde{\ve{p}}, T)\,.
\end{align*}
Note that when $T \rightarrow \infty$, we have $\rho_{\tFDdiff}(\tilde{\ve{p}}, \tilde{\ve{q}}, T) \rightarrow \rho_{\tFD}(\tilde{\ve{p}}, \tilde{\ve{q}})$ and $\rho_{\tHGdiff}(\tilde{\ve{p}}, \tilde{\ve{q}}, T) = \rho_{\tHG}(\tilde{\ve{p}}, \tilde{\ve{q}})$. We show the distance balls of different radii using the differentiable approximation of the $t$-Hilbert distance in Figure~\ref{fig:diff-simplex}.

To empirically demonstrate the approximation error, we calculate the $t$-Hilbert distance as well as its differentiable approximation  between \num{10000} randomly sampled pairs of \acrotem s. We plot the histogram of the normalized relative error,
\[
\frac{\rho_{\tHGdiff}(\tilde{\ve{p}}, \tilde{\ve{q}}, T) - \rho_{\tHG}(\tilde{\ve{p}}, \tilde{\ve{q}})}{\rho_{\tHG}(\tilde{\ve{p}}, \tilde{\ve{q}})}\,,
\]
across different temperatures $t$ and smoothing factors $T$ in Figure~\ref{fig:diff-approx}(a)-(c). The results indicate that $\rho_{\tHGdiff}$ is an under (over) estimator of the true distance $\rho_{\tHG}$ for $t > 1$ ($t < 1$). Also, the relative approximation error is larger when $t$ deviates further from $1$. Nonetheless, we note experimentally that the differentiable max operator~\eqref{eq:lse} is less stable numerically for larger $T$ but more stable when $t < 1$. To utilize this property, we explore applying the differentiable max operator~\eqref{eq:lse} using a temperature $t < 1 - \delta\,,\,\,0 < \delta \ll 1$ for calculating the differentiable $t$-Hilbert distance. Let $\rho_{\ttpHGdiff}(\tilde{\ve{p}}, \tilde{\ve{q}}, T) = \text{LSE}_{1-\delta}(\log_t\sfrac{\tilde{\ve{p}}}{\tilde{\ve{q}}}, T) \oplus_t \text{LSE}_{1-\delta}(\log_t\sfrac{\tilde{\ve{q}}}{\tilde{\ve{p}}}, T)$ denote such differentiable $t$-Hilbert distance with a mismatched differentiable-max temperature $1-\delta$. We compare the relative error of the mismatched approximation for different values of $\delta$ and $T$ in Figure~\ref{fig:diff-approx}(d).

\begin{figure*}[t!]
\begin{center}
    \subfigure[]{\includegraphics[width=0.24\linewidth]{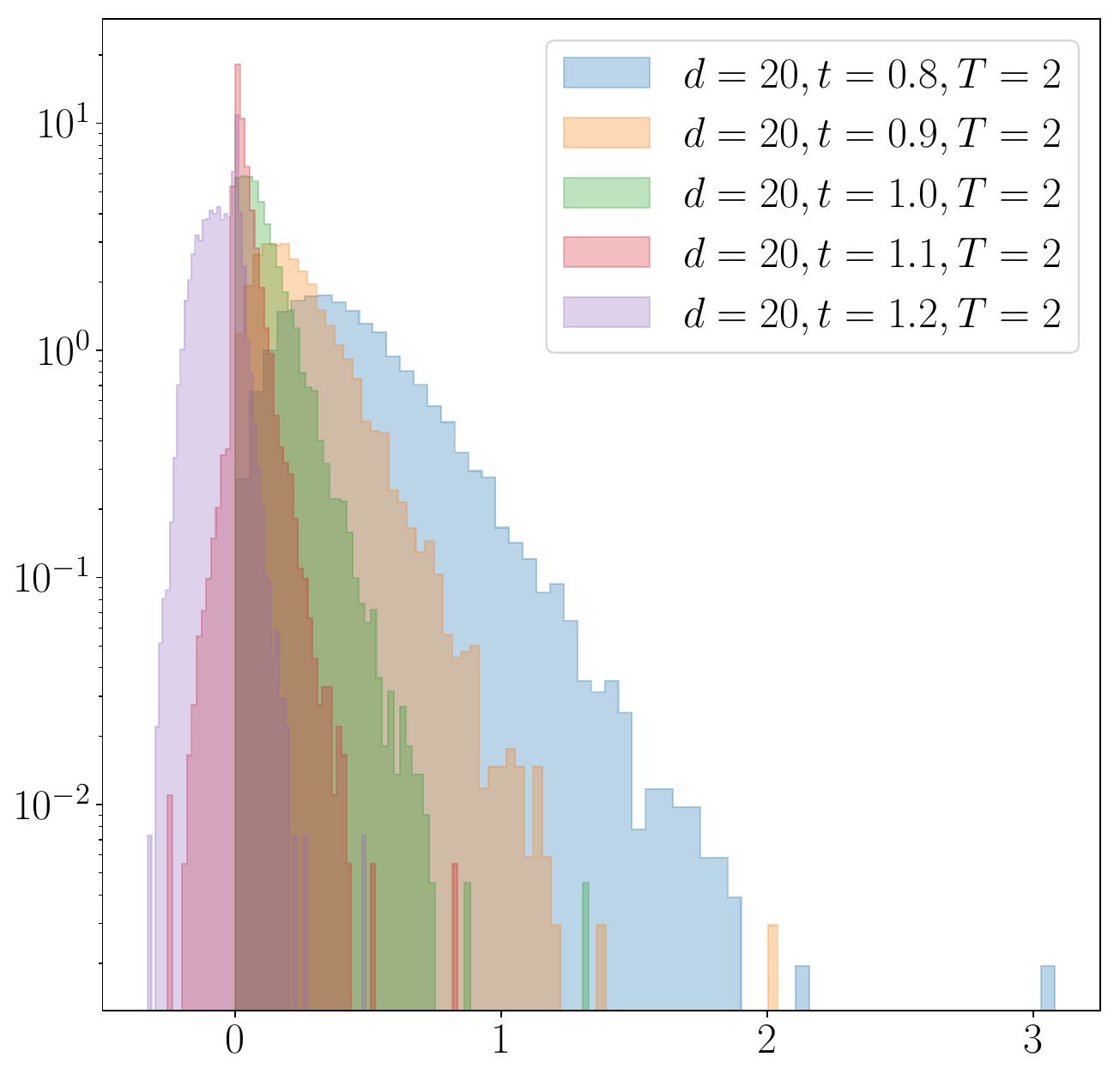}}
    \subfigure[]{\includegraphics[width=0.24\linewidth]{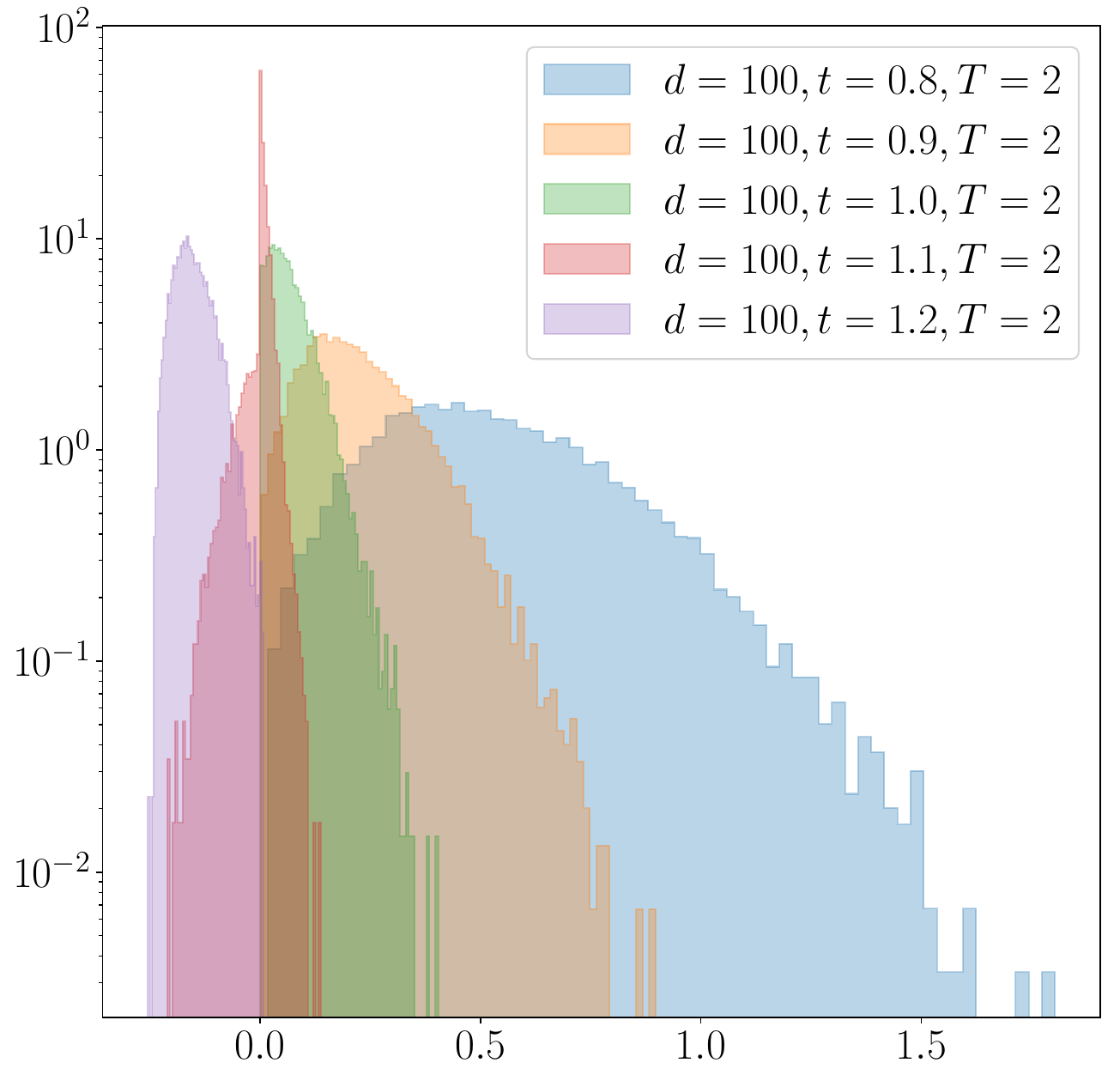}}
    \subfigure[]{\includegraphics[width=0.24\linewidth]{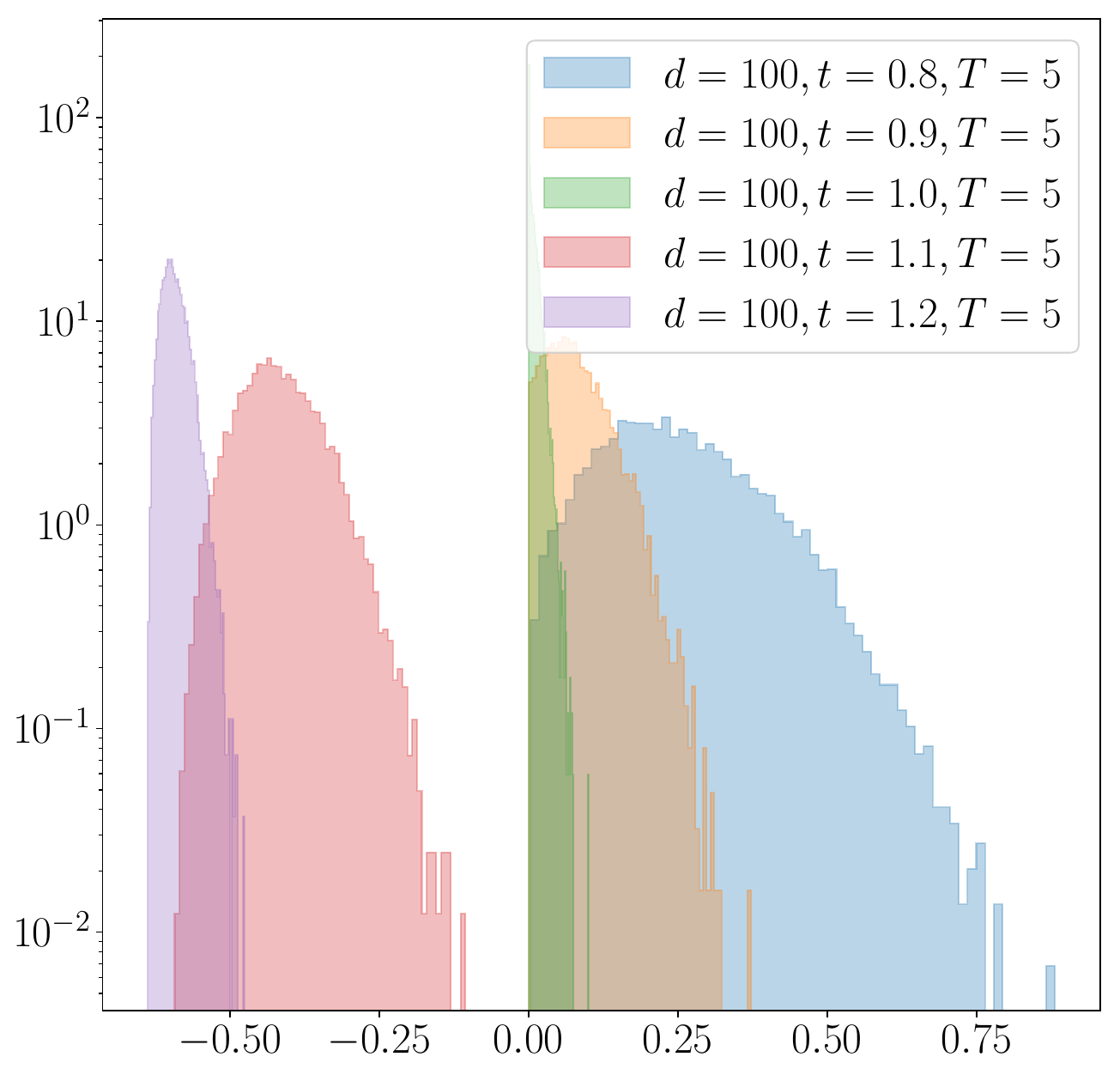}}
    \subfigure[]{\includegraphics[width=0.25\linewidth]{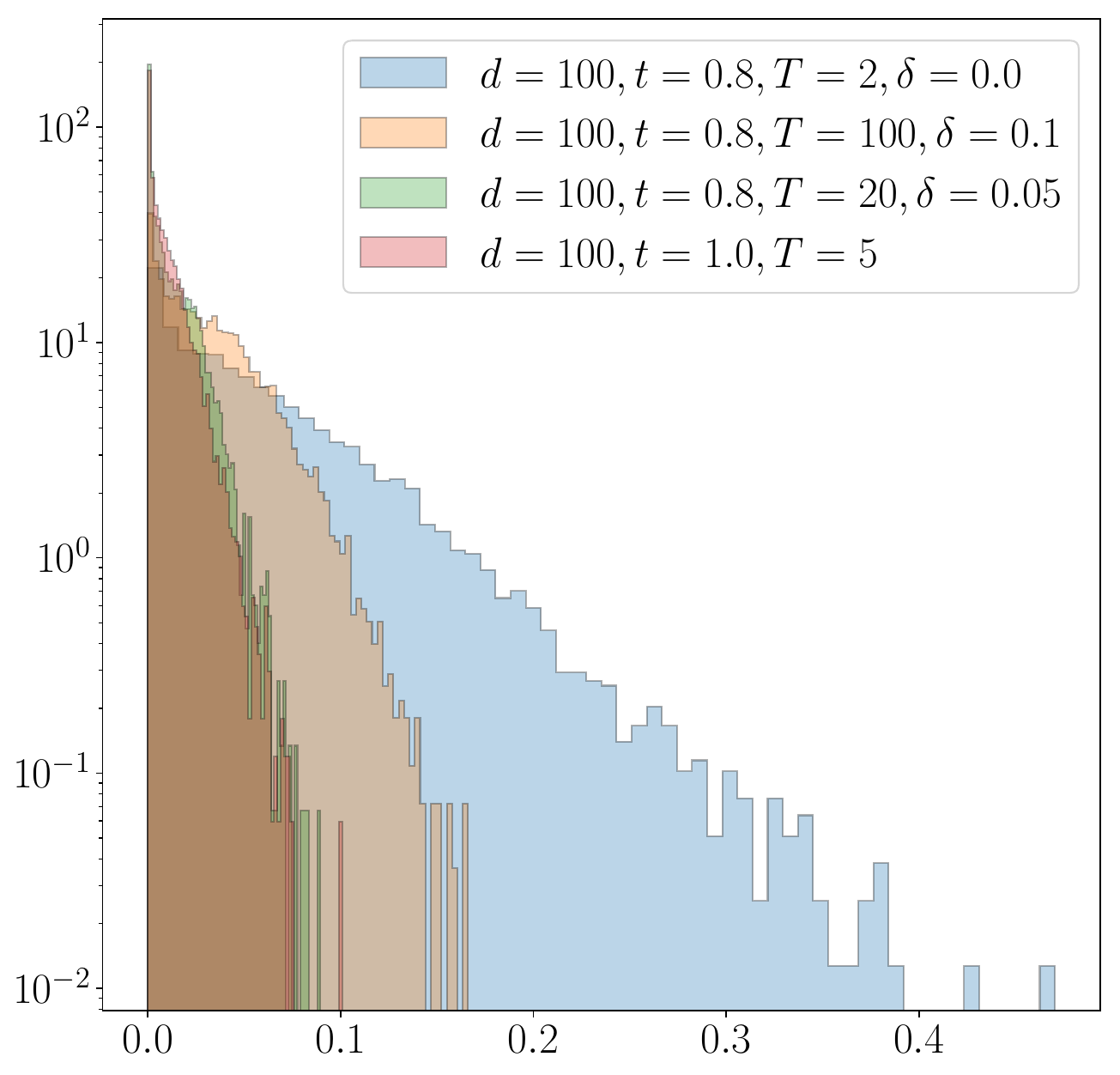}}
    \caption{(a)-(c) Histograms of the normalized relative error across different temperatures $t$ and smoothing factors $T$ across different dimensions $d$. $\rho_{\tHGdiff}$ is an under (over) estimator of the true distance $\rho_{\tHG}$ for $t > 1$ ($t < 1$). Also, the relative approximation error is larger when $t$ deviates further from $1$. (d) The histograms of relative errors when using a mismatched temperature $1-\delta$ for the differentiable-max operator.}
    \label{fig:diff-approx}
    \end{center}
\end{figure*}

\begin{figure*}[t!]
\begin{center}
    $T=1\,$  \subfigure{\includegraphics[width=0.25\linewidth]{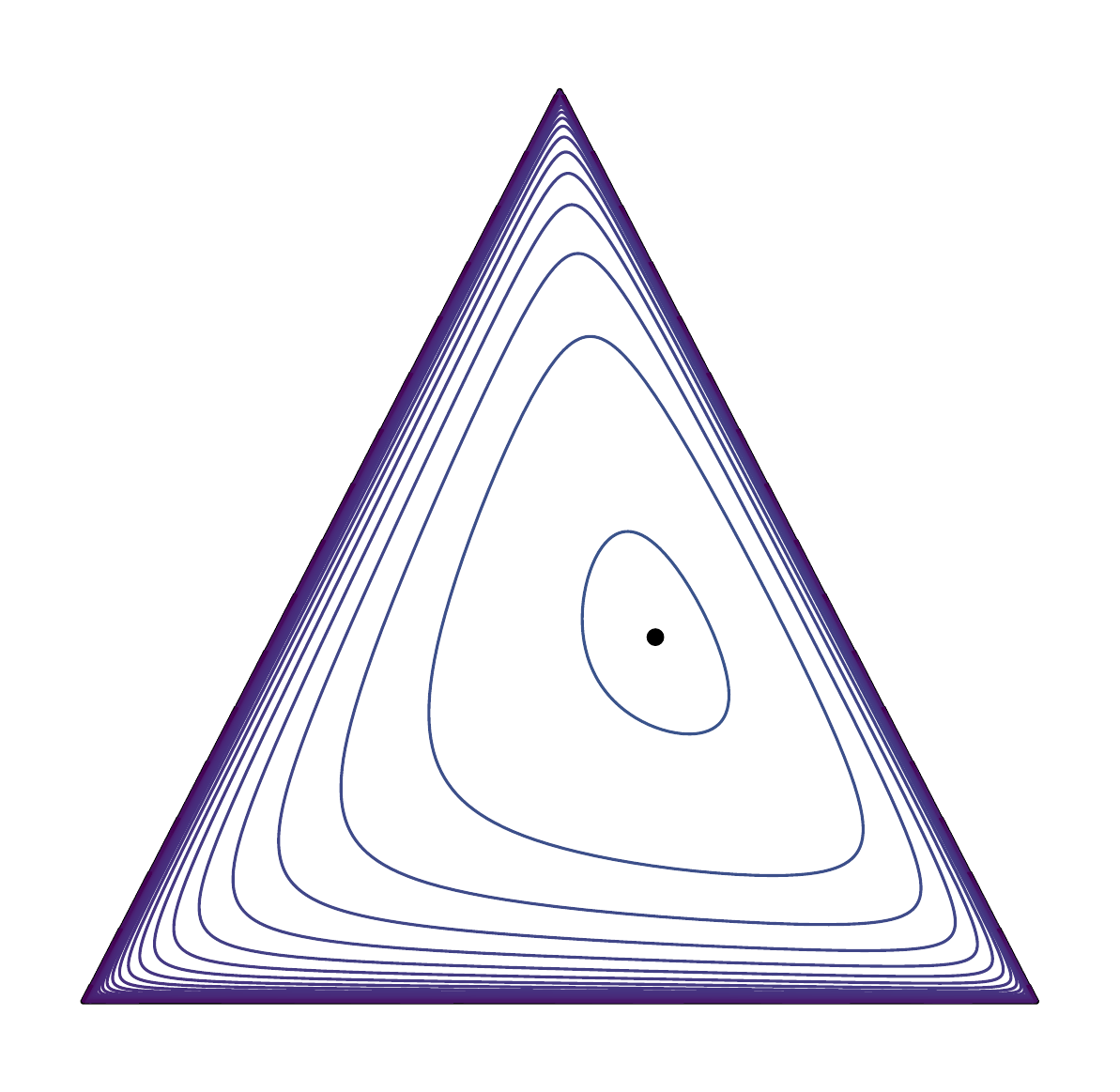}}
    \subfigure{\includegraphics[width=0.255\linewidth]{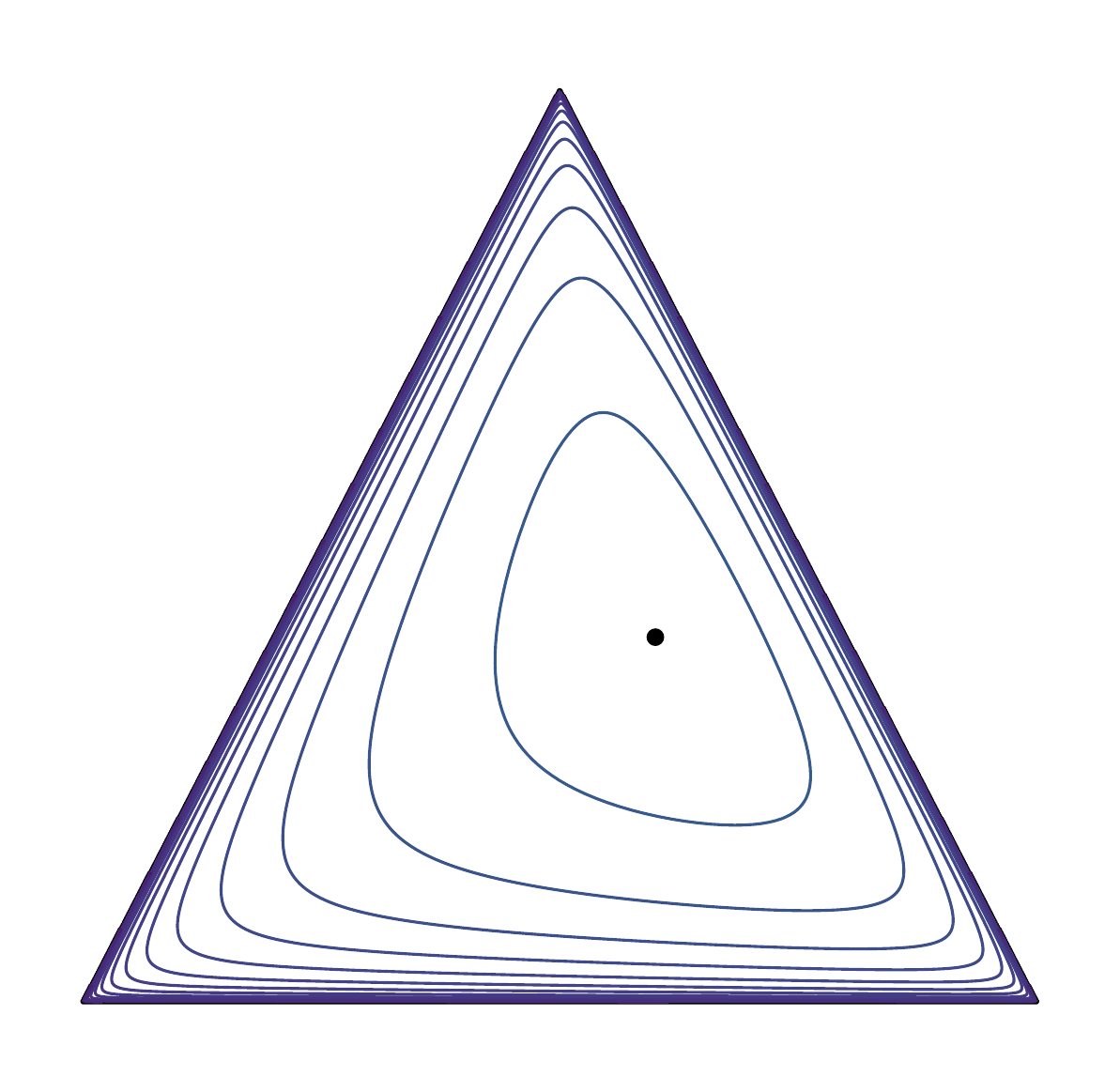}}
    \subfigure{\includegraphics[width=0.25\linewidth]{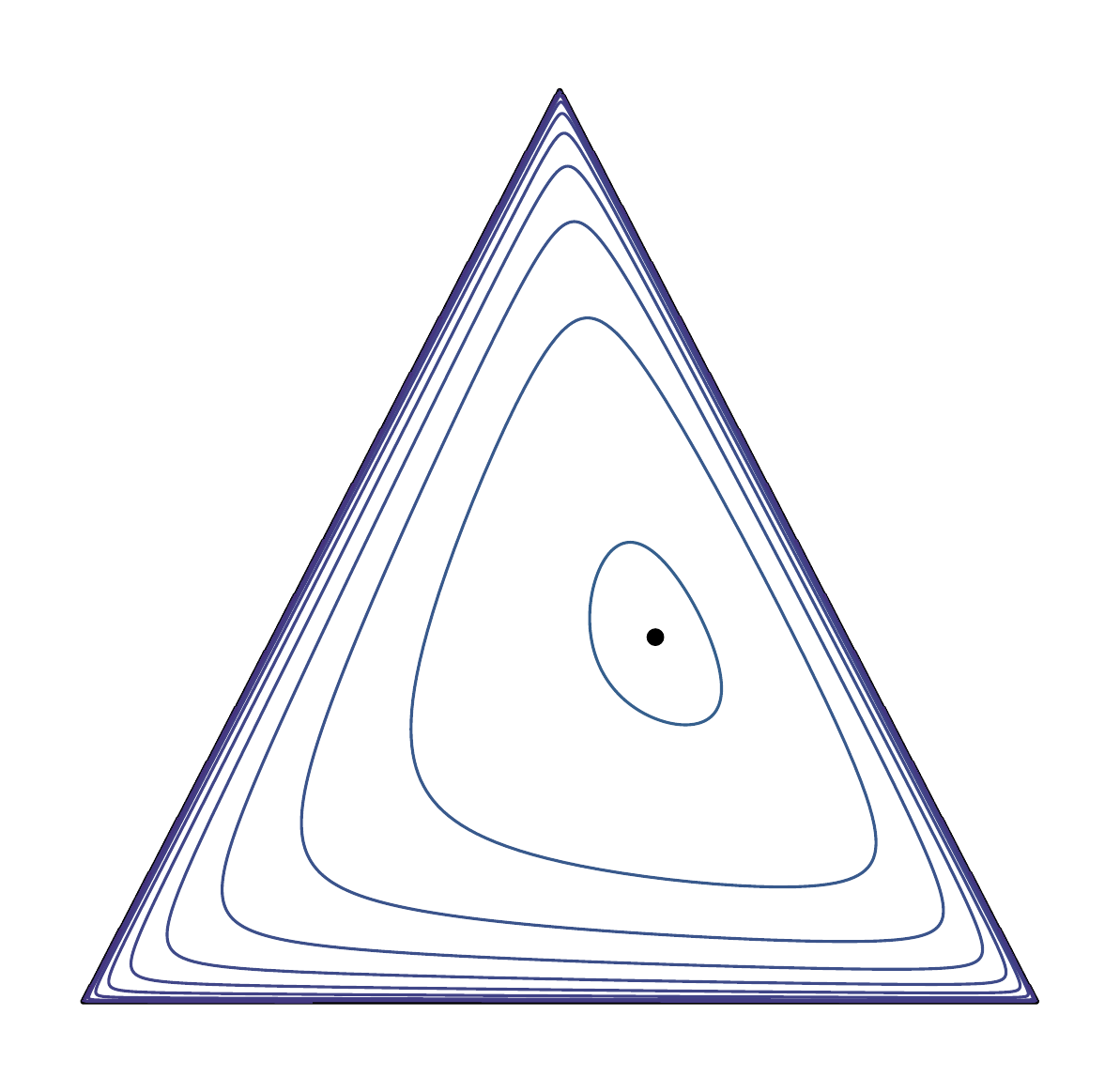}}\\\vspace{-0.5cm}
    $T=10$  \subfigure{\includegraphics[width=0.25\linewidth]{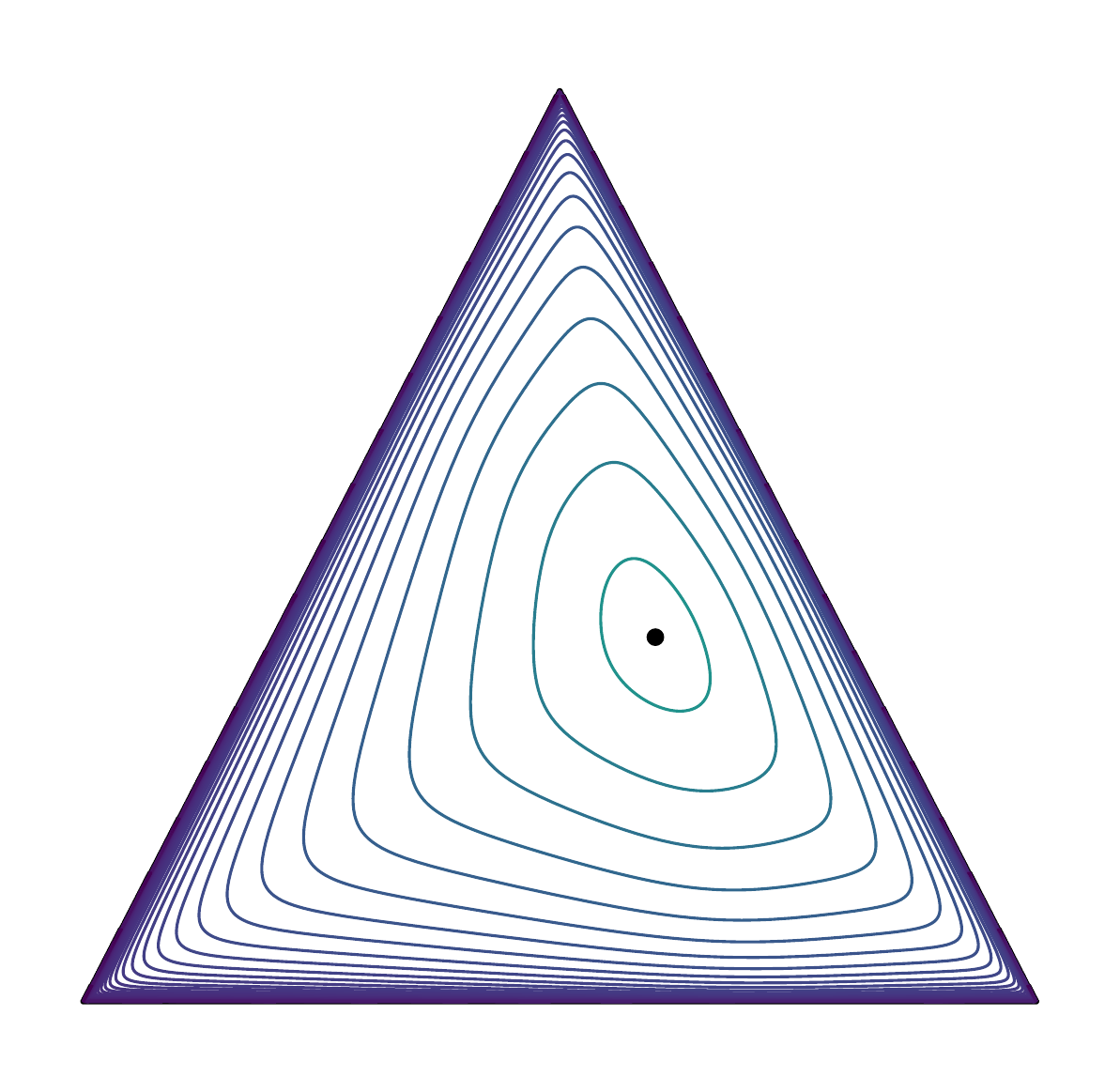}}
    \subfigure{\includegraphics[width=0.255\linewidth]{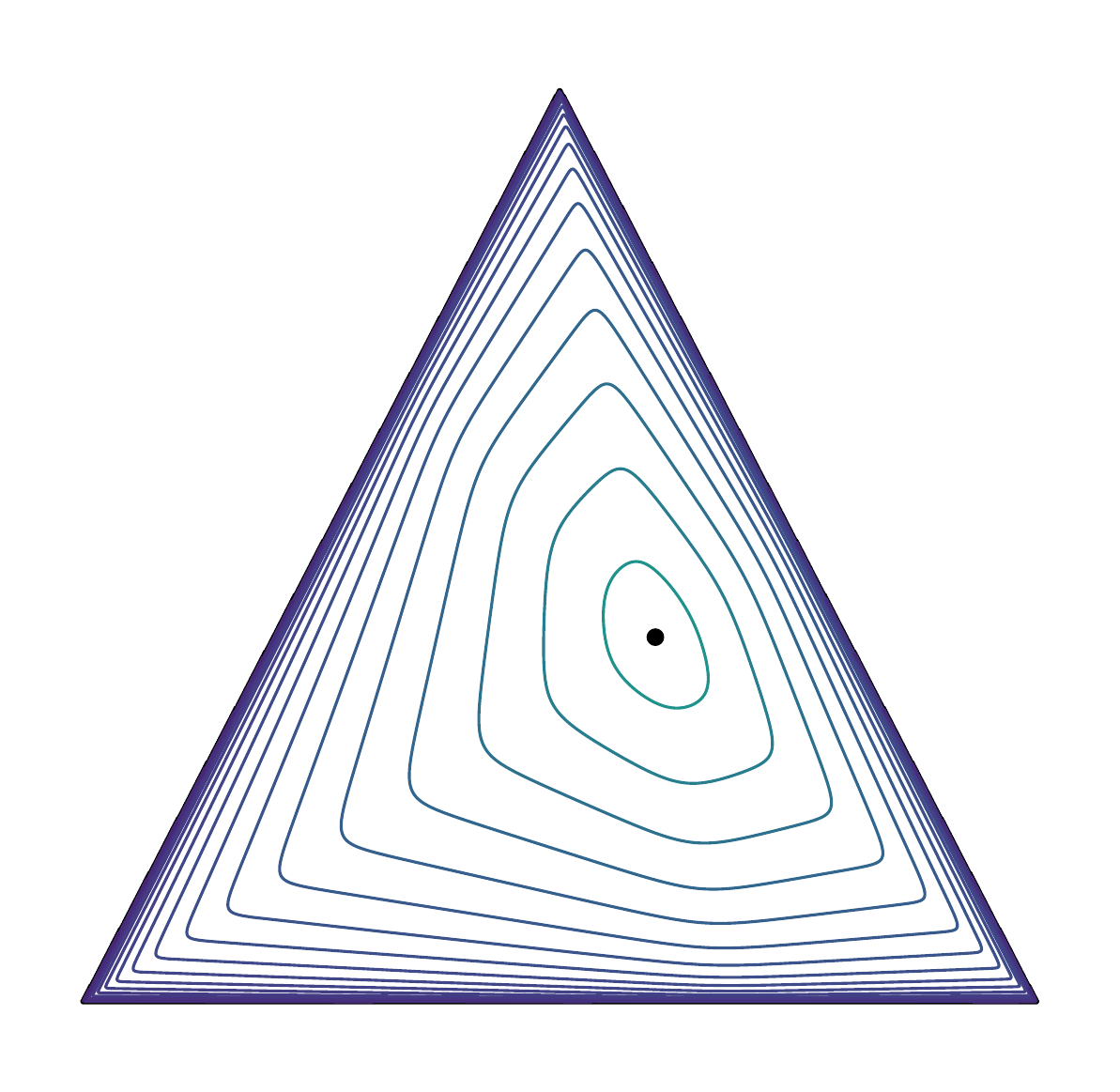}}
    \subfigure{\includegraphics[width=0.25\linewidth]{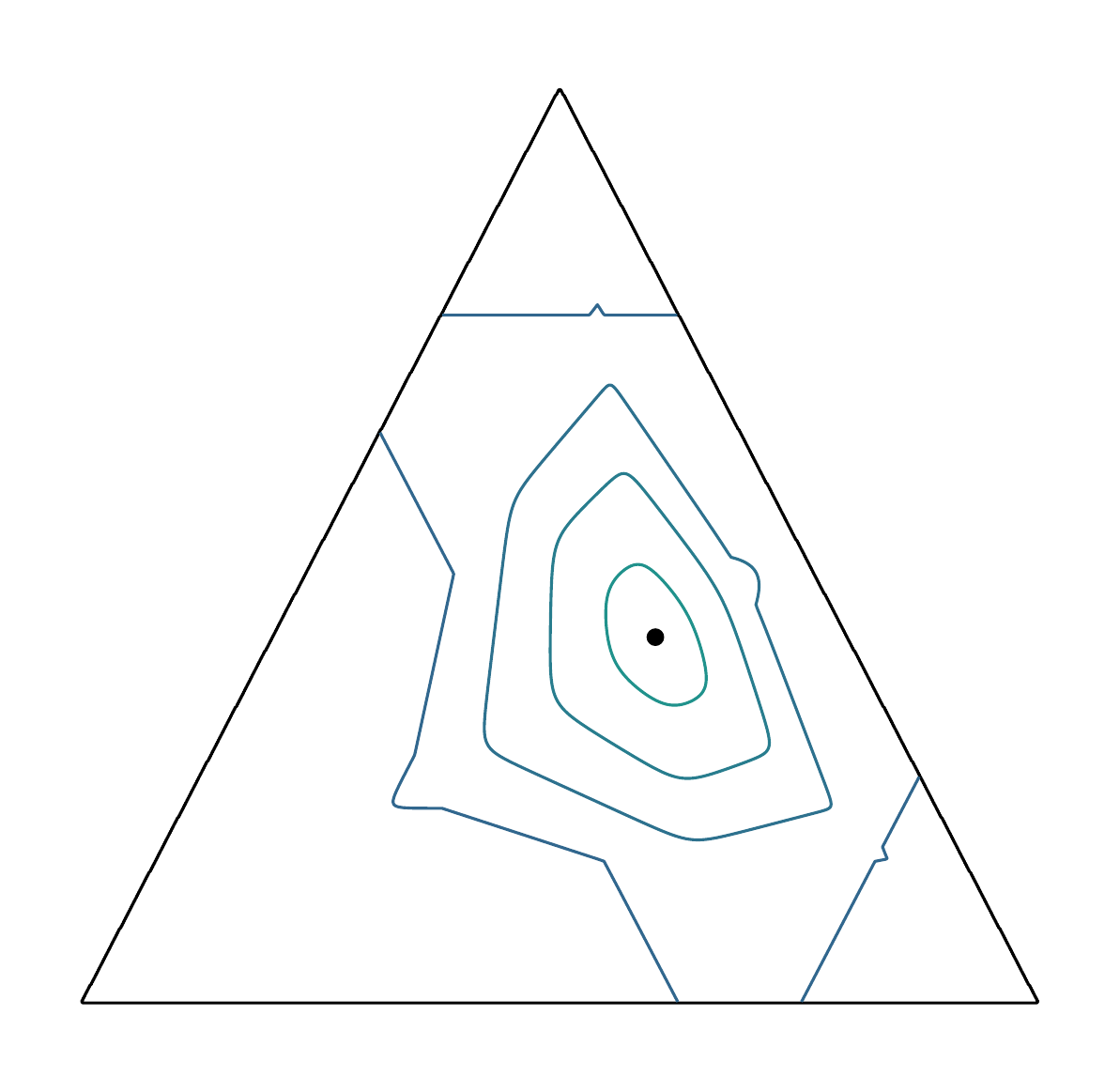}}\\\vspace{-0.5cm}
    $T=20$ \subfigure[$t=0.8$]{\addtocounter{subfigure}{-6}\includegraphics[width=0.25\linewidth]{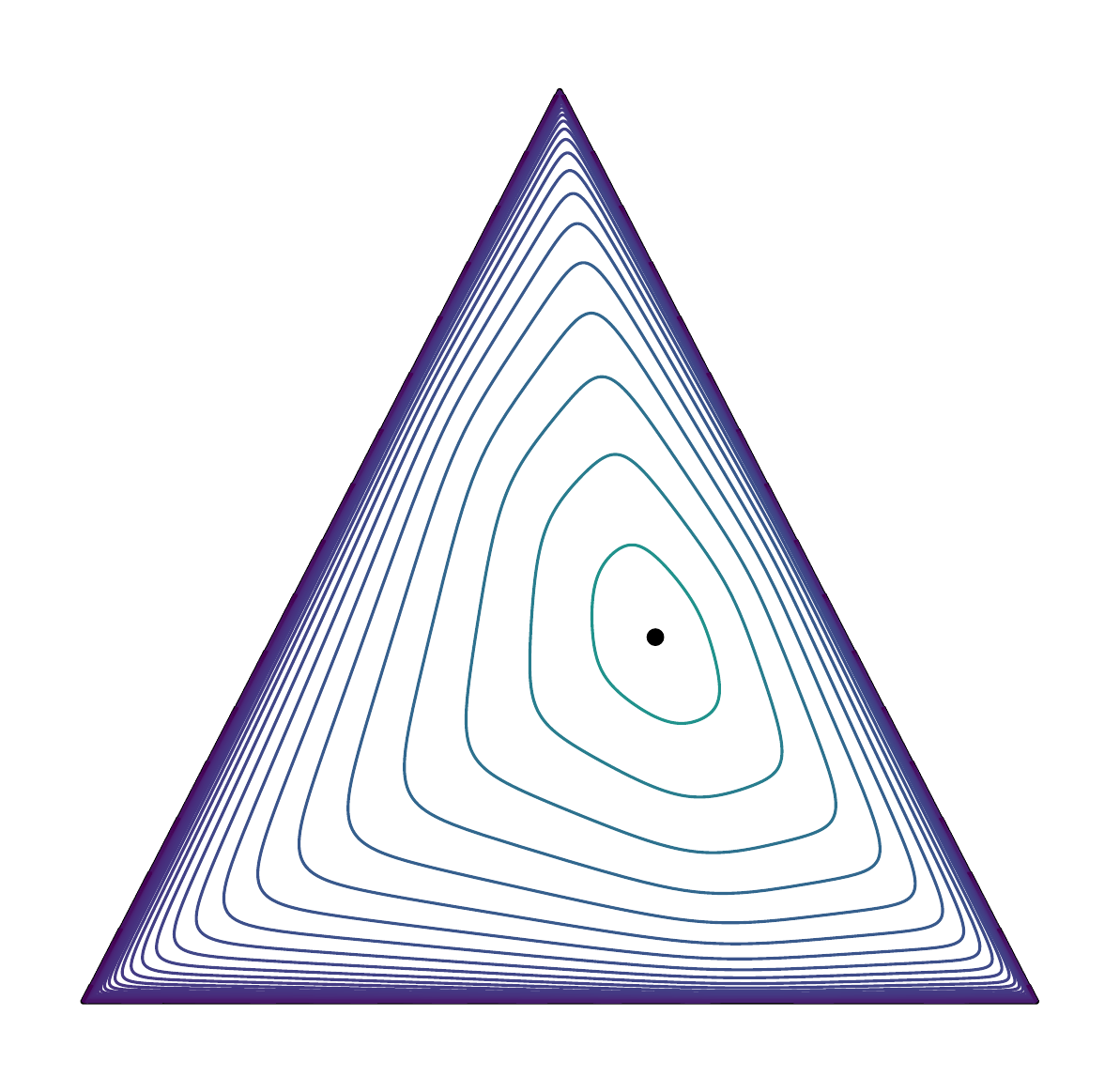}}
    \subfigure[$t=1.0$]{\includegraphics[width=0.255\linewidth]{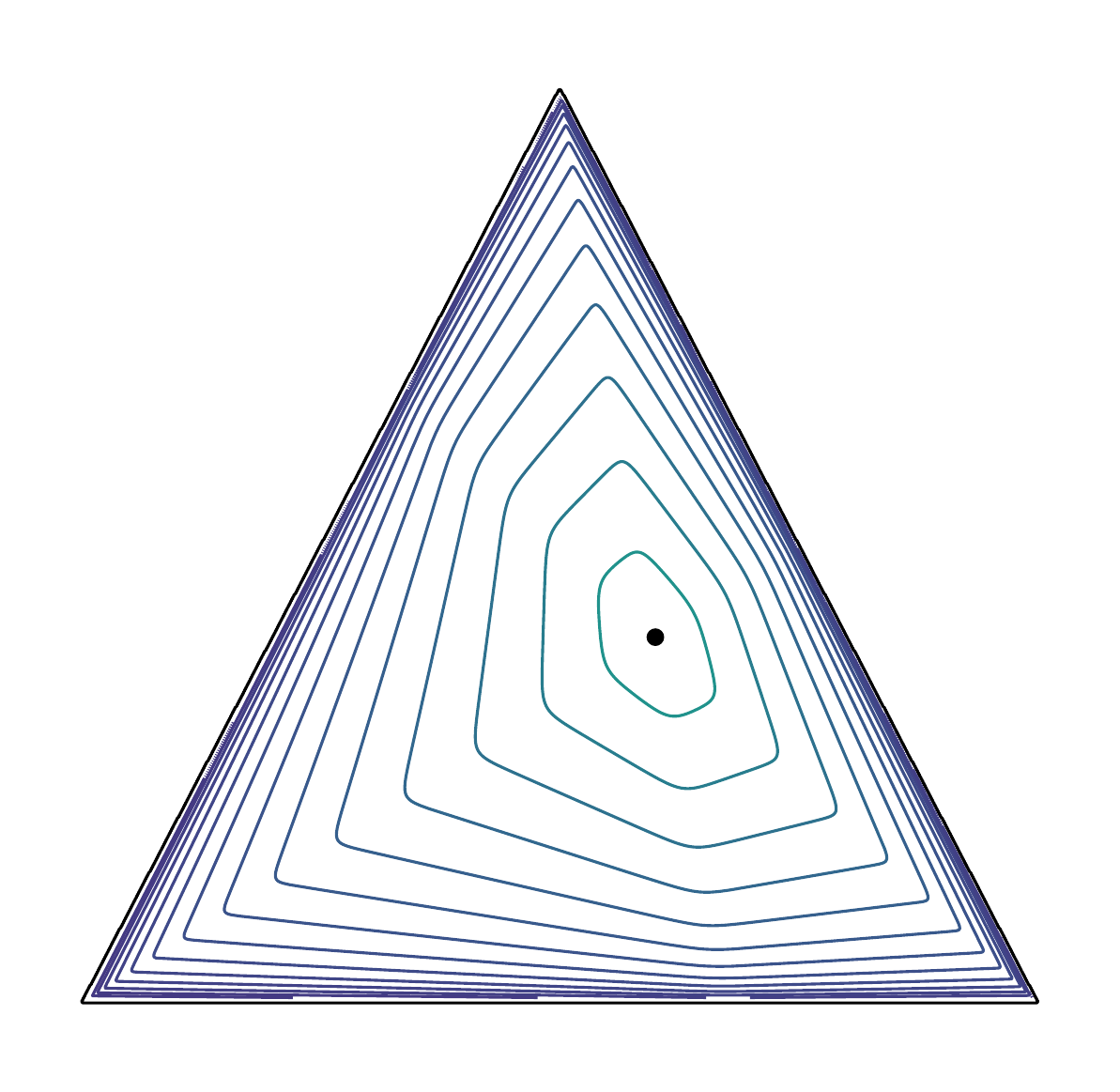}}
    \subfigure[$t=1.1$]{\includegraphics[width=0.25\linewidth]{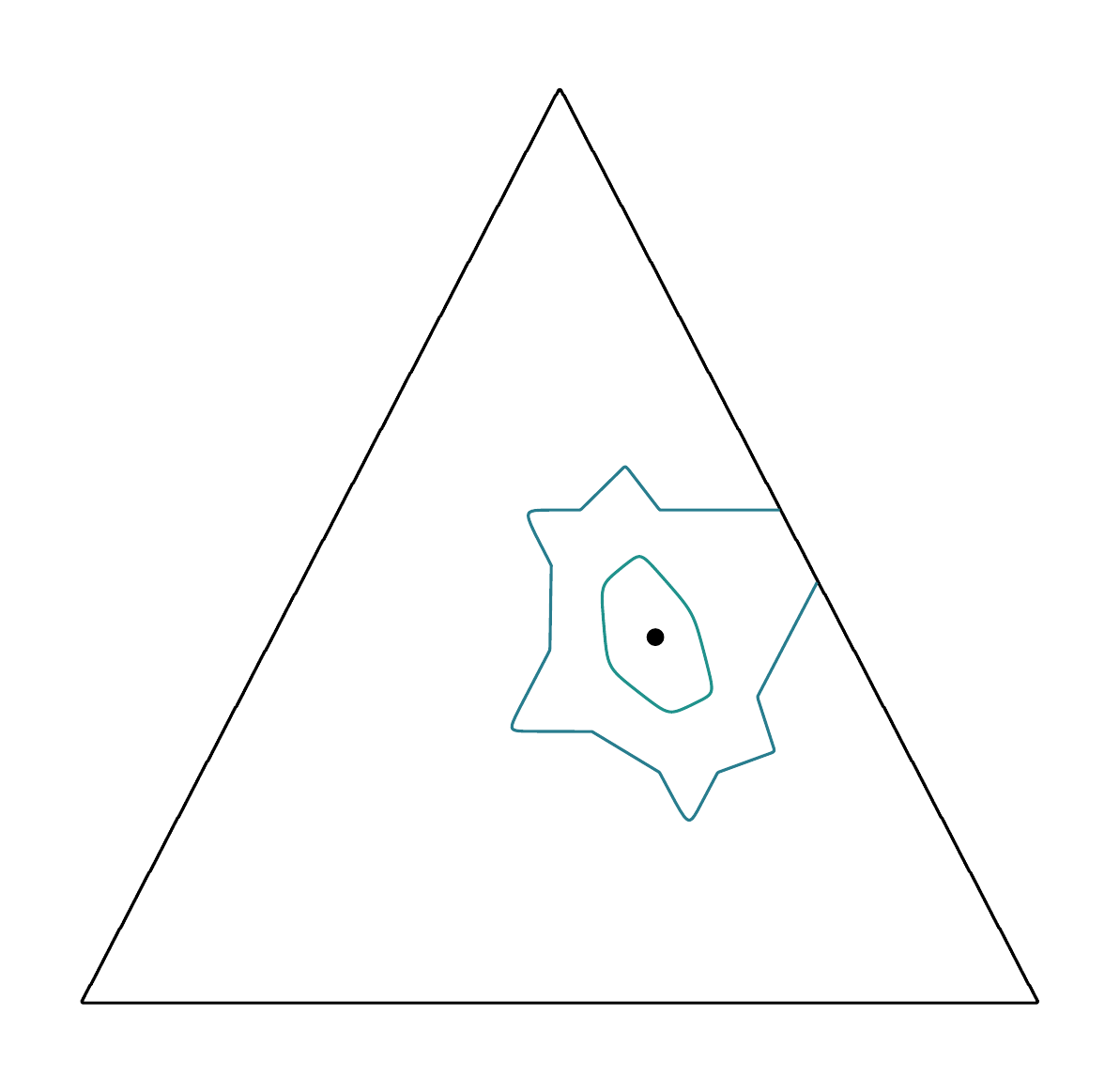}}
    \caption{Balls of different radii with respect to $\rho_{\tHGdiff}$ shown on the simplex for different temperature $t$ and smoothing factor $T$. The center is shown with a black dot. Darker colors indicate larger distances. Only when $t \leq 1$, we have $\rho_{\tHGdiff} \rightarrow \rho_{\tHG}$ as $T \rightarrow \infty$.}
    \label{fig:diff-simplex}
    \end{center}
\end{figure*}

\section{Details of the Experiments on Comparing Different Geometries}
We follow the procedure in~\citet{NLEHilbert-2023} for comparing the quality of the embeddings. Let $\mathcal{M}^d$ denote the manifold that we use to embed a set of points from a dataset $\mathcal{D}$ and $\mathcal{Y} = \{\ve{y}_i\}_{i\in [n]}$ denote the embedding. Given the pairwise distances  $D_{ij}$ between the points $i,j\in[n]$ in the original representation, we measure the embedding loss
\begin{equation}
    \mathcal{L}(D, \mathcal{M}^d) = \inf_{\mathcal{Y}\in  (\mathcal{M}^d)^n} \frac{1}{n^2} \sum_{i\in[n]}\sum_{j\in[n]} (D_{ij} - \rho_{\mathcal{M}}(\ve{y}_i, \ve{y}_j))\,,
\end{equation}
where $\rho_{\mathcal{M}}$ denotes the distance function of the geometry. We consider: i) Euclidean, ii) Minkowski hyperboloid model, iii) Hilbert simplex embedding, and iv) $t$-Hilbert co-simplex embedding. For Hilbert simplex embedding, we represent the points in log-coordinates. For $t$-Hilbert, we use the same $\log$-representation as Hilbert embedding but apply the function $x\mapsto\log_t\exp x$ on the Hilbert distance to convert the pairwise distances. We use \texttt{PyTorch} for the implementation and \texttt{Adam} as the optimizer. The details of the initialization and tuning follow~\citet{NLEHilbert-2023} and are omitted.

\section{Tempered Klein and Poincar\'{e} Disk Models}
The Klein model represents the hyperbolic plane as the interior of a unit ball. The distance between any two points $\ve{r}$ and $\ve{s}$ in the interior of the unit ball corresponds to Hilbert's distance with $\chi=\frac{1}{2}$ and $\Omega=\{ \ve{x}\in\bbR^d \st \|\ve{x}\|_2<1\}$ as the convex domain~\citep{ProjectiveBook-2011}:
$$
\rho_{\text{\tiny K}}(\ve{r},\ve{s})=\arccosh\left( \frac{1-\ve{r}\cdot \ve{s}}{\sqrt{(1-\ve{r}\cdot \ve{r})\, (1-\ve{s}\cdot \ve{s})}}\right)\,,
$$
where $\ve{r}\cdot \ve{s}$ denotes the Euclidean inner-product and
$\arccosh(x)=\log(x+\sqrt{x^2-1})$ for $x>1$. The geodesics in this model connecting any two points correspond to line segments, although the representation is non-conformal (i.e., it does not preserve the angles except at the disk origin).

Consider the map
\begin{equation}
    \label{eq:dist-map}
    \psi_{t, \chi}(u) = \chi \log_t \exp \frac{u}{\chi}\,,\,\,\chi > 0\,.
\end{equation}
Note that the map $\psi_{t, \chi}$ converts a given input $\log$-cross-ratio, scaled by $\chi$, into the $\log_t$ of the cross-ratio, scaled by the same factor $\chi$. Thus, to generalize the Klein model, we define the tempered Klein distance as $$\rho_{t\text{\tiny -K}}(\ve{r},\ve{s}) =  \psi_{t,\sfrac{1}{2}}(\rho_{\text{\tiny K}}(\ve{r},\ve{s}))\,.$$
Clearly, we have $\rho_{t\text{\tiny -K}}(\ve{r},\ve{s}) \rightarrow \rho_{\text{\tiny K}}(\ve{r},\ve{s})$ as $t \rightarrow 1$. 

The Poincar\'{e} disk model is a conformal representation also defined on the interior of a unit ball. In the Poincar\'{e} model, the distance between two points $\ve{r}$ and $\ve{s}$ in the domain is defined as 
$$
\rho_{\text{\tiny P}}(\ve{r},\ve{s})=\arccosh\left(1 + 2\, \frac{\|\ve{r} - \ve{s}\|^2_2}{\sqrt{(1-\|\ve{r}\|^2_2)\, (1-\ve{s}\cdot \|\ve{s}\|^2_2)}}\right)\,.
$$
In this model, the straight lines in the hyperbolic geometry are represented by arcs of circles perpendicular to the unit
ball (including straight line segments passing through the disk origin). We similarly define the tempered Poincar\'{e} distance as $$\rho_{t\text{\tiny -P}}(\ve{r},\ve{s}) =  \psi_{t,\sfrac{1}{2}}(\rho_{\text{\tiny P}}(\ve{r},\ve{s}))\,,$$
where similarly we have $\rho_{t\text{\tiny -P}}(\ve{r},\ve{s}) \rightarrow \rho_{\text{\tiny P}}(\ve{r},\ve{s})$ as $t \rightarrow 1$. 
A point $\ve{k}$ in the Klein disk model is mapped to a point in the Poincar\'{e} disk via the radial rescaling map
\begin{equation}
\label{eq:map-k-to-p}
\pi:\,\ve{k} \mapsto \frac{1 - \sqrt{1 - \|\ve{k}\|^2_2}}{\|\ve{k}\|^2_2}\,\ve{k}\,.
\end{equation}
Since the tempered Klein distance $\rho_{t\text{\tiny -K}}$ is obtained from the Klein distance $\rho_{\text{\tiny K}}$ via the monotonic transformation $\psi_{t, \sfrac{1}{2}}$, the same formula~\eqref{eq:map-k-to-p} maps the tempered Klein model to the tempered Poincar\'{e} model. We visualize several straight lines in the tempered Klein disk model that have identical distances between the start point $\bullet$ and the endpoint $\ve{\times}$ in Figure~\ref{fig:klein-poincare}(a) for different $t$. Note that the value of the distance changes with $t$, but the distances are identical for the same value of $t$. We also show the point on the line that is $0.2\times$ away from the start point $\bullet$ relative to the total distance. The relative distance is dilated for $t < 1$  and contracted $t > 1$, compared to the Klein model. Similarly, we show these lines in the Poincar\'{e} model in Figure~\ref{fig:klein-poincare}(b), where the same dilation and contraction of the distances are evident. We repeat the same procedure, but for a point $0.5\times$ away relative to the start point $\bullet$, located at the origin, for the Klein and Poincar\'{e} models in Figure~\ref{fig:klein-poincare}(c) and (d), respectively.

\begin{figure*}[t!]
\begin{center}
    \subfigure[]{\includegraphics[width=0.24\linewidth]{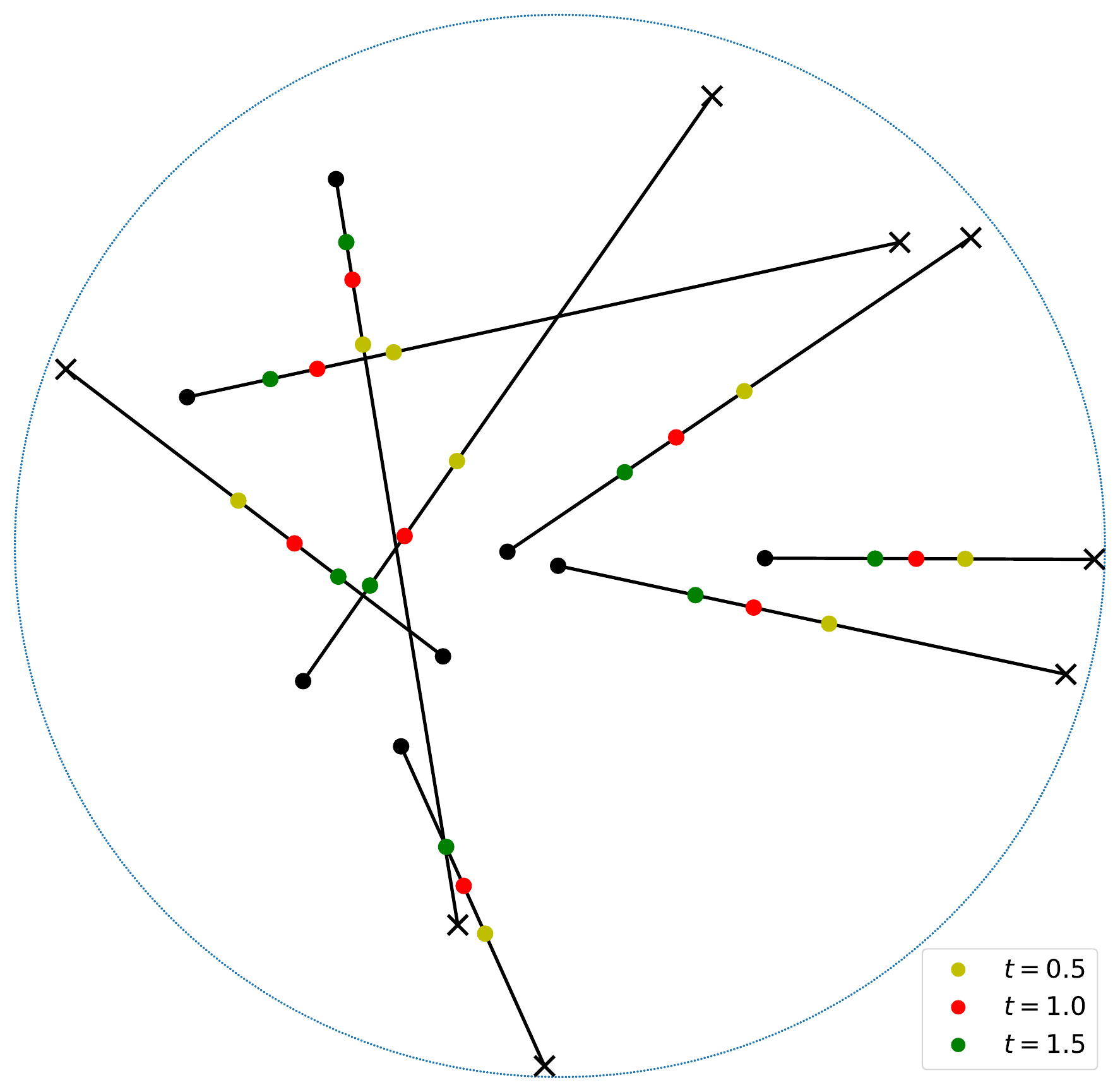}}
    \subfigure[]{\includegraphics[width=0.24\linewidth]{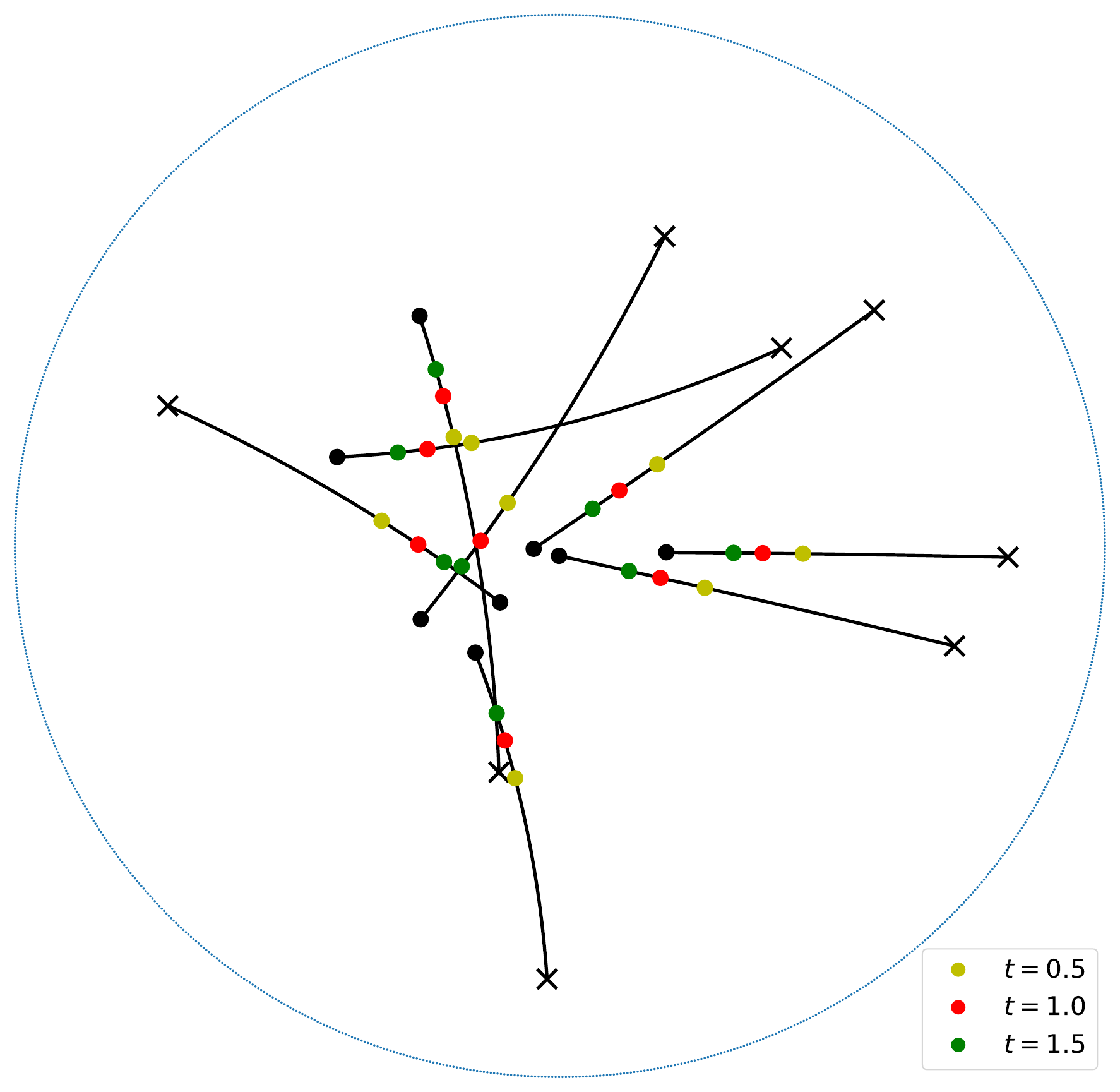}}
    \subfigure[]{\includegraphics[width=0.24\linewidth]{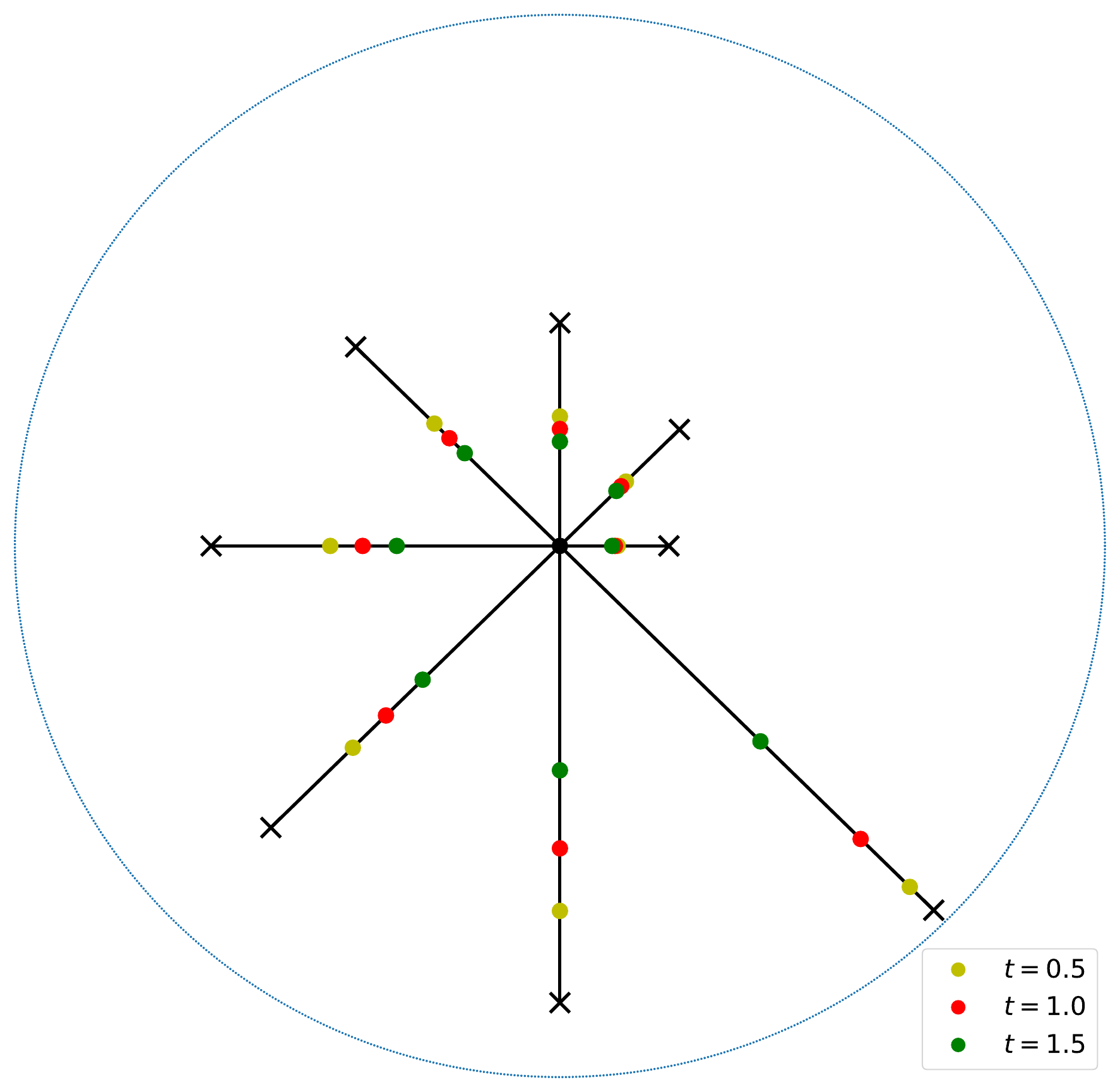}}
    \subfigure[]{\includegraphics[width=0.24\linewidth]{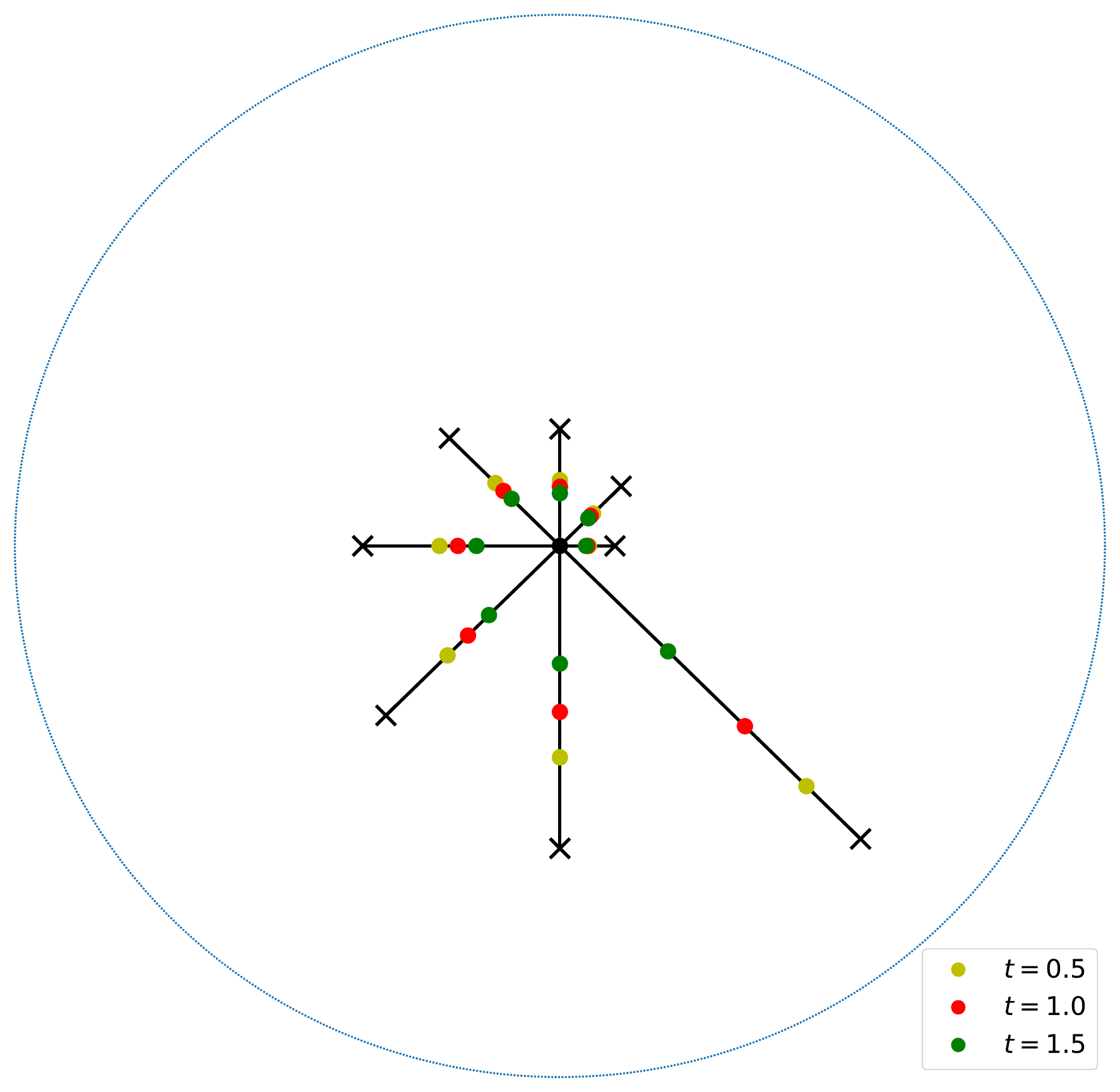}}
    \caption{Visualization of straight lines in the Klein ((a) and (c)) and Poincar\'{e} ((b) and (d)) models. In (a) and (b), the lines have identical lengths, and the colored point indicates the point $0.2\times$ away from the start point $\bullet$ relative to the total length for different $t$. We similarly show the points $0.5\times$ away (midpoints) from the start point $\bullet$ relative to the total length in (c) and (d). The tempered distances represents a dilation for $t < 1$ and a contraction for $t > 1$.}
    \label{fig:klein-poincare}
    \end{center}
\end{figure*}
\end{document}